\documentclass{article}

\usepackage[main, final]{neurips_2026}
\usepackage[utf8]{inputenc} % allow utf-8 input
\usepackage[T1]{fontenc}    % use 8-bit T1 fonts
\usepackage{hyperref}       % hyperlinks
\usepackage{url}            % simple URL typesetting
\usepackage{booktabs}       % professional-quality tables
\usepackage{amsfonts}       % blackboard math symbols
\usepackage{nicefrac}       % compact symbols for 1/2, etc.
\usepackage{microtype}      % microtypography
\usepackage{xcolor}         % colors

\usepackage{amsmath}
\usepackage{amsthm}
\usepackage[textsize=tiny]{todonotes}
\usepackage{cleveref}
\usepackage{enumitem}
\usepackage{multirow}
\usepackage{graphicx}
\usepackage{caption}
\usepackage{wrapfig}
\usepackage{array}
\usepackage{amssymb}
\usepackage{pifont}
\usepackage{bbm}
\usepackage[table,xcdraw,dvipsnames]{xcolor}
\usepackage{titletoc} 
\usepackage{algorithm}
\usepackage{algpseudocode}

\usepackage{colortbl}
\definecolor{LightGreen}{HTML}{E8F5E9}
\definecolor{LightOrange}{HTML}{FFF3E0} % Delicate orange

\usepackage{listings}

\definecolor{codeblue}{rgb}{0.15, 0.35, 0.75}
\definecolor{codepink}{rgb}{0.85, 0.25, 0.55}
\definecolor{codegreen}{rgb}{0.35, 0.60, 0.55}

\lstdefinelanguage{PseudoPython}{
    language=Python,
    basicstyle=\ttfamily\small,
    commentstyle=\color{codegreen},
    keywordstyle=\color{codepink},
    morekeywords={stopgrad, autograd, dsm_loss, norm, has_csv, get_threshold, where, kthvalue},
    literate=
      {=}{{\textcolor{codeblue}{=}}}{1}
      {+}{{\textcolor{codeblue}{+}}}{1}
      {-}{{\textcolor{codeblue}{-}}}{1}
      {*}{{\textcolor{codeblue}{*}}}{1}
      {/}{{\textcolor{codeblue}{/}}}{1}
      {>}{{\textcolor{codeblue}{>}}}{1}
      {<}{{\textcolor{codeblue}{<}}}{1},
    showstringspaces=false,
    frame=tb, % Top and bottom lines
    rulecolor=\color{black},
}

\newcommand{\btheta}{\boldsymbol{\theta}}
\newcommand{\x}{\mathbf{x}}
\newcommand{\xt}{\mathbf{x}_t}
\newcommand{\diff}{\text{d}}

\newcommand{\ldriftcoef}{\mathbf{F}_t\mathbf{x}_t}

\newcommand{\ldiffcoef}{\mathbf{G}_t}
\newcommand{\fode}{\mathbf{F}_{\text{ODE}}(\xt, t)}
\newcommand{\fodet}{\mathbf{F}_{\text{ODE}}^{\btheta}(\xt, t)}
\newcommand{\generator}{\boldsymbol{\mathcal{G}}_{\btheta}}
\newcommand{\vareps}{\boldsymbol{\epsilon}}
\newcommand{\wiener}{\mathbf{w}_t}
\newcommand{\rwiener}{\overline{\mathbf{w}}_t}
\newcommand{\score}{\nabla_{\mathbf{x}_t}{\log{p(\mathbf{x}_t)}}}
\newcommand{\scorec}{\nabla_{\mathbf{x}_t}{\log{p(\mathbf{x}_t\mid\mathbf{x}_0)}}}

\newcommand{\expval}{\mathbb{E}}
\newcommand{\ie}{\emph{i.e.}}
\newcommand{\eg}{\emph{e.g.}}
\newcommand{\norm}[1]{\left\|#1\right\|}
\newcommand{\Tr}{\operatorname{Tr}}

\newcommand{\tvf}{\operatorname{TVF}}
\newcommand{\cov}{\operatorname{Cov}}
\newcommand{\var}{\operatorname{Var}}
\newcommand{\lmi}{{\operatorname{LMI}}}
\newcommand{\lid}{\operatorname{LID}}
\newcommand{\lidt}{\lid_{\boldsymbol{\theta}}}
\newcommand{\iq}{\operatorname{IQ}}
\newcommand{\pmptheta}{\hat{\x}_0^{\btheta}(\x_t)}

\algnewcommand{\manifoldt}{\boldsymbol{\mathcal{M}}_{\boldsymbol{\theta}}}

\newtheorem{theorem}{Theorem}
\newtheorem{corollary}{Corollary}
\newtheorem{proposition}{Proposition}

\newtheorem{assumption}{Assumption}

\makeatletter
\renewenvironment{proof}[1][\proofname]{\par
  \pushQED{\qed}%
  \normalfont \topsep6\p@\@plus6\p@\relax
  \trivlist
  \item[\hskip\labelsep
        \textbf{#1.}]\ignorespaces
}{%
  \popQED\endtrivlist\@endpefalse
}
\makeatother

\title{Local Intrinsic Dimension Unveils \\Hallucinations in Diffusion Models}

\author{%
  Bartlomiej Sobieski\textsuperscript{1,2,3,4}\thanks{Corresponding author at \href{b.sobieski@uw.edu.pl}{b.sobieski@uw.edu.pl}}\ \hspace{0.1em}\thanks{Work done during an internship at Harvard Medical School and Massachusetts General Hospital}\ , Matthew Tivnan\textsuperscript{3,4}, \textbf{Michał Jan Włodarczyk\textsuperscript{1}},\\ \textbf{Dawid Płudowski\textsuperscript{1}},
  \textbf{Pengfei Jin\textsuperscript{3,4}}, \textbf{Przemyslaw Biecek\textsuperscript{1,2}}, \textbf{Quanzheng Li\textsuperscript{3,4}} \\[0.2cm]
  \textsuperscript{1}Centre for Credible AI, Warsaw University of Technology, 
  \textsuperscript{2}University of Warsaw,\\ \textsuperscript{3}Massachusetts General Hospital, \textsuperscript{4}Harvard Medical School
}

\begin{document}

\maketitle

\begin{abstract}
Diffusion models are prone to generating structural hallucinations - samples that match the statistical properties of the training data yet defy underlying structural rules, resulting in anomalies like hands with more than five fingers. Recent research studied this failure mode from several viewpoints, offering partial explanations to their occurrence, such as mode interpolation. In this work, we propose a complementary perspective that treats hallucinations as instabilities on the model-induced manifold. We begin by showing that a hallucination filter based on such instabilities matches or exceeds the performance of the recently proposed temporal one. By tracing the source of these instabilities, we identify local intrinsic dimension (LID) as their primary driver and propose Intrinsic Quenching (IQ), a direct corrective mechanism that deflates it to alleviate hallucinations. IQ consistently outperforms standard hallucination reduction baselines across a wide array of benchmarks and offers a highly promising solution for enforcing anatomical consistency in downstream medical imaging tasks.
\end{abstract}
\vspace{-1em}
\section{Introduction}
\vspace{-1em}
\begin{minipage}[t]{0.48\textwidth}
  \vspace{0pt} 
Diffusion models \citep[DMs;][]{sohl2015deep,songscore,lipmanflow,liu2022,albergo2023building,song2023consistency,geng2025mean} drive modern generative modeling across computer vision and other modalities \citep{dhariwal2021beat,lou2024discrete,kong2021diffwave,kotelnikov2023tabddpm,hoogeboom2022equivariant}. Recent research highlights their critical failure mode termed \emph{hallucinations}: generated samples that successfully capture the statistical properties of the training data, but fundamentally violate its underlying physical, logical, or morphological patterns \citep{ramesh2021zero,rombach2022high}.
\end{minipage}
\hfill
\begin{minipage}[t]{0.49\textwidth}
  \vspace{0pt}
  \centering
  \includegraphics[width=\textwidth]{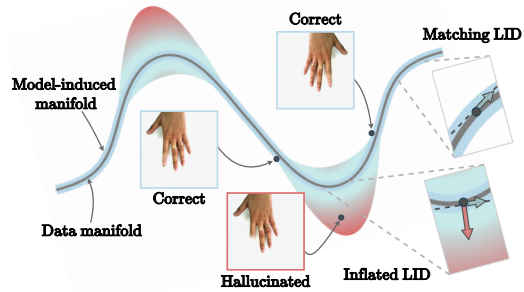}
  \captionof{figure}{A diffusion model-induced manifold, approximating the data manifold, invents spurious dimensions in unstable regions, leading to structural hallucinations.}
  \label{fig:teaser}
\end{minipage}

In particular, \emph{structural} hallucinations \citep{kim2024tackling} fail to recover correct object forms, generating anomalies like six-fingered hands or misaligned eyes.
Lacking a formal definition, identifying them relies strictly on human perception, restricting research to empirical and correlational studies \citep{oorloff2025mitigating,triaridis2025mitigating,tian2025rods}. Despite initial explanatory hypotheses \citep{aithal2024understanding} and concurrent mitigation approaches \citep{chandran2025laplacian,wewer2025spatial}, understanding their root causes and prevention remains an open question.

In this work, we propose shifting the analysis of structural hallucinations from temporal generative behavior \citep{aithal2024understanding,tian2025rods} to the local instability of the model-induced manifold. Specifically: \textbf{C1.} We demonstrate that a spatial hallucination filter based on local geometric instability matches or exceeds the recently proposed temporal filter \citep{aithal2024understanding}, identifying the inflation of local intrinsic dimension (\(\lid\)) as the primary indicator. \textbf{C2.} We theoretically formulate \emph{Intrinsic Quenching} (\(\iq\)), a thermodynamically-inspired corrective mechanism that dynamically deflates \(\lid\) of the generated sample on the model-induced manifold throughout the generative process, reducing the hallucinatory behavior.
\textbf{C3.} Through large-scale, human-annotated evaluations, we show that \(\iq\) outperforms all baselines in reducing structural hallucinations and offers a promising solution for enforcing anatomical consistency in downstream medical imaging tasks.

\section{Background}

DMs can be formulated through the framework of stochastic differential equations (SDEs):
\begin{align}
\diff \mathbf{x}_t &= \ldriftcoef \diff t + \ldiffcoef \diff \wiener,\label{eq:linear_forward_sde} \\
\diff \mathbf{x}_t &= [\ldriftcoef - \ldiffcoef \ldiffcoef^\top \score] \diff t + \ldiffcoef \diff \rwiener, \label{eq:linear_reverse_sde}
\end{align}
where \(\ldriftcoef \) is a linear \emph{drift term} with $\x_t\in\mathbb{R}^n$ and time-dependent matrix \(\mathbf{F}_t \in \mathbb{R}^{n\times n}\), \(\ldiffcoef\in \mathbb{R}^{n\times n}\) is a matrix-valued \emph{diffusion} coefficient, \(\wiener\in \mathbb{R}^{n}\) and \(\rwiener\in \mathbb{R}^{n}\) are the Wiener processes running forward and reverse in time respectively, \(\score\) is the \emph{score} function of the time-parameterized distribution $p(\x_t), t\in[0,1]$, where $p(\x_0)$ denotes the data and $p(\x_1)$ a terminal distribution, typically chosen to be Gaussian. A DM, parameterized by \(\btheta\), is trained to approximate the score function.  

\Cref{eq:linear_forward_sde} defines the \emph{forward} process, which gradually corrupts data samples into Gaussian noise, whereas \cref{eq:linear_reverse_sde} defines the \emph{reverse} process, which generates the data through iterative denoising. The Probability Flow ODE (PF-ODE) is the deterministic counterpart of \cref{eq:linear_reverse_sde} given by
% \begin{equation}
    \(\fode \triangleq \frac{\diff \mathbf{x}_t}{\diff t } = [\ldriftcoef - \frac{1}{2}\ldiffcoef \ldiffcoef^\top \score].\)
% \end{equation}
Substituting the \(\btheta\)-parameterized model's approximated score, we define the \emph{generator} as 
\begin{equation}
\generator(\x_1)=\x_1 - \int_0^1 \fodet\diff t,
\end{equation} 
denoting the model-based mapping of a sample \(\x_1\sim p(\x_1)\) to the approximate data distribution \(p_{\btheta}(\x_0)\) given by the DM. At the center of our reasoning, we rely on the following assumption.
\begin{assumption}\label{ass:stratified_manifold}
   We follow the stratified manifold hypothesis \citep{goresky1988stratified}, which posits that high-dimensional data resides on a disjoint collection of low-dimensional submanifolds, and assume that a DM is able to learn a statistical approximation of these strata. 
\end{assumption}
Based on \cref{ass:stratified_manifold}, we define \(
\manifoldt = \left\{ \x_0 \ \middle| \ \exists_{\x_1 \in \mathbb{R}^n} \x_0 = \generator(\x_1) \right\}
\), the set of all possible outputs of the generator, and refer to it as the model-induced manifold for simplicity. This construction is justified by foundational works and recent research on the geometry of DMs \citep{fefferman2016testing,pidstrigach2022score,pmlr-v235-stanczuk24a}.

Training DMs relies on a tractable formula for the forward kernel \(p(\x_t\mid\x_0) =\mathcal{N}(\mathbf{H}_t\x_0, \mathbf{\Sigma}_t)\), where \(\mathbf{H}_t = \mathbf{\Phi}(t, 0)\) with \(\mathbf{\Phi}(t, s) = \exp\left(\int_s^t \mathbf{F}_u \diff u\right)\), assuming $\mathbf{F}_t$ commutes for all $t$, and \(\mathbf{\Sigma}_t = \int_0^t \mathbf{\Phi}(t, \tau) \mathbf{G}_\tau \mathbf{G}_\tau^\top \mathbf{\Phi}(t, \tau)^\top \diff\tau\). To train a DM, one can either use $\scorec$ as target or utilize equivalent objectives such as
\begin{align}
\mathcal{L}_{\text{DSM}}(\x_0, t, \btheta) &\triangleq  \expval_{\vareps \sim \mathcal{N}(\mathbf{0}, \mathbf{I})} \left[ \norm{\vareps - \vareps_{\btheta}(\x_t)}_2^2 \right], \label{eq:l_dsm} \\
\mathcal{L}_{\text{ISM}}(\x_0, t, \btheta) &\triangleq \expval_{\vareps \sim \mathcal{N}(\mathbf{0}, \mathbf{I})} \left[ \Tr{(\boldsymbol{\Sigma}_t \nabla_{\x_t} \mathbf{s}_{\btheta}(\x_t))} + \frac{1}{2} \norm{\mathbf{s}_{\btheta}(\x_t)}_{\boldsymbol{\Sigma}_t}^2 \right]\label{eq:l_ism}
\end{align}
where \(\x_t=\mathbf{H}_t\x_0 + \boldsymbol{\Sigma}_t^{\frac{1}{2}}\vareps\), in each case marginalizing over $\x_0$ and $t$ to learn the underlying score function \citep{JMLR:v6:hyvarinen05a,vincent2011connection,yeats2025connection}. We refer to \(\vareps_{\boldsymbol{\theta}}\) as the noise prediction and \(\mathbf{s}_{\boldsymbol{\theta}}\) as the score prediction. 

Thanks to the Tweedie's formula \citep{efron2011tweedie,alain2014regularized}, one can directly relate the score function to the mean of the posterior \(p(\x_0\mid\x_t)\) through
\begin{equation}
\hat{\x}_0(\x_t) \triangleq \mathbb{E}[\x_0 \mid \x_t] = \mathbf{\Phi}(t, 0)^{-1} \left( \x_t + \mathbf{\Sigma}_t \score \right).\label{eq:tweedie}
\end{equation}
We denote by \(\pmptheta\) the result of replacing the true score function in \cref{eq:tweedie} with \(\mathbf{s}_{\btheta}\).

Arising from the manifold hypothesis, \(\lid\) \citep{levina2004maximum,bengio2013representation} encodes the local tangent space dimensionality, representing the valid degrees of freedom within a data stratum \citep{pope2021the}. Historically a diagnostic tool in early deep learning research \citep{ma2018characterizing,ansuini2019intrinsic}, \(\lid\) can now be natively estimated by DMs. This capability, which justifies \cref{ass:stratified_manifold}, relies on \(\pmptheta\) acting as an approximate orthogonal projector for small \(t\). While early DM-based estimators were computationally heavy \citep{tempczyk2022lidl,pmlr-v235-stanczuk24a,kamkari2024geometric}, \citet{yeats2025connection} demonstrated that standard training losses (\cref{eq:l_dsm,eq:l_ism}) provide efficient upper bounds of \(\lid\).
\section{Empirical Investigation}\label{sec:empirical_investigation}
The initial hypothesis given by \citet{aithal2024understanding} suggested that DMs' hallucinations stem from interpolating between the modes of the data distribution and thus amplifying the regions of low probability. To filter out hallucinations post-generation, the authors proposed the Trajectory Variance Filter (\(\tvf\)) defined as \( 
    \tvf(\x_0) = \int_{t_1}^{t_2} \norm{\pmptheta - \overline{ \hat{\x}^{\btheta}_{0,t_1:t_2}}}_2^2 \diff t,
\)
where \(\overline{ \hat{\x}^{\btheta}_{0,t_1:t_2}}\) is the average predicted posterior mean over the \([t_1,t_2]\) interval. This represents the \emph{temporal} view, which is used to analyze hallucinations using the intermediate properties of the reverse process. We propose to shift that perspective and instead focus on the \emph{spatial} view by analyzing the local geometry of the DM on the induced manifold.

We define the Local Manifold Instability ($\lmi$) of $\generator$ for an initial point $\x_1$ as
\begin{equation}
    \lmi(\x_1)\triangleq \Tr(\cov_{\boldsymbol{\varepsilon}}(\generator(\x_1+\boldsymbol{\varepsilon})))=\var_{\boldsymbol{\varepsilon}}(\generator(\x_1+\boldsymbol{\varepsilon})),\label{eq:lmi}
\end{equation}
where $\boldsymbol{\varepsilon}\sim\mathcal{N}(\mathbf{0}, \beta^2\mathbf{I})$ for some small $\beta>0$, often also referred to as \emph{local sensitivity} \citep{cacuci2005sensitivity}. Intuitively, $\lmi$ measures the total spatial spread of a small spherical region of noise after it is transported to the data space, providing a proxy for exploring the local geometry of a given DM.

\textbf{Geometric perspective.} We begin with a core question: \emph{If structural hallucinations can be partially captured by temporal variance during the reverse process, do these failure modes also manifest as localized spatial instabilities on the model-induced manifold?} To assess it, we compare the performance of \(\tvf\) and \(\lmi\) as hallucination filters on a benchmark of hand images from the \texttt{11kHands} dataset \citep{afifi201911kHands}, where they are easily recognizable as hands with deformed, missing or additional fingers. We use the capable EDM model \citep{karras2022elucidating} trained by \citet{tian2025rods} with a deterministic 40-step Euler solver to generate 128 samples. To decide whether a given sample is correct or represents a structural hallucination, we conduct a user study with independent human annotators to label the samples in a binary manner.
For \(\tvf\), we collect \(\pmptheta\) for each \(t\) throughout generation and then optimize \(t_1\) and \(t_2\) for best performance. For \(\lmi\), we estimate it with 32 perturbations per sample for a chosen \(\beta\). For more details, see \cref{sec:user_study}.

\Cref{fig:filter_performance} depicts densities for both filters divided between correct and hallucinated samples, as well as classification and separability performance metrics. Interestingly, both filters are highly effective in separating correct and hallucinated samples, with a minor advantage of $\lmi$ despite choosing the optimal time interval for $\tvf$. While the practical applicability of \(\lmi\) is limited due to its computational costs based on running the reverse process multiple times, these results provide a novel identification of hallucinations as unstable states on the model-induced manifold. We extend this experiment to other datasets in \cref{sec:filter_performance}, showing that while the performance of these two filters varies, the relationship between them is preserved.

To better understand the predictive ability of \(\lmi\), we provide a proposition that connects it with the \(\lid\) of the generated sample on the model-induced manifold, which we denote as \(\lidt\).

\begin{proposition}\label{prop:lmi_ub}
Let \(\x_0=\generator(\x_1)\) and assume a sufficiently small perturbation scale \(\beta > 0\) such that a first-order linear approximation holds. Let \(\sigma_1 \geq \sigma_2 \geq \dots \geq \sigma_n \geq 0\) denote the singular values of the generator's Jacobian \(\nabla_{\x_1}\generator(\x_1)\), where \(n\) is the ambient dimension. Then it holds that
\begin{equation}
    \lmi(\x_1) \approx \beta^2 \sum_{i=1}^{\lidt(\x_0)} \sigma_i^2.
\end{equation}
\end{proposition}
We provide the proofs and full formal statements of all theoretical results in \cref{sec:theoretical_results}.

\Cref{ass:stratified_manifold} implies that \(\lidt(\x_0)\) measures the \emph{effective} dimensionality (significantly non-zero singular values) of \(\x_0\) at \(\manifoldt\) when considering a theoretically full-rank \(\generator\). \Cref{prop:lmi_ub} hence instantly reveals the reasons for the inflated \(\lmi\): it is caused by \textbf{R1.} over-estimating the true singular values (of the true stratified manifold), thus resulting in more rapid but plausible modifications under perturbation, or \textbf{R2.} the invention of spurious, off-manifold directions that influence the generated sample in an implausible or artifactual manner.

With \cref{prop:lidt_estimators}, we extend the result of \citet{yeats2025connection} to show that the losses in \cref{eq:l_dsm,eq:l_ism}, shown to upper-bound the true \(\lid\), also form unbiased estimators of \(\lidt\).

\begin{proposition}\label{prop:lidt_estimators}
Let \(\x_0=\generator(\x_1)\). Under mild regularity conditions, it is true that
\(
\mathcal{L}_{\text{DSM}}(\x_0, t, \btheta) = \lidt(\x_0),
\)
\(
\mathcal{L}_{\text{ISM}}(\x_0, t, \btheta) = -\frac{1}{2}(n - \lidt(\x_0)),\)
where \(n\) is the dimension of the ambient space and \(t\) is sufficiently small.
\end{proposition}
Hence, \cref{prop:lidt_estimators} provides a tool to investigate the individual influences of \textbf{R1.} and \textbf{R2.} on the inflation of \(\lmi\) by estimating \(\lidt\) for both correct and hallucinated samples. Crucially, it also results in a continuous relaxation of the problem of estimating \(\lidt\), a theoretically discrete quantity.

Under the same experimental setup, we compare the performance of \(\lidt\) as a hallucination filter with \(\tvf\) and \(\lmi\) using \(\mathcal{L}_{\text{DSM}}\)-based estimator (\cref{eq:l_dsm}) with 32 noise samples for approximating the expectation. \Cref{fig:filter_performance} (right) first shows how this performance depends on \(t\), revealing large variability and the importance of choosing a small enough timestep, following the assumption of \cref{prop:lidt_estimators}. Crucially, \cref{fig:filter_performance} (middle) shows that \(\lidt\) significantly outperforms both \(\lmi\) and \(\tvf\) using the optimal \(t\), providing the best separability and classification performance with higher values indicating hallucinations. 

From a geometric perspective, this result reveals the primary origin of the local instabilities; hallucinations occur when the DM overestimates the complexity of the generated samples by inflating their \(\lidt\). This result suggests that hallucinations appear when the model-induced manifold possesses unnecessary directions of expansion, which represent illogical degrees of freedom, such as a varying number of fingers or the possibility of more than one thumb (\textbf{R2.}). This also explains the relatively worse performance of \(\lmi\), which is simultaneously inflated when the variability along entirely valid directions increases (\textbf{R1.}).

\begin{figure}
    \centering
    \includegraphics[width=0.99\linewidth]{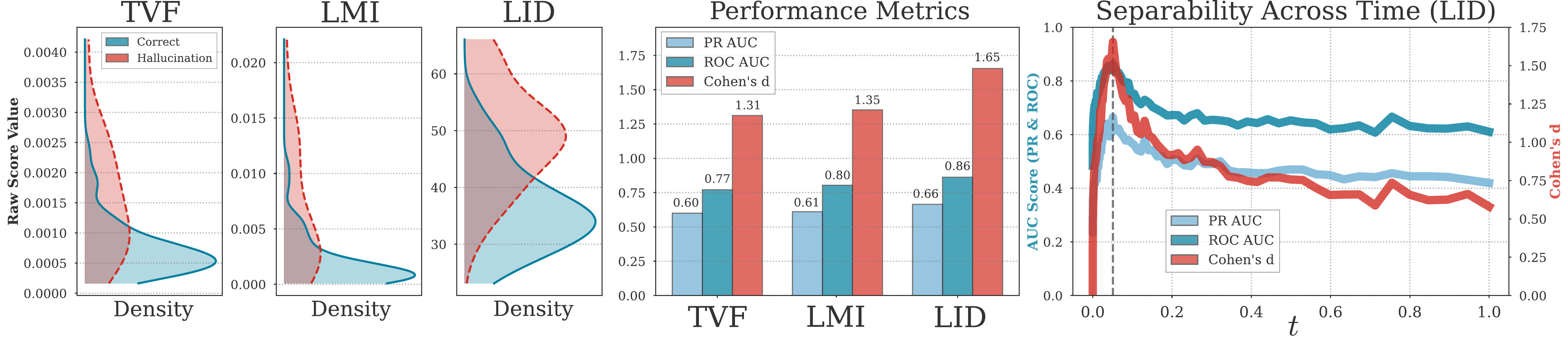}
    \caption{\textbf{Left.} Estimated densities for the values of each filter across correct and hallucinated samples. \textbf{Middle.} Separability and classification performance metrics for each filter. \textbf{Right.} Evolution of \(\lid\) filter performance across the time interval. Vertical line indicates the optimal timestep. All results are reported on \texttt{11kHands} dataset with EDM.}
    \label{fig:filter_performance}
\vspace{-1em}
\end{figure}
\vspace{-0.5em}
\section{Theoretical Foundation}\label{sec:theoretical_foundation}
\vspace{-0.5em}
While the empirical evidence justifies a hallucination filter based on \(\lidt\), filters themselves remain only a partial solution to the hallucination problem, as they inevitably require generating multiple invalid samples. In what follows, we propose a corrective mechanism for the sampling process inspired by the evidence that one should deflate \(\lidt\) for particular samples to avoid hallucinations.

\begin{theorem}\label{th:lmi_ub_tau}
Let \(\x_0=\generator(\x_1)\). Assume that \(\generator(\x_1)\) is decomposed as \(\generator = \generator^{\leq\tau} \circ \generator^{>\tau}\) for sufficiently small time \(\tau > 0\), where \(\generator^{>\tau}: \x_1 \mapsto \x_\tau\) and \(\generator^{\leq\tau}: \x_\tau \mapsto \x_0\). Let the singular values of the macroscopic flow Jacobian \(\nabla_{\x_1}\generator^{>\tau}(\x_1)\) be monotonically ordered as \(\sigma_1^{>\tau} \geq \sigma_2^{>\tau} \geq \dots \geq \sigma_n^{>\tau} \geq 0\). For sufficiently small \(\beta>0\), it is true that
\begin{equation}
    \lmi(\x_1) \lessapprox \beta^2 \sum_{i=1}^{\lidt(\hat{\x}_0^{\boldsymbol{\theta}}(\x_\tau))} (\sigma_i^{>\tau})^2.
\end{equation}
\end{theorem}

The contribution of \cref{th:lmi_ub_tau} is two-fold. First, it preserves the connection between local instability and the sample's \(\lidt\), while only requiring the posterior mean \(\hat{\x}_0^{\boldsymbol{\theta}}(\x_\tau)\) for some small \(\tau>0\) instead of the final \(\x_0\). Second, it replaces the global singular values from \cref{prop:lmi_ub} with a spectral bottleneck: a summation where the macroscopic variance (\(t>\tau\)) is strictly truncated by the intrinsic dimensionality of the terminal projection (\(t\leq\tau\)). Therefore, actively steering the trajectory to minimize \(\lidt(\hat{\x}_0^{\boldsymbol{\theta}}(\x_\tau))\) dynamically drops variance terms from this summation, provably decreasing the upper bound of \(\lmi\) for the resulting sample. We formalize the implied sampling process below.

\begin{theorem}\label{th:boltzmann}
Assume $t \leq \tau$, where $\tau>0$ is sufficiently small. Replacing the standard trained score function $\mathbf{s}_{\btheta}(\mathbf{x}_t)$ in the reverse process (\cref{eq:linear_reverse_sde}) with
\begin{equation}\label{eq:modified_score}
    \tilde{\mathbf{s}}_{\btheta}(\mathbf{x}_t) = \mathbf{s}_{\btheta}(\mathbf{x}_t) - \lambda_t \nabla_{\mathbf{x}_t} \mathcal{E}(\mathbf{x}_t),
\end{equation}
where the energy \(\mathcal{E}(\mathbf{x}_t)=\lidt(\pmptheta)\), is equivalent to sampling from the ideal Boltzmann distribution over the true terminal states, $p_t^{\btheta, \lambda_t}(\mathbf{x}_t) \propto p_t^{\btheta}(\mathbf{x}_t) \mathbb{E}_{p^{\btheta}(\mathbf{x}_0 | \mathbf{x}_t)}[\exp(-\lambda_t \lidt(\mathbf{x}_0))]$.
\end{theorem}

\Cref{th:boltzmann} guarantees that, for small enough \(t\), correcting the reverse process with the gradient of \(\lidt\) for the posterior mean \(\pmptheta\) moves the resulting sample along \(\manifoldt\) towards terminal states of lower \(\lidt\), \ie, a lower-dimensional stratum. We obtain the gradient through \(\nabla_{\mathbf{x}_t} \lidt(\pmptheta) \approx \nabla_{\mathbf{x}_t}\mathcal{L}_{\text{DSM}}(\pmptheta, t, \btheta)\), where the approximation comes from estimating the DSM expectation with \(k\) i.i.d noise samples.

\textbf{Probabilistic perspective.} The resulting correction represents a projected gradient descent scheme on \(\manifoldt\). However, utilizing the \(\mathcal{L}_{\text{DSM}}\) estimator does not immediately reveal the probabilistic effect of that correction on the generated sample \(\x_0\). To address that, we provide \cref{cor:collinearity} showing that one may interchangeably use either the \(\mathcal{L}_{\text{DSM}}\)- or \(\mathcal{L}_{\text{ISM}}\)-based estimator.

\begin{corollary}\label{cor:collinearity}
    The gradients \(\nabla_{\mathbf{x}_t}\mathcal{L}_{\text{DSM}}(\pmptheta, t, \btheta)\) and \(\nabla_{\mathbf{x}_t}\mathcal{L}_{\text{ISM}}(\pmptheta, t, \btheta)\) are exactly collinear and \(\nabla_{\mathbf{x}_t}\mathcal{L}_{\text{DSM}}(\pmptheta, t, \btheta)=2\nabla_{\mathbf{x}_t}\mathcal{L}_{\text{ISM}}(\pmptheta, t, \btheta)\).
\end{corollary}

This collinearity allows us to reason about the probabilistic effect of the stable and computationally cheap \(\mathcal{L}_{\text{DSM}}\) gradient through the theoretical lens of \(\mathcal{L}_{\text{ISM}}\).

\begin{corollary}\label{cor:mode_seeking}
    Assuming $\mathbf{s}_{\btheta}(\x_t) = \nabla_{\x_t} \log p_t^{\btheta}(\x_t)$, \cref{eq:l_ism} can be equivalently reformulated as \(\mathcal{L}_{\text{ISM}}(\x_0, t, \btheta) = \expval_{\vareps \sim \mathcal{N}(\mathbf{0}, \mathbf{I})} \left[ \Tr{(\boldsymbol{\Sigma}_t \nabla_{\x_t}^2 \log p_t^{\btheta}(\x_t))} + \frac{1}{2} \norm{\mathbf{s}_{\btheta}(\x_t)}_{\boldsymbol{\Sigma}_t}^2 \right]\). Thus, as \(t \to 0\), minimizing the energy from \cref{eq:modified_score} steers towards stationary points (\(\nabla_{\x_t} \log p_t^{\btheta}(\x_t) \to \mathbf{0}\)) with maximal negative curvature (\(\Tr{(\boldsymbol{\Sigma}_t \nabla_{\x_t}^2 \log p_t^{\btheta}(\x_t))}\ll 0\)) of the model-induced probability distribution \(p_t^{\btheta}(\x_t)\), inducing a mode-seeking behavior.
\end{corollary}

By mitigating mode interpolations, \cref{cor:mode_seeking} directly connects our optimization scheme with the work of \citet{aithal2024understanding}. As the proposed correction essentially aims at eliminating unstable, low-probability states by `cooling' them and shifting towards stable local maxima, we term our approach as \emph{Intrinsic Quenching} (\(\iq\)).
\begin{figure}
    \centering
    \includegraphics[width=0.99\linewidth]{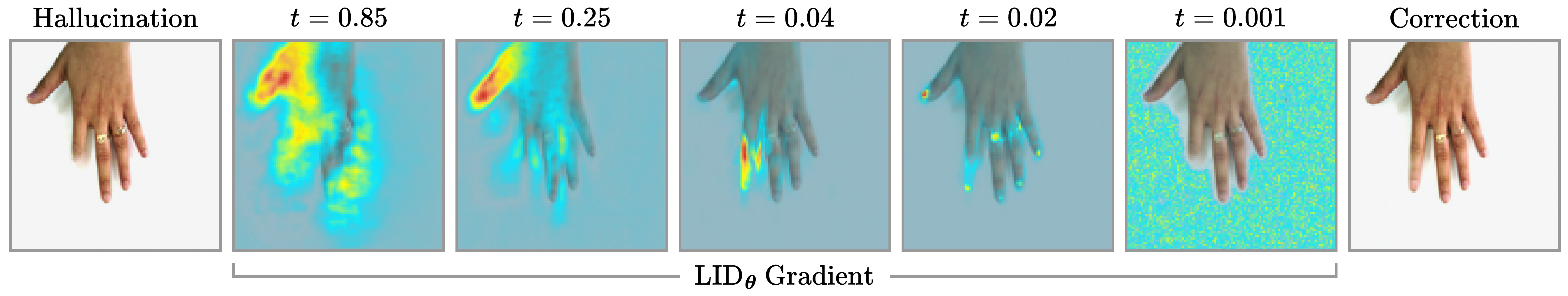}
    \caption{\emph{Hallucination} depicts the initial unconditionally generated sample by EDM on \texttt{11kHands}. For a set of manually selected timesteps, \(\lidt\) gradients are visualized to highlight their coarse-to-fine transition and perceptual connection to hallucinated content. \emph{Correction} displays the sample generated with \(\iq\), the proposed corrective mechanism applied in the \([0.02, 0.05]\) interval.}
    \label{fig:gradient_evolution}
\vspace{-1.5em}
\end{figure}

\textbf{From theory to practice.} Both the theoretical assumptions and the separability performance from \cref{fig:filter_performance} indicate that \(\iq\) should be only used for small times \(t<\tau\). We thus apply it solely within a narrow interval \([t_1,t_2]\) of the reverse process, where \(t_2\) is small. This focused intervention is also empirically justified when inspecting the energy gradients \(\nabla_{\x_t} \mathcal{E}(\x_t)\) (\cref{fig:gradient_evolution}). By overlaying them over the final generated image, we discover a smooth transition from global, structural attributes through localized, highly-specific features correlated with the hallucinatory content, up to dispersed but cohesive states. This observation naturally aligns with multiple prior studies that decompose the diffusion process into semantically interpretable phases \citep{park2023understanding,sclocchi2025phase,wanganalytical}. In this case, however, \(\lidt\) provides an inherent attribution map \citep{10.5555/3295222.3295230,sundararajan2017axiomatic} that allows quantifying such effects, as it indicates pixels with the highest contribution to the model's error at time \(t\).

To ensure that the modified updates behave stably, we compute the \(\lambda_t\) scale dynamically for each \(t\) by ensuring that the magnitude of the energy term projected into data space (\(\norm{\lambda_t \mathbf{\Phi}(t, 0)^{-1} \mathbf{\Sigma}_t \nabla_{\x_t} \mathcal{E}(\x_t)}_2\)) always equates to a fixed ratio \(\lambda\) of the natural update. For example, in the EDM framework, where the network is trained directly to predict \(\pmptheta\) (and \(\mathbf{\Phi}(t, 0) = \mathbf{I}\)), we use \(\lambda_t = \lambda \frac{\norm{\pmptheta - \x_t}_2}{\norm{\mathbf{\Sigma}_t \nabla_{\x_t} \mathcal{E}(\x_t)}_2 + \epsilon}\), where \(\epsilon\) ensures numerical stability.

The DM's score function naturally interacts with the energy gradient from \(\iq\); the former pushes the sample towards \(\manifoldt\) with decreasing \(t\), while the latter shifts the sample along the manifold to regions of lower \(\lidt\), suppressing the most volatile directions. To ensure that this intervention selectively targets hallucinations rather than uniformly restricting all samples, we employ \(\lidt\) as a dynamic, mid-generation filter. Specifically, we disable the correction (\(\lambda_t=0\)) for stable samples satisfying \(\mathcal{E}(\xt) < q_t\), where \(q_t\) is a time-dependent threshold. To calibrate \(q_t\), we unconditionally generate a reference set of samples from the DM, precompute \(\mathcal{E}(\xt)\) for each \(\xt\) across the active interval \([t_1, t_2]\), and set \(q_t\) to the \(q\)-th percentile of these empirical energy values at each timestep. We provide the pseudocode of \(\iq\) in \cref{alg:iq}.
\vspace{-0.5em}
\section{Related Works}
\textbf{Hallucination reduction.} Observed since early natural image DMs \citep{ramesh2021zero,rombach2022high}, hallucinations were initially hypothesized by \citet{aithal2024understanding} as effects of mode interpolations, exacerbating errors during iterative DM training. Subsequent mitigations \citep{fu2025counting,cho2025tag,bhosale2026varianceguided,chandran2025laplacian} include Adaptive Attention Modulation \citep[AAM;][]{oorloff2025mitigating}, which optimizes attention temperatures \citep{vaswani2017attention} via an anomaly detector \citep{roth2022towards}, requiring the original training data. Dynamic Guidance \citep[DG;][]{triaridis2025mitigating} applies Classifier Guidance \citep[CG;][]{dhariwal2021beat} toward the most probable class at each step, necessitating an independently trained noisy classifier. More recently, \citet{tian2025rods} (RODS) formulate DMs as a continuation method \citep{allgower2012numerical}, mitigating hallucinations by intervening intrinsically upon detecting vector field instabilities.

\textbf{Task-specific methods.} Because deformed hands are a prominent failure mode, several methods aim to correct or filter them specifically \citep{narasimhaswamy2024handiffuser,lu2024handrefiner,Shi_Guo_Shui_Chen_Shen_2026}. Similarly, \citet{lu2025towards} identify local generation bias causing visual text failures. Hallucinations pose critical risks in image translation and reconstruction by affecting real-world decisions. Addressing this, \citet{kim2024tackling} propose local diffusion for conditional translation, partitioning in- and out-of-distribution regions. In reconstruction, several works introduce quantitative hallucination scores \citep{tivnan2024hallucination,ren2025hallucination,10981211}. Distinctively, \citet{cao2025temporal} employ temporal score analysis, akin to \citet{aithal2024understanding} and \citet{tian2025rods}, to eliminate related out-of-distribution artifacts.

\section{Experiments}\label{sec:experiments}
\begin{table}[htbp]
  \caption{Quantitative comparison of all sampling methods across six datasets. We highlight the primary performance indicators, based on human evaluation, regarding perceived image quality (\colorbox{LightOrange}{UP}) and hallucination ratio (\colorbox{LightGreen}{HR}), reported as averages across annotators. For 95\% CIs, see \cref{tab:appendix_ci_metrics}.}
  \label{tab:main_results}
  \centering
  
  \setlength{\tabcolsep}{4pt} % Base padding
  
  % Measure the top table to force the bottom table to match its exact width
  \newsavebox{\tophalftable}
  \sbox{\tophalftable}{%
    \begin{tabular}{l ccc >{\columncolor{LightOrange}}c >{\columncolor{LightGreen}}c 
                          ccc >{\columncolor{LightOrange}}c >{\columncolor{LightGreen}}c 
                          ccc >{\columncolor{LightOrange}}c >{\columncolor{LightGreen}}c}
      \toprule
      & \multicolumn{5}{c}{\texttt{11kHands}} & \multicolumn{5}{c}{\texttt{FFHQ}} & \multicolumn{5}{c}{\texttt{AFHQV2}} \\
      \cmidrule(lr){2-6} \cmidrule(lr){7-11} \cmidrule(lr){12-16}
      Method & FID $\downarrow$ & IV $\uparrow$ & DSV $\uparrow$ & \cellcolor{LightOrange}UP $\uparrow$ & \cellcolor{LightGreen}HR $\downarrow$ & FID $\downarrow$ & IV $\uparrow$ & DSV $\uparrow$ & \cellcolor{LightOrange}UP $\uparrow$ & \cellcolor{LightGreen}HR $\downarrow$ & FID $\downarrow$ & IV $\uparrow$ & DSV $\uparrow$ & \cellcolor{LightOrange}UP $\uparrow$ & \cellcolor{LightGreen}HR $\downarrow$ \\
      \midrule
      Baseline                  & 16.3 & 0.032 & 0.15 & 39.8 & 29.3 & 13.5 & 0.067 & 0.40 & 45.3 & 8.2  & 16.7 & 0.054 & 0.48 & 41.8 & 6.9 \\
      DG                        & 16.2 & 0.033 & 0.15 & 39.5 & 29.7 & 13.8 & 0.070 & 0.40 & 44.1 & 9.8  & 16.7 & 0.054 & 0.48 & 40.2 & 8.6 \\
      AAM                       & 16.4 & 0.033 & 0.15 & 40.6 & 29.3 & 13.5 & 0.068 & 0.40 & 45.7 & 10.2 & 17.6 & 0.055 & 0.49 & 42.2 & 14.3 \\
      RODS\textsuperscript{CAS} & 15.8 & 0.033 & 0.15 & 40.2 & 25.8 & 13.6 & 0.066 & 0.40 & 45.3 & 7.7  & 16.9 & 0.054 & 0.48 & 42.2 & 6.9 \\
      RODS\textsuperscript{SAS} & 16.2 & 0.033 & 0.15 & 41.0 & 29.7 & 13.6 & 0.066 & 0.40 & 45.3 & 8.2  & 16.8 & 0.054 & 0.48 & 42.2 & 6.6 \\
      \(\iq\)                       & 16.6 & 0.030 & 0.13 & 68.0 & 9.0  & 13.9 & 0.065 & 0.39 & 46.1 & 4.2  & 17.0 & 0.054 & 0.48 & 42.6 & 5.9 \\
      \bottomrule
    \end{tabular}%
  }

  % Resize the unified block to the text width
\resizebox{\textwidth}{!}{%
    \begin{tabular}{@{}c@{}}
  
    % --- TOP HALF ---
    \usebox{\tophalftable} \\
    
    % --- SPACING BETWEEN HALVES ---
    \\[1ex]
    
    % --- BOTTOM HALF (Expanded to top half's width) ---
    \begin{tabular*}{\wd\tophalftable}{l @{\extracolsep{\fill}} ccc >{\columncolor{LightGreen}}c ccc >{\columncolor{LightGreen}}c c >{\columncolor{LightGreen}}c}
      \toprule
      & \multicolumn{4}{c}{\texttt{MNIST}} & \multicolumn{4}{c}{\texttt{SimpleShapes}} & \multicolumn{2}{c}{\texttt{GaussianGrid}} \\
      \cmidrule(lr){2-5} \cmidrule(lr){6-9} \cmidrule(lr){10-11}
      Method & FID $\downarrow$ & IV $\uparrow$ & DSV $\uparrow$ & \cellcolor{LightGreen}HR $\downarrow$ & FID $\downarrow$ & IV $\uparrow$ & DSV $\uparrow$ & \cellcolor{LightGreen}HR $\downarrow$ & MMD $\downarrow$ & \cellcolor{LightGreen}HR $\downarrow$ \\
      \midrule
      Baseline                  & 32.3 & 0.050 & 0.13 & 37.3 & 27.5 & 0.031 & 0.11 & 25.8 & 0.010 & 20.2 \\
      DG                        & 32.1 & 0.050 & 0.14 & 23.6 & 27.8 & 0.031 & 0.11 & 27.3 & 0.080 & 11.1 \\
      AAM                       & 32.3 & 0.049 & 0.13 & 34.8 & 27.2 & 0.025 & 0.06 & 55.5 & -     & - \\
      RODS\textsuperscript{CAS} & 35.4 & 0.049 & 0.13 & 41.8 & 28.1 & 0.031 & 0.10 & 30.5 & 0.016 & 10.4 \\
      RODS\textsuperscript{SAS} & 34.8 & 0.047 & 0.13 & 55.3 & 27.8 & 0.031 & 0.11 & 28.9 & 0.016 & 19.4 \\
      \(\iq\)                       & 31.8 & 0.051 & 0.14 & 10.2 & 23.3 & 0.033 & 0.13 & 9.4  & 0.016 & 8.9 \\
      \bottomrule
    \end{tabular*}
    
  \end{tabular}%
  }
\vspace{-1em}
\end{table}
\textbf{Datasets.} We perform a large-scale quantitative comparison of \(\iq\) with recent methods for hallucination reduction by unifying several existing benchmarks. As a tractable toy problem, we use the \texttt{GaussianGrid} dataset proposed by \citet{aithal2024understanding}, comprising a 2D mixture of uniformly spaced Gaussians in a form of rotated grid with 25 modes. For semantic structural hallucinations, we evaluate on \texttt{SimpleShapes} and \texttt{MNIST} for synthetic images following \citet{aithal2024understanding,triaridis2025mitigating,oorloff2025mitigating}, and animal faces \citep[\texttt{AFHQV2}, \(64\times64\) resolution;][]{choi2020starganv2}, human faces \citep[\texttt{FFHQ}, \(64\times64\) resolution][]{karras2019style} and hand images \citep[\texttt{11kHands}, \(128\times128\) resolution;][]{afifi201911kHands} entirely following \citet{tian2025rods} for natural images.

\textbf{Models.} For \texttt{GaussianGrid}, we train a 3-layer MLP with hidden dimension and time embedding of size 256 based on DDPM \citep{ho2020denoising} with a linear schedule and 1000 steps. For \texttt{SimpleShapes}, we use a UNet \citep{ronneberger2015u} with 3 convolutional blocks and an attention layer for both encoder and decoder, train it using 1000-step DDPM with squared cosine scheduler \citep{nichol2021improved} and use 250-step sampling for inference. For \texttt{MNIST}, we use a pretrained UNet\footnote{\href{https://huggingface.co/1aurent/ddpm-mnist}{https://huggingface.co/1aurent/ddpm-mnist}} from HuggingFace \citep{von-platen-etal-2022-diffusers} also based on 1000-step DDPM and 250-step inference. For \texttt{AFHQV2} and \texttt{FFHQ}, we use unconditional VE EDM checkpoints from the original paper \citep{karras2022elucidating} and sample with a deterministic 40-step Euler solver, following \citet{tian2025rods}. For \texttt{11kHands}, we reuse the VE EDM network pretrained by \citet{tian2025rods}, also with 40-step sampling for inference.

\textbf{Baselines.} We compare \(\iq\) against recent methods for hallucination reduction: DG \citep{triaridis2025mitigating}, AAM \citep{oorloff2025mitigating} and RODS \citep{tian2025rods}. Importantly, DG requires labeled training data of the DM of interest and assumes access to a pretrained noisy classifier for CG. Similarly, AAM also assumes access to training data to obtain the anomaly detection model and is limited to architectures with attention layers. In this context, RODS and \(\iq\) remain the only methods with no additional assumptions about the DM of interest. Our comparisons also include the standard unmodified sampling as a default baseline. For details regarding the unification of methods and adaptation to our codebase, see \cref{sec:experimental_setup}.

\textbf{Metrics.} On \texttt{GaussianGrid}, we measure hallucination ratio (HR) as the percentage of samples exceeding 3 standard deviations from the closest mode \citep{aithal2024understanding}, alongside Maximum Mean Discrepancy (MMD) for distributional similarity. For images, we evaluate 2048 samples per method using proxy metrics: Fréchet Inception Distance \citep[FID;][]{heusel2017gans} for quality, Inception Variance \citep[IV;][]{miao2024training} and DreamSim Variance \citep[DSV;][]{domingo-enrich2025adjoint} for diversity \citep{tian2025rods}.

Because automated evaluations struggle with semantic perception, we quantify HR on image datasets via a blinded user study (\cref{sec:user_study}). Annotators first complete a \emph{calibration phase} viewing 128 true images to grasp dataset variability. In the subsequent \emph{labeling phase}, they identify structural hallucinations in a blinded manner on outputs generated by competing methods from identical latent points \(\x_1\). Furthermore, feature-based metrics notoriously misalign with human perception, often favoring texture over structure \citep{geirhos2018imagenet}, a flaw amplified in DMs where human preference remains the ultimate ground truth \citep{stein2023exposing,jayasumana2024rethinking}. Addressing this, we expand our study to evaluate User Preference (UP) on natural images. For each latent point, annotators select the highest-quality generations, with UP reported as the overall selection frequency per method.

\begin{figure}
    \centering
    \includegraphics[width=0.99\linewidth]{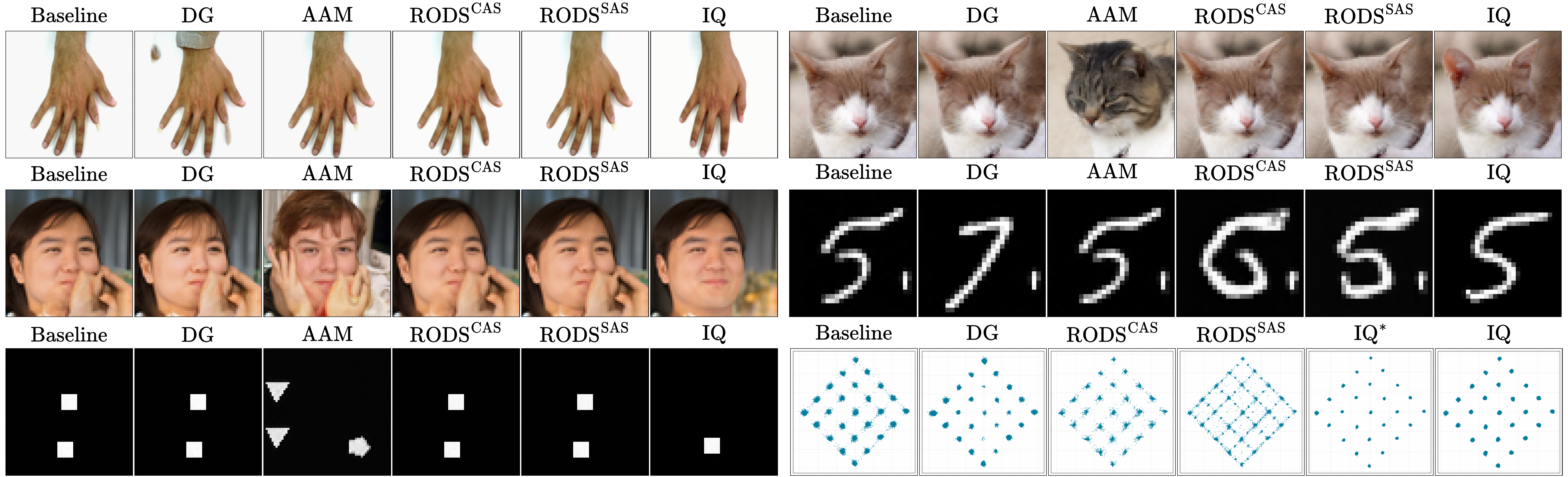}
    \caption{Qualitative comparison of all sampling methods across six datasets. AAM is omitted on \texttt{GaussianGrid} due to incompatible network architecture. \(\iq^*\) denotes a variant of \(\iq\) without filtering (\(q=0\)). For high-resolution results on \texttt{GaussianGrid}, see \cref{fig:hq_gauss_vis}.}
    \label{fig:qualitative_comparison}
\vspace{-2em}
\end{figure}

\textbf{Results.} Quantitative and qualitative results are summarized in \cref{tab:main_results} and \cref{fig:qualitative_comparison}. On \texttt{GaussianGrid}, \(\iq\) confirms its theoretical validity: it minimally deviates from the baseline while maximally reducing hallucinations by eliminating mode interpolations. \Cref{fig:qualitative_comparison} also highlights the role of dynamic filtering; without it, \(\iq\) over-squeezes samples into modes, whereas filtering halts correction appropriately. Crucially, \(\iq\) succeeds even when true data \(\lid\) equals the ambient dimension, confirming our probabilistic interpretation. Conversely, DG achieves comparable HR but samples unevenly due to classifier bias, while RODS\textsuperscript{SAS} fails to meaningfully reduce hallucinations despite matching RODS\textsuperscript{CAS} in MMD.

On \texttt{SimpleShapes}, \(\iq\) reduces over 60\% of hallucinations. We hypothesize the concurrent improvement in feature-based metrics stems from the specific nature of these failures: generating extra shapes creates natural outliers that inflate FID. Since these anomalies dominate baseline generation, their removal naturally increases valid diversity. This is corroborated on \texttt{MNIST}, where FID and diversity strongly correlate with HR. Despite extensive tuning, baselines struggle on these simple sets. This validates our \(\lid\) formulation: an optimal DM must bound its factors of variation, and while \(\iq\) directly suppresses unnecessary degrees of freedom, other methods cannot effectively enforce this constraint.

On natural images, where human-centric UP and HR are the primary performance indicators, differences in feature-based metrics shrink. Crucially, on \texttt{11kHands}, \(\iq\) reduces hallucinations by almost 70\% and achieves almost 70\% UP, vastly outperforming baselines. The accompanying minor regressions in FID and diversity corroborate findings \citep{stein2023exposing,jayasumana2024rethinking} that such metrics fail to accurately capture human preference. On \texttt{FFHQ} and \texttt{AFHQV2}, performance gaps narrow. While users preferred \(\iq\) slightly more, overlapping confidence intervals (CIs) indicate no statistically significant UP difference, which annotators also attributed to the datasets' lower resolutions. Nonetheless, \(\iq\) successfully cuts hallucinations by around 50\% on \texttt{FFHQ} and 15\% on the highly challenging \texttt{AFHQV2}, where \(\iq\)'s HR CI overlaps with RODS and baseline sampling despite a better average.

\textbf{Hallucinations in medical diagnosis.} Generative hallucinations pose severe risks in real-world decision-making, particularly in medical imaging where models restore high-frequency details to enable low-dose radiation \citep{wang2025generative}. To demonstrate \(\iq\)'s practical capabilities, we evaluate hallucination reduction in low-dose computed tomography (LDCT) reconstruction.

We formulate the inverse problem \(\mathbf{y} = \mathbf{A}\x + \sigma \boldsymbol{\epsilon}\), where \(\x\) represents ground truth \texttt{RSNA} brain CT scans \citep{rsna-intracranial-hemorrhage-detection}, \(\mathbf{A}\) is the forward projection matrix (X-ray line integrals), \(\boldsymbol{\epsilon}\sim\mathcal{N}(\mathbf{0}, \mathbf{I})\), and \(\mathbf{y}\) is the measurement sinogram. By zeroing the singular values of \(\mathbf{A}\) below a threshold \(\xi\) and adding noise, we simulate low-dose sparse-view CT. We employ System-Embedded Diffusion Bridges \citep[SDB;][]{sobieski2026systemembedded}, which solve linear Gaussian problems by embedding measurement parameters into the diffusion coefficients: \(\mathbf{H}_t\x=\mathbf{A}^+\mathbf{A}\mathbf{x} + \alpha_t(\mathbf{I} - \mathbf{A}^+\mathbf{A})\mathbf{x}\) and \(\boldsymbol{\Sigma}_t = \gamma_t\mathbf{A}^+\boldsymbol{\Sigma}{\mathbf{A}^{+}}^\top + \beta_t(\mathbf{I} - \mathbf{A}^+\mathbf{A})\), utilizing the Moore-Penrose pseudoinverse \(\mathbf{A}^+\). Because SDB maps measurements \(\mathbf{y}\) to reconstructions \(\mathbf{x}\) via a matrix-valued diffusion process, it allows for direct integration of the evaluated hallucination reduction methods. Throughout this experiment, we strictly adhere to the original SDB experimental setup on \texttt{RSNA} \citep{sobieski2026systemembedded}.

\begin{table}[t]
  \tiny
  \caption{Quantitative comparison of all sampling methods applied to SDB on the \texttt{RSNA} dataset. We include standard reconstruction metrics regarding perception and distortion. We highlight \colorbox{LightGreen}{mAP} and \colorbox{LightGreen}{mROC}, performance measures of the ResNet50 observer, as indicators of reconstruction errors and hallucinations, where higher values suggest improved anatomical consistency.}
  \label{tab:rsna_results}
  \centering
  \resizebox{\textwidth}{!}{%
  \begin{tabular}{l cccccc >{\columncolor{LightGreen}}c >{\columncolor{LightGreen}}c}
    \toprule
    Method & FID $\downarrow$ & IV $\uparrow$ & DSV $\uparrow$ & PSNR $\uparrow$ & SSIM $\uparrow$ & LPIPS $\downarrow$ & \cellcolor{LightGreen}mAP $\uparrow$ & \cellcolor{LightGreen}mROC $\uparrow$ \\
    \midrule
    Baseline & 33.3 & 0.066 & 0.264 & 33.54 & 0.898 & 0.0387 & 0.27 & 0.85 \\
    DG       & 35.6 & 0.066 & 0.264 & 33.55 & 0.899 & 0.0388 & 0.31 & 0.86 \\
    AAM      & 35.6 & 0.066 & 0.264 & 33.54 & 0.898 & 0.0388 & 0.29 & 0.86 \\
    RODS\textsuperscript{CAS} & 33.3 & 0.067 & 0.264 & 33.54 & 0.898 & 0.0388 & 0.29 & 0.86 \\
    RODS\textsuperscript{SAS} & 33.3 & 0.066 & 0.263 & 33.55 & 0.898 & 0.0387 & 0.29 & 0.86 \\
    \(\iq\)      & 33.3 & 0.066 & 0.264 & 33.54 & 0.901 & 0.0388 & 0.31 & 0.86 \\
    \bottomrule
  \end{tabular}%
  }
\vspace{-3em}
\end{table}

\textbf{Reducing misdiagnosis.} In this context, structural hallucinations manifest as mischaracterizations of the patient condition, \eg, fabricated synthetic pathology. To quantify diagnostic safety, we trained a ResNet50 observer to detect five hemorrhage types using 17,948 ground-truth \texttt{RSNA} slices. On a 512-slice validation split, the observer achieves 0.38 macro-averaged multi-label Average Precision (mAP) and 0.91 macro-averaged multi-label Area Under the ROC (mROC). When evaluated on baseline SDB reconstructions \citep{sobieski2026systemembedded}, performance drops to 0.27 mAP and 0.85 mROC, reflecting quality degradation and generative hallucinations.

We adapt all baseline mitigations to the SDB framework, tracking image quality and diversity (FID, IV, DSV), distortion (PSNR, SSIM, LPIPS) \citep{zhang2018unreasonable,blau2018perception}, and diagnostic safety (mAP, mROC) (\cref{tab:rsna_results}). Because SDB strictly preserves the measurement range space, image and reconstruction quality metric variations are heavily constrained to the null space and remain mostly flat, although DG and AAM noticeably degrade FID. Diagnostic performance, however, reveals a sharp distinction. While all methods marginally raise mROC to 0.86, \(\iq\) and DG increase mAP by roughly 15\% (reaching 0.31). Crucially, DG requires an independently trained classifier with direct access to ground-truth hemorrhage labels, granting it an inherent advantage. Despite operating entirely zero-shot, \(\iq\) matches DG's diagnostic gains while strictly preserving reconstruction quality. This underscores \(\iq\)'s universal applicability and viability for broader inverse problem frameworks \citep{chung2023diffusion,liu20232,luo2023image,Garber_2024_CVPR,yue2024image,zhoudenoising}.

\textbf{Mitigating hallucinations in foundation models.} Prior work on unsupervised hallucination in DMs \citep{oorloff2025mitigating,triaridis2025mitigating,tian2025rods} limited the evaluation to models trained on small-scale, homogeneous data, where hallucinations are more easily identifiable. A natural next step are DM-based foundation models (FMs) \citep{rombach2022high} trained on large-scale heterogeneous data, where hallucination mitigation offers direct practical returns. However, the high complexity of FMs training data makes evaluation much more difficult, as hallucinations may appear within varying contexts and in more diversified forms, which is why current literature lacks appropriate benchmarks. 

\begin{minipage}[t]{0.74\textwidth}
To at least partially address this gap, we perform a small-scale evaluation on Stable Diffusion XL \citep[SDXL;][]{podell2024sdxl}. We first manually search for a detailed prompt that results in SDXL generating realistic hand images, while also producing clear hallucinations (\eg, additional/missing fingers, nonanatomic shapes). We then reimplement both $\iq$ and RODS (as the most direct competitor) to work directly in SDXL's latent space. After an initial hyperparameter search, both methods and unconditional generation were evaluated by 3 human annotators on randomly chosen set of 64 generated images. \Cref{tab:sdxl_results} provides initial evidence that $\iq$ is once again able to mitigate hallucinations most effectively, reducing the baseline rate by around 30\%, with RODS being able to reduce only at most 13\%. Annotators reported no visible drop in quality across the entire comparison.
\end{minipage}
\hfill
\begin{minipage}[t]{0.25\textwidth}
\begin{center}
\vspace{-1em}
\captionof{table}{Quantitative small-scale comparison of $\iq$, RODS and unconditional generation on SDXL.}
\label{tab:sdxl_results}
\centering
\begin{tabular}{l c}
  \toprule
  Method & HR $\downarrow$ \\
  \midrule
  Baseline                    & 77.1 \\
  RODS\textsuperscript{SAS}   & 74.7 \\
  RODS\textsuperscript{CAS}   & 67.5 \\
  \(\iq\)                     & 53.1 \\
  \bottomrule
\end{tabular}
\end{center}
\end{minipage}

This study shows both the potential of transferring $\iq$ to large-scale models and highlights the need for further development, as many hallucinations remain. We attribute this gap primarily to the heterogeneity of the data. As the model-induced manifold gets more complex and diverse, we hypothesize that hallucinations no longer occupy a single range of superfluous dimensions, but instead make it highly dependent on the surrounding data and their own complexity. For example, the same value of $\lidt$ might both indicate a valid but complex generated sample (like a hand with a sophisticated pattern in the background) or an actual hallucination (additional finger in a hand placed in plain environment) due to high heterogeneity.

This line of reasoning is supported by our initial exploration into how $\iq$ transfers across different hallucination types in SDXL. \Cref{fig:qualitative_sdxl} displays pairs of originally hallucinated samples and their corresponding $\iq$ corrections, where for each pair a different hyperparameter configuration was used. Crucially, we were not able to identify a single setting that worked universally across all cases, but were able to easily tune $\lambda,t_1$ and $t_2$ for each case individually. These hyperparameters translate directly into how the underlying geometry is affected, highlighting that it differs across different data subdomains, but is influenced in the same way -- the spurious dimensions are stripped, which corrects initially wrong samples. It is the heterogeneity of this data that forces per-case tuning, showcasing an important future work direction -- once an objective fingerprint of hallucinations across different subdomains is found, it can be used to decide automatically on where ($t_1,t_2$) and how heavily ($\lambda$) a method such as $\iq$ should be applied.

\begin{figure}
    \centering
    \includegraphics[width=0.99\linewidth]{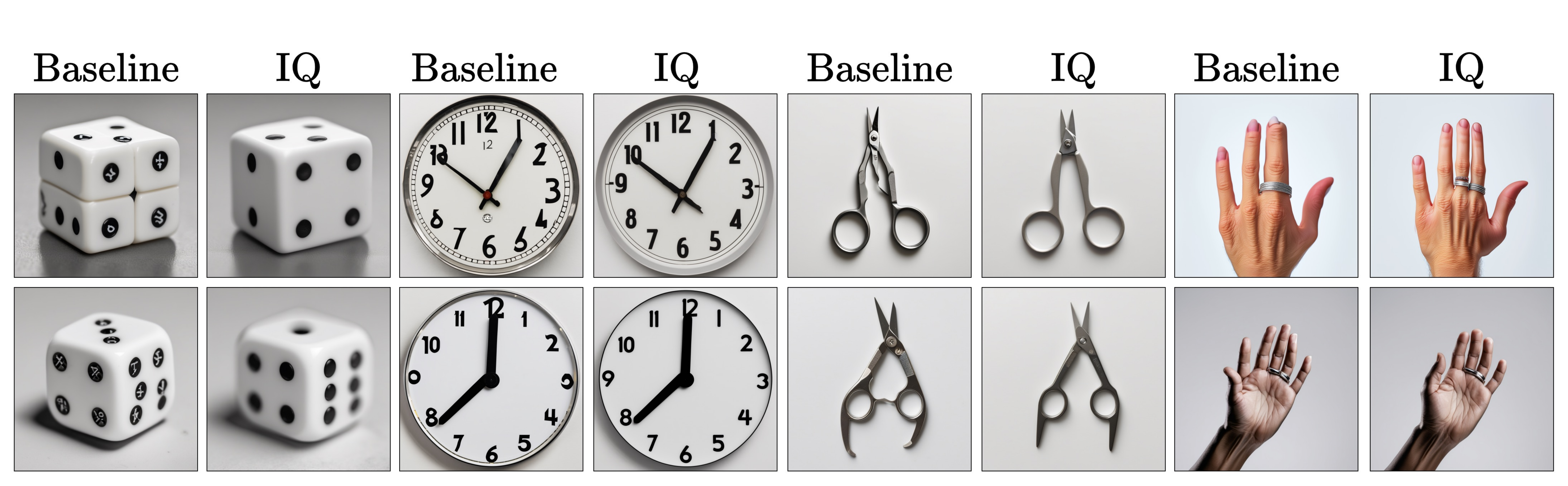}
    \caption{Qualitative comparison of hallucinated samples and their corresponding $\iq$ corrections on SDXL. For each pair, a different setting of $\lambda,t_1$ and $t_2$ was used, with other hyperparameters fixed.}
    \label{fig:qualitative_sdxl}
    \vspace{-1em}
\end{figure}

\textbf{Ablation studies.} Due to limited space, we continue the experimental evaluation in \cref{sec:extended_experiments}, providing ablation studies for each component of \(\iq\) (\cref{sec:ablation_studies}), runtime comparison (\cref{sec:runtime_comparison}), qualitative examples (\cref{sec:qualitative_examples}) and more.

\section{Discussion and Limitations}

Our work serves as a direct counterpart to \citet{ross2025a}, who demonstrate how the collapse of \(\lid\) values provides a direct indicator of \emph{memorized} samples. These complementary findings offer empirical evidence for an intuitive understanding of \(\lid\): it serves as a proxy for the \emph{creativity} of a DM, where hallucinations (the model being overly creative) and memorization (a lack of creativity) span the full generative spectrum.

Evaluation persists as the primary limitation of \(\iq\) and other hallucination reduction methods. While human annotators generally agree when identifying structural failures, disagreements remain common, highlighting the subjective nature of the task. Constructing fully objective, large-scale benchmarks with automated evaluation would streamline the development of more effective methods. Regarding the observed drops in DSV and IV, proxy metrics for diversity, in datasets like \texttt{11kHands} and \texttt{FFHQ}, our analysis in \cref{sec:diversity_analysis} indicates that they are caused by the removal of hallucinated samples, which artificially inflate these metrics; on the subset of correct samples, the diversity of \(\iq\) matches all baselines.

Furthermore, methods \emph{intrinsic} to the model, such as RODS and \(\iq\), still induce a moderate computational overhead. Discovering cheaper estimators for \(\lid\) would lead to immediate practical improvements for \(\iq\). Finally, while unconditional generation is the cornerstone of generative modeling, hallucinations pose the highest risk in conditional settings involving real-world decision-making and diagnosis \citep{antun2020instabilities}, highlighting a critical direction for future investigation.

\begin{ack}

Work on this project is financially supported by the Polish National Science Centre PRELUDIUM BIS grant No. \texttt{2023/50/O/ST6/00301}, the Foundation for Polish Science (FNP) grant ‘Centre for Credible AI’ No. \texttt{FENG.02.01-IP.05-0058/24}.

The computational resources for this work were provided by the Laboratory of Bioinformatics and Computational Genomics and the High Performance Computing Center of the Faculty of Mathematics and Information Science, Warsaw University of Technology. We also gratefully acknowledge Poland's High-performance Infrastructure PLGrid ACC Cyfronet AGH for providing computer facilities and support within computational grant no. PLG/2025/018330.
\end{ack}

\bibliographystyle{plainnat}
\bibliography{main}

@inproceedings{ho2020denoising,
  title={Denoising diffusion probabilistic models},
  author={Ho, Jonathan and Jain, Ajay and Abbeel, Pieter},
  booktitle={Advances in neural information processing systems},
  year={2020}
}

@inproceedings{sohl2015deep,
  title={Deep unsupervised learning using nonequilibrium thermodynamics},
  author={Sohl-Dickstein, Jascha and Weiss, Eric and Maheswaranathan, Niru and Ganguli, Surya},
  booktitle={International Conference on Machine Learning},
  year={2015}
}

@inproceedings{songscore,
  title={Score-Based Generative Modeling through Stochastic Differential Equations},
  author={Song, Yang and Sohl-Dickstein, Jascha and Kingma, Diederik P and Kumar, Abhishek and Ermon, Stefano and Poole, Ben},
  booktitle={International Conference on Learning Representations},
year={2021}
}

@inproceedings{liu20232,
  title={{I2SB: Image-to-Image Schr{\"o}dinger Bridge}},
  author={Liu, Guan-Horng and Vahdat, Arash and Huang, De-An and Theodorou, Evangelos and Nie, Weili and Anandkumar, Anima},
  booktitle={International Conference on Machine Learning},
  year={2023},
}

@inproceedings{dhariwal2021beat,
  title={Diffusion Models Beat GANs on Image Synthesis},
  author={Dhariwal, Prafulla and Nichol, Alex},
  booktitle={Advances in neural information processing systems},
  year={2021}
}

@inproceedings{rombach2022high,
  title={High-resolution image synthesis with latent diffusion models},
  author={Rombach, Robin and Blattmann, Andreas and Lorenz, Dominik and Esser, Patrick and Ommer, Bj{\"o}rn},
  booktitle={Proceedings of the IEEE/CVF conference on computer vision and pattern recognition},
  pages={10684--10695},
  year={2022}
}

@inproceedings{zhang2018unreasonable,
  title={The unreasonable effectiveness of deep features as a perceptual metric},
  author={Zhang, Richard and Isola, Phillip and Efros, Alexei A and Shechtman, Eli and Wang, Oliver},
  booktitle={Proceedings of the IEEE conference on computer vision and pattern recognition},
  year={2018}
}

@inproceedings{sundararajan2017axiomatic,
  title={Axiomatic attribution for deep networks},
  author={Sundararajan, Mukund and Taly, Ankur and Yan, Qiqi},
  booktitle={International Conference on Machine Learning},
  pages={3319--3328},
  year={2017},
  organization={PMLR}
}

@inproceedings{heusel2017gans,
  title={Gans trained by a two time-scale update rule converge to a local nash equilibrium},
  author={Heusel, Martin and Ramsauer, Hubert and Unterthiner, Thomas and Nessler, Bernhard and Hochreiter, Sepp},
  booktitle={Advances in neural information processing systems},
  year={2017}
}

@inproceedings{
    chung2023diffusion,
    title={Diffusion Posterior Sampling for General Noisy Inverse Problems},
    author={Hyungjin Chung and Jeongsol Kim and Michael Thompson Mccann and Marc Louis Klasky and Jong Chul Ye},
    booktitle={The Eleventh International Conference on Learning Representations},
    year={2023}
}

@inproceedings{10.5555/3295222.3295230,
author = {Lundberg, Scott M. and Lee, Su-In},
title = {A unified approach to interpreting model predictions},
year = {2017},
booktitle = {Proceedings of the 31st International Conference on Neural Information Processing Systems},
numpages = {10},
location = {Long Beach, California, USA},
}

@inproceedings{ronneberger2015u,
  title={U-net: Convolutional networks for biomedical image segmentation},
  author={Ronneberger, Olaf and Fischer, Philipp and Brox, Thomas},
  booktitle={Medical image computing and computer-assisted intervention--MICCAI 2015: 18th international conference, Munich, Germany, October 5-9, 2015, proceedings, part III 18},
  pages={234--241},
  year={2015},
  organization={Springer}
}

@inproceedings{zhoudenoising,
  title={Denoising Diffusion Bridge Models},
  author={Zhou, Linqi and Lou, Aaron and Khanna, Samar and Ermon, Stefano},
  booktitle={International Conference on Learning Representations},
    year={2024}
}

@inproceedings{luo2023image,
  title={Image restoration with mean-reverting stochastic differential equations},
  author={Luo, Ziwei and Gustafsson, Fredrik K and Zhao, Zheng and Sj{\"o}lund, Jens and Sch{\"o}n, Thomas B},
  booktitle={International Conference on Machine Learning},
  year={2023}
}

@inproceedings{yue2024image,
  title={Image restoration through generalized ornstein-uhlenbeck bridge},
  author={Yue, Conghan and Peng, Zhengwei and Ma, Junlong and Du, Shiyan and Wei, Pengxu and Zhang, Dongyu},
  booktitle={International Conference on Machine Learning},
  year={2024}
}

@InProceedings{Garber_2024_CVPR,
    author    = {Garber, Tomer and Tirer, Tom},
    title     = {Image Restoration by Denoising Diffusion Models with Iteratively Preconditioned Guidance},
    booktitle = {Proceedings of the IEEE/CVF Conference on Computer Vision and Pattern Recognition (CVPR)},
    year      = {2024},
}

@inproceedings{song2023consistency,
  title={Consistency Models},
  author={Song, Yang and Dhariwal, Prafulla and Chen, Mark and Sutskever, Ilya},
  booktitle={International Conference on Machine Learning},
  year={2023}
}

@misc{rsna-intracranial-hemorrhage-detection,
    author = {Anouk Stein, MD and Carol Wu and Chris Carr and George Shih and Jayashree Kalpathy-Cramer and Julia Elliott and kalpathy and Luciano Prevedello and Marc Kohli, MD and Matt Lungren and Phil Culliton and Robyn Ball and Safwan Halabi MD},
    title = {RSNA Intracranial Hemorrhage Detection},
    year = {2019},
    howpublished = {\url{https://kaggle.com/competitions/rsna-intracranial-hemorrhage-detection}},
    note = {Kaggle}
}

@inproceedings{blau2018perception,
  title={The perception-distortion tradeoff},
  author={Blau, Yochai and Michaeli, Tomer},
  booktitle={Proceedings of the IEEE conference on computer vision and pattern recognition},
  pages={6228--6237},
  year={2018}
}

@inproceedings{lipmanflow,
  title={Flow Matching for Generative Modeling},
  author={Lipman, Yaron and Chen, Ricky TQ and Ben-Hamu, Heli and Nickel, Maximilian and Le, Matthew},
  booktitle={International Conference on Learning Representations},
    year={2023}
}

@inproceedings{aithal2024understanding,
  title={Understanding hallucinations in diffusion models through mode interpolation},
  author={Aithal, Sumukh K and Maini, Pratyush and Lipton, Zachary and Kolter, J Zico},
  booktitle={Advances in neural information processing systems},
  year={2024}
}

@inproceedings{yeats2025connection,
  title={A Connection Between Score Matching and Local Intrinsic Dimension},
  author={Yeats, Eric and Jacobson, Aaron and Hannan, Darryl and Jia, Yiran and Doster, Timothy and Kvinge, Henry and Mahan, Scott},
  booktitle={NeurIPS 2025 Workshop on Structured Probabilistic Inference $\&$ Generative Modeling},
  year={2025}
}

@inproceedings{
    tian2025rods,
    title={{RODS}: Robust Optimization Inspired Diffusion Sampling for Detecting and Reducing Hallucination in Generative Models},
    author={Yiqi Tian and Pengfei Jin and Mingze Yuan and Na Li and Bo Zeng and Quanzheng Li},
    booktitle={The Thirty-ninth Annual Conference on Neural Information Processing Systems},
    year={2025},
    url={https://openreview.net/forum?id=fhuqIxoPcr}
}

@inproceedings{triaridis2025mitigating,
  title={Mitigating Diffusion Model Hallucinations with Dynamic Guidance},
  author={Triaridis, Kostas and Graikos, Alexandros and Chatziagapi, Aggelina and Chrysos, Grigorios G and Samaras, Dimitris},
  booktitle={arXiv},
  year={2025}
}

@inproceedings{kim2024tackling,
  title={Tackling structural hallucination in image translation with local diffusion},
  author={Kim, Seunghoi and Jin, Chen and Diethe, Tom and Figini, Matteo and Tregidgo, Henry FJ and Mullokandov, Asher and Teare, Philip and Alexander, Daniel C},
  booktitle={European Conference on Computer Vision},
  year={2024},
  organization={Springer}
}

@inproceedings{oorloff2025mitigating,
  title={Mitigating hallucinations in diffusion models through adaptive attention modulation},
  author={Oorloff, Trevine and Yacoob, Yaser and Shrivastava, Abhinav},
  booktitle={arXiv},
  year={2025}
}

@misc{
bhosale2026varianceguided,
title={Variance-Guided Score Regularization for Hallucination Mitigation in Diffusion Models},
author={Mahesh Bhosale and Naresh Kumar Devulapally and Abdul Wasi and Chau Pham and Vishnu Suresh Lokhande and David Doermann},
year={2026},
url={https://openreview.net/forum?id=nY4nULFzDP}
}

@inproceedings{ren2025hallucination,
  title={Hallucination Score: Towards Mitigating Hallucinations in Generative Image Super-Resolution},
  author={Ren, Weiming and Goyal, Raghav and Hu, Zhiming and Aumentado-Armstrong, Tristan Ty and Mohomed, Iqbal and Levinshtein, Alex},
  booktitle={arXiv},
  year={2025}
}

@inproceedings{cao2025temporal,
  title={Temporal score analysis for understanding and correcting diffusion artifacts},
  author={Cao, Yu and Zhao, Zengqun and Patras, Ioannis and Gong, Shaogang},
  booktitle={Proceedings of the Computer Vision and Pattern Recognition Conference},
  pages={7707--7716},
  year={2025}
}

@inproceedings{tivnan2024hallucination,
  title={Hallucination index: An image quality metric for generative reconstruction models},
  author={Tivnan, Matthew and Yoon, Siyeop and Chen, Zhennong and Li, Xiang and Wu, Dufan and Li, Quanzheng},
  booktitle={International Conference on Medical Image Computing and Computer-Assisted Intervention},
  pages={449--458},
  year={2024},
  organization={Springer}
}

@INPROCEEDINGS{10981211,
  author={Ku, Alice and Tivnan, Matthew and Wu, Dufan},
  booktitle={2025 IEEE 22nd International Symposium on Biomedical Imaging (ISBI)}, 
  title={Hallucination Analysis of Score-Based Diffusion Models for CT Denoising Through Spatial Frequency Decomposition}, 
  year={2025},
  volume={},
  number={},
  pages={1-5},
  doi={10.1109/ISBI60581.2025.10981211}}

@inproceedings{
ross2025a,
title={A Geometric Framework for Understanding Memorization in Generative Models},
author={Brendan Leigh Ross and Hamidreza Kamkari and Tongzi Wu and Rasa Hosseinzadeh and Zhaoyan Liu and George Stein and Jesse C. Cresswell and Gabriel Loaiza-Ganem},
booktitle={The Thirteenth International Conference on Learning Representations},
year={2025},
url={https://openreview.net/forum?id=aZ1gNJu8wO}
}

@inproceedings{
lu2025towards,
title={Towards Understanding Text Hallucination of Diffusion Models via Local Generation Bias},
author={Rui Lu and Runzhe Wang and Kaifeng Lyu and Xitai Jiang and Gao Huang and Mengdi Wang},
booktitle={The Thirteenth International Conference on Learning Representations},
year={2025},
url={https://openreview.net/forum?id=SKW10XJlAI}
}

@inproceedings{narasimhaswamy2024handiffuser,
  title={Handiffuser: Text-to-image generation with realistic hand appearances},
  author={Narasimhaswamy, Supreeth and Bhattacharya, Uttaran and Chen, Xiang and Dasgupta, Ishita and Mitra, Saayan and Hoai, Minh},
  booktitle={Proceedings of the IEEE/CVF Conference on Computer Vision and Pattern Recognition},
  pages={2468--2479},
  year={2024}
}

@inproceedings{lu2024handrefiner,
  title={Handrefiner: Refining malformed hands in generated images by diffusion-based conditional inpainting},
  author={Lu, Wenquan and Xu, Yufei and Zhang, Jing and Wang, Chaoyue and Tao, Dacheng},
  booktitle={Proceedings of the 32nd ACM International Conference on Multimedia},
  pages={7085--7093},
  year={2024}
}

@inproceedings{Shi_Guo_Shui_Chen_Shen_2026, 
    title={SGMHand: Structure-Guided Modulation for Structure-Aware Hand Inpainting}, 
    booktitle={Proceedings of the AAAI Conference on Artificial Intelligence}, 
    author={Shi, Chuancheng and Guo, Shiming and Shui, Ke and Chen, Yixiang and Shen, Fei}, 
    year={2026}
}

@inproceedings{kamkari2024geometric,
  title={A geometric view of data complexity: Efficient local intrinsic dimension estimation with diffusion models},
  author={Kamkari, Hamidreza and Ross, Brendan L and Hosseinzadeh, Rasa and Cresswell, Jesse C and Loaiza-Ganem, Gabriel},
  booktitle={Advances in Neural Information Processing Systems},
  year={2024}
}

@InProceedings{pmlr-v235-stanczuk24a,
  title = 	 {Diffusion Models Encode the Intrinsic Dimension of Data Manifolds},
  author =       {Stanczuk, Jan Pawel and Batzolis, Georgios and Deveney, Teo and Sch\"{o}nlieb, Carola-Bibiane},
  booktitle = 	 {Proceedings of the 41st International Conference on Machine Learning},
  year = 	 {2024}
}

@inproceedings{tempczyk2022lidl,
  title={Lidl: Local intrinsic dimension estimation using approximate likelihood},
  author={Tempczyk, Piotr and Michaluk, Rafa{\l} and Garncarek, Lukasz and Spurek, Przemys{\l}aw and Tabor, Jacek and Golinski, Adam},
  booktitle={International Conference on Machine Learning},
  pages={21205--21231},
  year={2022},
  organization={PMLR}
}

@inproceedings{
liu2022,
title={Flow Straight and Fast: Learning to Generate and Transfer Data with Rectified Flow},
author={Xingchao Liu and Chengyue Gong and Qiang Liu},
booktitle={The Eleventh International Conference on Learning Representations },
year={2023}
}

@inproceedings{
albergo2023building,
title={Building Normalizing Flows with Stochastic Interpolants},
author={Michael Samuel Albergo and Eric Vanden-Eijnden},
booktitle={The Eleventh International Conference on Learning Representations },
year={2023}
}

@inproceedings{
geng2025mean,
title={Mean Flows for One-step Generative Modeling},
author={Zhengyang Geng and Mingyang Deng and Xingjian Bai and J Zico Kolter and Kaiming He},
booktitle={The Thirty-ninth Annual Conference on Neural Information Processing Systems},
year={2025}
}

@inproceedings{
lou2024discrete,
title={Discrete Diffusion Modeling by Estimating the Ratios of the Data Distribution},
author={Aaron Lou and Chenlin Meng and Stefano Ermon},
booktitle={Forty-first International Conference on Machine Learning},
year={2024}
}

@inproceedings{
kong2021diffwave,
title={DiffWave: A Versatile Diffusion Model for Audio Synthesis},
author={Zhifeng Kong and Wei Ping and Jiaji Huang and Kexin Zhao and Bryan Catanzaro},
booktitle={International Conference on Learning Representations},
year={2021},
url={https://openreview.net/forum?id=a-xFK8Ymz5J}
}

@inproceedings{kotelnikov2023tabddpm,
  title={Tabddpm: Modelling tabular data with diffusion models},
  author={Kotelnikov, Akim and Baranchuk, Dmitry and Rubachev, Ivan and Babenko, Artem},
  booktitle={International conference on machine learning},
  year={2023},
  organization={PMLR}
}

@inproceedings{hoogeboom2022equivariant,
  title={Equivariant diffusion for molecule generation in 3d},
  author={Hoogeboom, Emiel and Satorras, V{\i}ctor Garcia and Vignac, Cl{\'e}ment and Welling, Max},
  booktitle={International conference on machine learning},
  pages={8867--8887},
  year={2022},
  organization={PMLR}
}

@inproceedings{park2023understanding,
  title={Understanding the latent space of diffusion models through the lens of riemannian geometry},
  author={Park, Yong-Hyun and Kwon, Mingi and Choi, Jaewoong and Jo, Junghyo and Uh, Youngjung},
  booktitle={Advances in Neural Information Processing Systems},
  volume={36},
  pages={24129--24142},
  year={2023}
}

@inproceedings{karras2022elucidating,
  title={Elucidating the design space of diffusion-based generative models},
  author={Karras, Tero and Aittala, Miika and Aila, Timo and Laine, Samuli},
  booktitle={Advances in neural information processing systems},
  year={2022}
}

@inproceedings{sclocchi2025phase,
  title={A phase transition in diffusion models reveals the hierarchical nature of data},
  author={Sclocchi, Antonio and Favero, Alessandro and Wyart, Matthieu},
  booktitle={Proceedings of the National Academy of Sciences},
  year={2025},
  publisher={National Academy of Sciences}
}

@inproceedings{ramesh2021zero,
  title={Zero-shot text-to-image generation},
  author={Ramesh, Aditya and Pavlov, Mikhail and Goh, Gabriel and Gray, Scott and Voss, Chelsea and Radford, Alec and Chen, Mark and Sutskever, Ilya},
  booktitle={International conference on machine learning},
  year={2021},
  organization={PMLR}
}

@inproceedings{vincent2011connection,
  title={A connection between score matching and denoising autoencoders},
  author={Vincent, Pascal},
  booktitle={Neural computation},
  year={2011},
  publisher={MIT Press}
}

@inproceedings{efron2011tweedie,
  title={Tweedie’s formula and selection bias},
  author={Efron, Bradley},
  booktitle={Journal of the American Statistical Association},
  year={2011},
  publisher={Taylor \& Francis}
}

@inproceedings{alain2014regularized,
  title={What regularized auto-encoders learn from the data-generating distribution},
  author={Alain, Guillaume and Bengio, Yoshua},
  booktitle={The Journal of Machine Learning Research},
  year={2014},
  publisher={JMLR}
}

@inproceedings{fefferman2016testing,
  title={Testing the manifold hypothesis},
  author={Fefferman, Charles and Mitter, Sanjoy and Narayanan, Hariharan},
  booktitle={Journal of the American Mathematical Society},
  year={2016}
}

@inproceedings{pidstrigach2022score,
  title={Score-based generative models detect manifolds},
  author={Pidstrigach, Jakiw},
  booktitle={Advances in Neural Information Processing Systems},
  year={2022}
}

@inproceedings{afifi201911kHands,
  title = {11K Hands: gender recognition and biometric identification using a large dataset of hand images},
  author = {Afifi, Mahmoud},
  booktitle = {Multimedia Tools and Applications},
  year={2019}
}

@inproceedings{vaswani2017attention,
  title={Attention is all you need},
  author={Vaswani, Ashish and Shazeer, Noam and Parmar, Niki and Uszkoreit, Jakob and Jones, Llion and Gomez, Aidan N and Kaiser, {\L}ukasz and Polosukhin, Illia},
  booktitle={Advances in neural information processing systems},
  year={2017}
}

@inproceedings{roth2022towards,
  title={Towards total recall in industrial anomaly detection},
  author={Roth, Karsten and Pemula, Latha and Zepeda, Joaquin and Sch{\"o}lkopf, Bernhard and Brox, Thomas and Gehler, Peter},
  booktitle={Proceedings of the IEEE/CVF conference on computer vision and pattern recognition},
  pages={14318--14328},
  year={2022}
}

@book{allgower2012numerical,
  title={Numerical continuation methods: an introduction},
  author={Allgower, Eugene L and Georg, Kurt},
  year={2012},
  publisher={Springer Science \& Business Media}
}

@inproceedings{choi2020starganv2,
  title={StarGAN v2: Diverse Image Synthesis for Multiple Domains},
  author={Yunjey Choi and Youngjung Uh and Jaejun Yoo and Jung-Woo Ha},
  booktitle={Proceedings of the IEEE Conference on Computer Vision and Pattern Recognition},
  year={2020}
}

@inproceedings{karras2019style,
  title={A style-based generator architecture for generative adversarial networks},
  author={Karras, Tero and Laine, Samuli and Aila, Timo},
  booktitle={Proceedings of the IEEE/CVF conference on computer vision and pattern recognition},
  pages={4401--4410},
  year={2019}
}

@inproceedings{nichol2021improved,
  title={Improved denoising diffusion probabilistic models},
  author={Nichol, Alexander Quinn and Dhariwal, Prafulla},
  booktitle={International conference on machine learning},
  pages={8162--8171},
  year={2021},
  organization={PMLR}
}

@misc{von-platen-etal-2022-diffusers,
  author = {Patrick von Platen and Suraj Patil and Anton Lozhkov and Pedro Cuenca and Nathan Lambert and Kashif Rasul and Mishig Davaadorj and Dhruv Nair and Sayak Paul and William Berman and Yiyi Xu and Steven Liu and Thomas Wolf},
  title = {Diffusers: State-of-the-art diffusion models},
  year = {2022},
  publisher = {GitHub},
  journal = {GitHub repository},
  howpublished = {\url{https://github.com/huggingface/diffusers}}
}

@inproceedings{miao2024training,
  title={Training diffusion models towards diverse image generation with reinforcement learning},
  author={Miao, Zichen and Wang, Jiang and Wang, Ze and Yang, Zhengyuan and Wang, Lijuan and Qiu, Qiang and Liu, Zicheng},
  booktitle={Proceedings of the IEEE/CVF Conference on Computer Vision and Pattern Recognition},
  pages={10844--10853},
  year={2024}
}

@inproceedings{
domingo-enrich2025adjoint,
title={Adjoint Matching: Fine-tuning Flow and Diffusion Generative Models with Memoryless Stochastic Optimal Control},
author={Carles Domingo-Enrich and Michal Drozdzal and Brian Karrer and Ricky T. Q. Chen},
booktitle={The Thirteenth International Conference on Learning Representations},
year={2025},
url={https://openreview.net/forum?id=xQBRrtQM8u}
}

@inproceedings{geirhos2018imagenet,
  title={ImageNet-trained CNNs are biased towards texture; increasing shape bias improves accuracy and robustness},
  author={Geirhos, Robert and Rubisch, Patricia and Michaelis, Claudio and Bethge, Matthias and Wichmann, Felix A and Brendel, Wieland},
  booktitle={International conference on learning representations},
  year={2018}
}

@inproceedings{stein2023exposing,
  title={Exposing flaws of generative model evaluation metrics and their unfair treatment of diffusion models},
  author={Stein, George and Cresswell, Jesse and Hosseinzadeh, Rasa and Sui, Yi and Ross, Brendan and Villecroze, Valentin and Liu, Zhaoyan and Caterini, Anthony L and Taylor, Eric and Loaiza-Ganem, Gabriel},
  booktitle={Advances in Neural Information Processing Systems},
  year={2023}
}

@inproceedings{jayasumana2024rethinking,
  title={Rethinking fid: Towards a better evaluation metric for image generation},
  author={Jayasumana, Sadeep and Ramalingam, Srikumar and Veit, Andreas and Glasner, Daniel and Chakrabarti, Ayan and Kumar, Sanjiv},
  booktitle={Proceedings of the IEEE/CVF conference on computer vision and pattern recognition},
  pages={9307--9315},
  year={2024}
}

@inproceedings{wang2025generative,
  title={Generative Artificial Intelligence in Medical Imaging: Foundations, Progress, and Clinical Translation},
  author={Wang, Shanshan and Zhou, Xuanru and Li, Cheng and Wang, Shuqiang and Li, Ye and Tan, Tao and Zheng, Hairong},
  booktitle={Research},
  volume={8},
  pages={1029},
  year={2025},
  publisher={AAAS}
}

@inproceedings{
sobieski2026systemembedded,
title={System-Embedded Diffusion Bridge Models},
author={Bartlomiej Sobieski and Matthew Tivnan and Yuang Wang and Siyeop yoon and Pengfei Jin and Dufan Wu and Quanzheng Li and Przemyslaw Biecek},
booktitle={The Thirty-ninth Annual Conference on Neural Information Processing Systems},
year={2026},
url={https://openreview.net/forum?id=cipx3rwfWp}
}

@inproceedings{antun2020instabilities,
  title={On instabilities of deep learning in image reconstruction and the potential costs of AI},
  author={Antun, Vegard and Renna, Francesco and Poon, Clarice and Adcock, Ben and Hansen, Anders C},
  booktitle={Proceedings of the National Academy of Sciences},
  year={2020},
  publisher={National Academy of Sciences}
}

@inproceedings{JMLR:v6:hyvarinen05a,
  author  = {Aapo Hyv{{\"a}}rinen},
  title   = {Estimation of Non-Normalized Statistical Models by Score Matching},
  booktitle = {Journal of Machine Learning Research},
  year    = {2005}
}

@inproceedings{goresky1988stratified,
  title={Stratified morse theory},
  author={Goresky, Mark and MacPherson, Robert},
  booktitle={Stratified Morse Theory},
  year={1988},
  publisher={Springer}
}

@inproceedings{wanganalytical,
  title={An Analytical Theory of Spectral Bias in the Learning Dynamics of Diffusion Models},
  author={Wang, Binxu and Pehlevan, Cengiz},
  booktitle={The Thirty-ninth Annual Conference on Neural Information Processing Systems},
  year={2025}
}

@inproceedings{bengio2013representation,
  title={Representation learning: A review and new perspectives},
  author={Bengio, Yoshua and Courville, Aaron and Vincent, Pascal},
  booktitle={IEEE transactions on pattern analysis and machine intelligence},
  year={2013},
  publisher={IEEE}
}

@inproceedings{levina2004maximum,
  title={Maximum likelihood estimation of intrinsic dimension},
  author={Levina, Elizaveta and Bickel, Peter},
  booktitle={Advances in neural information processing systems},
  year={2004}
}

@inproceedings{
pope2021the,
title={The Intrinsic Dimension of Images and Its Impact on Learning},
author={Phil Pope and Chen Zhu and Ahmed Abdelkader and Micah Goldblum and Tom Goldstein},
booktitle={International Conference on Learning Representations},
year={2021}
}

@book{cacuci2005sensitivity,
  title={Sensitivity and uncertainty analysis, volume II: applications to large-scale systems},
  author={Cacuci, Dan G and Ionescu-Bujor, Mihaela and Navon, Ionel Michael},
  year={2005},
  publisher={CRC press}
}

@inproceedings{ma2018characterizing,
  title={Characterizing Adversarial Subspaces Using Local Intrinsic Dimensionality},
  author={Ma, Xingjun and Li, Bo and Wang, Yisen and Erfani, Sarah M and Wijewickrema, Sudanthi and Schoenebeck, Grant and Song, Dawn and Houle, Michael E and Bailey, James},
  booktitle={International Conference on Learning Representations},
  year={2018}
}

@inproceedings{ansuini2019intrinsic,
  title={Intrinsic dimension of data representations in deep neural networks},
  author={Ansuini, Alessio and Laio, Alessandro and Macke, Jakob H and Zoccolan, Davide},
  booktitle={Advances in Neural Information Processing Systems},
  year={2019}
}

@inproceedings{fu2025counting,
  title={Counting Hallucinations in Diffusion Models},
  author={Fu, Shuai and Zhou, Jian and Chen, Qi and Jing, Huang and Nguyen, Huy Anh and Liu, Xiaohan and Zeng, Zhixiong and Ma, Lin and Zhang, Quanshi and Wu, Qi},
  booktitle={arXiv},
  year={2025}
}

@inproceedings{cho2025tag,
  title={TAG: Tangential Amplifying Guidance for Hallucination-Resistant Diffusion Sampling},
  author={Cho, Hyunmin and Ahn, Donghoon and Hong, Susung and Kim, Jee Eun and Kim, Seungryong and Jin, Kyong Hwan},
  booktitle={arXiv},
  year={2025}
}

@inproceedings{podell2024sdxl,
  title={Sdxl: Improving latent diffusion models for high-resolution image synthesis},
  author={Podell, Dustin and English, Zion and Lacey, Kyle and Blattmann, Andreas and Dockhorn, Tim and M{\"u}ller, Jonas and Penna, Joe and Rombach, Robin},
  booktitle={International Conference on Learning Representations},
  year={2024}
}

@inproceedings{wewer2025spatial,
  title={Spatial Reasoning with Denoising Models},
  author={Wewer, Christopher and Pogodzinski, Bart and Schiele, Bernt and Lenssen, Jan Eric},
  booktitle={Proceedings of the 42nd International Conference on Machine Learning},
  year={2025}
}

@inproceedings{chandran2025laplacian,
  title={Laplacian Score Sharpening for Mitigating Hallucination in Diffusion Models},
  author={{Chandran.C}, Barath and Anumasa, Srinivas and Liu, Dianbo},
  booktitle={arXiv},
  year={2025}
}

@inproceedings{yu2026pixeldit,
  title={{PixelDiT}: Pixel Diffusion Transformers for Image Generation},
  author={Yu, Yongsheng and Xiong, Wei and Nie, Weili and Sheng, Yichen and Liu, Shiqiu and Luo, Jiebo},
  booktitle={Proceedings of the IEEE/CVF Conference on Computer Vision and Pattern Recognition},
  year={2026}
}

@article{e7257ba672d8434ba0a01b8a10e3e0b8,
title = "Improving and generalizing flow-based generative models with minibatch optimal transport",
author = "Alexander Tong and Kilian Fatras and Nikolay Malkin and Guillaume Huguet and Yanlei Zhang and Jarrid Rector-Brooks and Guy Wolf and Yoshua Bengio",
year = "2024",
journal = "Transactions on Machine Learning Research",
}

%%%%%%%%%%%%%%%%%%%%%%%%%%%%%%%%%%%%%%%%%%%%%%%%%%%%%%%%%%%%
\newpage
\newpage
\appendix

\section*{Appendix}
\startcontents[sections]
\printcontents[sections]{l}{1}{\setcounter{tocdepth}{2}}
\setcounter{proposition}{0}
\setcounter{theorem}{0}
\setcounter{corollary}{0}

\section{Theoretical results}\label{sec:theoretical_results}

\subsection{\Cref{prop:lmi_ub}}
\begin{proposition}\label{prop:lmi_ub}
Let \(\generator\) be the deterministic generative mapping of a diffusion model from the latent noise space to the induced data manifold \(\manifoldt\), evaluated at an initial state \(\x_1\). Assume the perturbation scale \(\beta > 0\) is sufficiently small such that a first-order approximation holds. Let \(\sigma_1 \geq \sigma_2 \geq \dots \geq \sigma_n \geq 0\) denote the singular values of the generator's Jacobian \(\nabla_{\x_1}\generator(\x_1)\). Let \(\lidt(\x_0)\) denote the local intrinsic dimensionality of the generated data point \(\x_0 = \generator(\x_1)\) with respect to the induced manifold \(\manifoldt\). The Local Manifold Instability (\(\lmi\)) is approximately equal to the scaled sum of its top \(\left\lfloor\lidt(\x_0)\right\rfloor\) squared singular values:
\begin{equation}
    \lmi(\x_1) \approx \beta^2 \sum_{i=1}^{\left\lfloor\lidt(\x_0)\right\rfloor} \sigma_i^2.
\end{equation}
\end{proposition}

\begin{proof}
By definition, the Local Manifold Instability (\(\lmi\)) at an initial state \(\x_1\) is given by \(\lmi(\x_1) \triangleq \Tr(\cov_{\vareps}(\generator(\x_1+\vareps)))\), where \(\vareps \sim \mathcal{N}(\mathbf{0}, \beta^2\mathbf{I})\) for a perturbation scale \(\beta > 0\). Applying a first-order Taylor expansion to the generative mapping \(\generator\) around \(\x_1\) yields:
\begin{equation}
    \generator(\x_1 + \vareps) = \generator(\x_1) + \nabla_{\x_1}\generator(\x_1) \vareps + \mathcal{O}(\norm{\vareps}_2^2).
\end{equation}
Substituting this expansion into the covariance expression and recognizing that \(\generator(\x_1)\) is a deterministic constant with respect to \(\vareps\), it vanishes from the computation. Letting \(\mathbf{J} = \nabla_{\x_1}\generator(\x_1)\), we obtain:
\begin{equation}
    \cov_{\vareps}(\generator(\x_1+\vareps)) = \mathbf{J} \cov_{\vareps}(\vareps) \mathbf{J}^\top + \mathcal{O}(\beta^3) = \beta^2 \mathbf{J} \mathbf{J}^\top + \mathcal{O}(\beta^3).
\end{equation}
Applying the trace operator to calculate the total variance (\(\lmi\)):
\begin{equation}
    \lmi(\x_1) = \Tr(\beta^2 \mathbf{J} \mathbf{J}^\top) + \mathcal{O}(\beta^3) = \beta^2 \norm{\mathbf{J}}_F^2 + \mathcal{O}(\beta^3).
\end{equation}
For a sufficiently small noise scale \(\beta\), the linear regime dominates and the higher-order remainder \(\mathcal{O}(\beta^3)\) can be safely neglected, yielding the first-order approximation:
\begin{equation}\label{eq:lmi_frob_approx}
    \lmi(\x_1) \approx \beta^2 \norm{\mathbf{J}}_F^2.
\end{equation}

Our next objective is to relate this squared Frobenius norm to the effective spectral properties of the mapping. The squared Frobenius norm of any matrix is algebraically equivalent to the sum of all its squared singular values:
\begin{equation}
    \norm{\mathbf{J}}_F^2 = \sum_{i=1}^n \sigma_i^2.
\end{equation}

We address the geometric dimensionality of the generation process. The generator \(\generator\) continuously maps the \(n\)-dimensional ambient noise space into a lower-dimensional induced data manifold \(\manifoldt\). The local intrinsic dimensionality at the specific generated data point \(\x_0\), denoted as \(\lidt(\x_0)\), represents the effective number of independent, structurally significant directions in which the model's output varies.

While deep neural networks practically possess full-rank Jacobians (meaning all \(\sigma_i > 0\)), \cref{ass:stratified_manifold} implies that the mapping highly contracts the space along off-manifold directions. Consequently, the singular values corresponding to the remaining \(n - \left\lfloor\lidt(\x_0)\right\rfloor\) directions are infinitesimally small, representing negligible off-manifold numerical noise rather than structurally meaningful degrees of freedom. By treating these tail singular values as negligible (\(\sigma_i \approx 0\) for \(i > \left\lfloor\lidt(\x_0)\right\rfloor\)), we can approximate the full spectral sum by truncating it at the effective dimensionality:
\begin{equation}
    \sum_{i=1}^n \sigma_i^2 \approx \sum_{i=1}^{\left\lfloor\lidt(\x_0)\right\rfloor} \sigma_i^2.
\end{equation}

Finally, substituting this spectral truncation back into \cref{eq:lmi_frob_approx} yields the completed theoretical approximation:
\begin{equation}
    \lmi(\x_1) \approx \beta^2 \sum_{i=1}^{\left\lfloor\lidt(\x_0)\right\rfloor} \sigma_i^2.
\end{equation}
This demonstrates that, under a localized linear regime, the spatial instability of the generator is closely approximated by the variance of the injected noise scaled by the structurally significant spectral mass of the local manifold.
\end{proof}
\subsection{\Cref{prop:lidt_estimators}}
\begin{proposition}\label{prop:lidt_estimators}
Let \(\manifoldt\) be the \(\lidt(\x_0)\)-dimensional manifold induced by a fully trained diffusion model parameterized by \(\btheta\), embedded in an \(n\)-dimensional ambient space, with probability density \(p_{\btheta}(\x_0)\). Assume \(p_{\btheta}(\x_0)\) is locally uniform on \(\manifoldt\) with negligible curvature at the limit \(t \to 0\). Let the forward process be defined by an arbitrary drift matrix \(\mathbf{H}_t\) and a strictly positive-definite covariance matrix \(\boldsymbol{\Sigma}_t\). As \(t \to 0\), the expected DSM loss evaluates exactly to the local intrinsic dimension \(\lidt(\x_0)\), and the expected ISM loss evaluates exactly to \(-\frac{1}{2}(n - \lidt(\x_0))\).
\end{proposition}

\begin{proof}
Let the forward diffusion state be \(\x_t = \mathbf{H}_t\x_0 + \boldsymbol{\Sigma}_t^{\frac{1}{2}} \vareps\), where \(\vareps \sim \mathcal{N}(\mathbf{0}, \mathbf{I})\) and \(\x_0 \in \manifoldt\). By Tweedie's formula applied to the model's induced distribution, the predicted score function \(\mathbf{s}_{\btheta}(\x_t)\) defines the model's posterior mean estimate \(\pmptheta\) via:
\begin{equation}
    \mathbf{s}_{\btheta}(\x_t) = -\boldsymbol{\Sigma}_t^{-1} (\x_t - \mathbf{H}_t\pmptheta).
\end{equation}
The predicted noise is defined as \(\vareps_{\btheta}(\x_t) \triangleq -\boldsymbol{\Sigma}_t^{\frac{1}{2}} \mathbf{s}_{\btheta}(\x_t)\). Substituting the predicted score yields:
\begin{equation}
    \vareps_{\btheta}(\x_t) = \boldsymbol{\Sigma}_t^{-\frac{1}{2}} (\x_t - \mathbf{H}_t\pmptheta).
\end{equation}
To evaluate the DSM expectation \(\expval_{\vareps} \left[ \norm{\vareps - \vareps_{\btheta}(\x_t)}_2^2 \right]\), we express the target forward noise as \(\vareps = \boldsymbol{\Sigma}_t^{-\frac{1}{2}}(\x_t - \mathbf{H}_t\x_0)\). The residual error inside the norm becomes:
\begin{equation}\label{eq:noise_residual}
    \vareps - \vareps_{\btheta}(\x_t) = \boldsymbol{\Sigma}_t^{-\frac{1}{2}} \mathbf{H}_t (\pmptheta - \x_0).
\end{equation}

As \(t \to 0\), the model's posterior mean \(\pmptheta\) acts as the optimal projection of the noisy state back onto the locally flat model-induced manifold \(\manifoldt\). Because the forward process introduces spatial drift \(\mathbf{H}_t\) and covariance \(\boldsymbol{\Sigma}_t\), the projected state minimizes the Mahalanobis distance \(\norm{\x_t - \mathbf{H}_t\x}_{\boldsymbol{\Sigma}_t^{-1}}^2\) for \(\x \in \manifoldt\). 

Let \(\mathcal{T}_{\x_0}\manifoldt\) be the tangent space of the model's manifold at \(\x_0\). The projection onto this mapped tangent space \(\mathbf{H}_t\mathcal{T}_{\x_0}\manifoldt\) under the Mahalanobis inner product \(\langle \mathbf{u}, \mathbf{v} \rangle_{\boldsymbol{\Sigma}_t^{-1}} = \mathbf{u}^\top \boldsymbol{\Sigma}_t^{-1} \mathbf{v}\) is governed by an oblique projection matrix \(\mathbf{P}_{\mathcal{T}}^{\boldsymbol{\Sigma}_t, \mathbf{H}_t}\). Therefore, the deviation of the model's posterior mean mapped through the drift is exactly the oblique projection of the forward perturbation:
\begin{equation}
    \mathbf{H}_t(\pmptheta - \x_0) \approx \mathbf{P}_{\mathcal{T}}^{\boldsymbol{\Sigma}_t, \mathbf{H}_t} (\x_t - \mathbf{H}_t\x_0) = \mathbf{P}_{\mathcal{T}}^{\boldsymbol{\Sigma}_t, \mathbf{H}_t} \boldsymbol{\Sigma}_t^{\frac{1}{2}} \vareps.
\end{equation}

Substituting this geometric relation back into \cref{eq:noise_residual}, the DSM expectation becomes:
\begin{equation}
    \expval_{\vareps}[\mathcal{L}_{\text{DSM}}(\x_0, t, \btheta)] = \expval_{\vareps} \left[ \norm{\boldsymbol{\Sigma}_t^{-\frac{1}{2}} \mathbf{P}_{\mathcal{T}}^{\boldsymbol{\Sigma}_t, \mathbf{H}_t} \boldsymbol{\Sigma}_t^{\frac{1}{2}} \vareps}_2^2 \right].
\end{equation}

Define the transformed operator \(\mathbf{M}_{\mathcal{T}} \triangleq \boldsymbol{\Sigma}_t^{-\frac{1}{2}} \mathbf{P}_{\mathcal{T}}^{\boldsymbol{\Sigma}_t, \mathbf{H}_t} \boldsymbol{\Sigma}_t^{\frac{1}{2}}\). This operator represents the tangent space projection mapped into the whitened isotropic noise space. We prove \(\mathbf{M}_{\mathcal{T}}\) is a standard Euclidean orthogonal projector by checking symmetry. By definition, \(\mathbf{P}_{\mathcal{T}}^{\boldsymbol{\Sigma}_t, \mathbf{H}_t}\) is self-adjoint with respect to the Mahalanobis inner product, meaning \((\mathbf{P}_{\mathcal{T}}^{\boldsymbol{\Sigma}_t, \mathbf{H}_t})^\top \boldsymbol{\Sigma}_t^{-1} = \boldsymbol{\Sigma}_t^{-1} \mathbf{P}_{\mathcal{T}}^{\boldsymbol{\Sigma}_t, \mathbf{H}_t}\). Thus:
\begin{equation}
    \mathbf{M}_{\mathcal{T}}^\top = \boldsymbol{\Sigma}_t^{\frac{1}{2}} (\mathbf{P}_{\mathcal{T}}^{\boldsymbol{\Sigma}_t, \mathbf{H}_t})^\top \boldsymbol{\Sigma}_t^{-\frac{1}{2}} = \boldsymbol{\Sigma}_t^{\frac{1}{2}} (\boldsymbol{\Sigma}_t^{-1} \mathbf{P}_{\mathcal{T}}^{\boldsymbol{\Sigma}_t, \mathbf{H}_t} \boldsymbol{\Sigma}_t) \boldsymbol{\Sigma}_t^{-\frac{1}{2}} = \boldsymbol{\Sigma}_t^{-\frac{1}{2}} \mathbf{P}_{\mathcal{T}}^{\boldsymbol{\Sigma}_t, \mathbf{H}_t} \boldsymbol{\Sigma}_t^{\frac{1}{2}} = \mathbf{M}_{\mathcal{T}}.
\end{equation}

Because \(\mathbf{M}_{\mathcal{T}}\) is symmetric (\(\mathbf{M}_{\mathcal{T}}^\top = \mathbf{M}_{\mathcal{T}}\)) and idempotent (\(\mathbf{M}_{\mathcal{T}}^2 = \mathbf{M}_{\mathcal{T}}\)), it is a valid Euclidean orthogonal projector. Its rank is strictly preserved under similarity transformations and the non-singular linear mapping \(\mathbf{H}_t\), meaning \(\text{rank}(\mathbf{M}_{\mathcal{T}}) = \text{rank}(\mathbf{P}_{\mathcal{T}}^{\boldsymbol{\Sigma}_t, \mathbf{H}_t}) = \lidt(\x_0)\).

Applying the standard identity for the expected squared norm of an orthogonally projected standard Gaussian, we evaluate the trace:
\begin{equation}
    \expval_{\vareps}[\mathcal{L}_{\text{DSM}}(\x_0, t, \btheta)] = \expval_{\vareps} \left[ \vareps^\top \mathbf{M}_{\mathcal{T}}^\top \mathbf{M}_{\mathcal{T}} \vareps \right] = \Tr(\mathbf{M}_{\mathcal{T}}) = \lidt(\x_0).
\end{equation}

Finally, by Stein's Lemma, the generalized score matching objectives satisfy the exact algebraic identity \(\expval_{\vareps}[\mathcal{L}_{\text{DSM}}(\x_0, t, \btheta)] = n + 2\expval_{\vareps}[\mathcal{L}_{\text{ISM}}(\x_0, t, \btheta)]\). Rearranging this relationship and substituting the evaluated DSM expectation directly yields:
\begin{equation}
    \expval_{\vareps}[\mathcal{L}_{\text{ISM}}(\x_0, t, \btheta)] = -\frac{1}{2}(n - \lidt(\x_0)),
\end{equation}
proving geometric and algebraic consistency for arbitrary drift and covariance schedules exclusively within the model-induced geometry.
\end{proof}
\subsection{\Cref{th:lmi_ub_tau}}
We include a more general formulation of \cref{th:lmi_ub_tau} that is not limited to isotropic noise as \(t\rightarrow0\). This allows us to provide theoretical guarantees also for frameworks like SDB.
\begin{theorem}\label{th:lmi_ub_tau}
Let \(\x_0=\generator(\x_1)\). Assume that \(\generator(\x_1)\) is decomposed as \(\generator = \generator^{\leq\tau} \circ \generator^{>\tau}\) for sufficiently small time \(\tau > 0\), where \(\generator^{>\tau}: \x_1 \mapsto \x_\tau\) and \(\generator^{\leq\tau}: \x_\tau \mapsto \x_0\). Let the singular values of the macroscopic flow Jacobian \(\nabla_{\x_1}\generator^{>\tau}(\x_1)\) be monotonically ordered as \(\sigma_1^{>\tau} \geq \sigma_2^{>\tau} \geq \dots \geq \sigma_n^{>\tau} \geq 0\). Let \(K = \norm{\nabla_{\x_\tau}\generator^{\leq\tau}(\x_\tau)}_2^2 \geq 1\) represent the squared spectral norm of the terminal oblique projection. For sufficiently small \(\beta>0\), it is true that
\begin{equation}
    \lmi(\x_1) \lessapprox K \beta^2 \sum_{i=1}^{\left\lfloor\lidt(\hat{\x}_0^{\boldsymbol{\theta}}(\x_\tau))\right\rfloor} (\sigma_i^{>\tau})^2.
\end{equation}
\end{theorem}

\begin{proof}
Applying the multivariate chain rule to the decomposed generative flow \(\generator = \generator^{\leq\tau} \circ \generator^{>\tau}\), the total generator Jacobian \(\mathbf{J} = \nabla_{\x_1}\generator(\x_1)\) is the matrix product of the stage-wise Jacobians:
\begin{equation}
    \mathbf{J} = \nabla_{\x_\tau}\generator^{\leq\tau}(\x_\tau) \nabla_{\x_1}\generator^{>\tau}(\x_1) = \mathbf{J}_{\leq\tau} \mathbf{J}_{>\tau}.
\end{equation}

Recall from \cref{prop:lmi_ub} the first-order approximation for \(\lmi\) for sufficiently small \(\beta\), which evaluates to the scaled squared Frobenius norm of the total Jacobian: \(\lmi(\x_1) \approx \beta^2 \norm{\mathbf{J}}_F^2\). Substituting the chain rule expansion into this norm yields:
\begin{equation}
    \norm{\mathbf{J}}_F^2 = \norm{\mathbf{J}_{\leq\tau} \mathbf{J}_{>\tau}}_F^2.
\end{equation}

For sufficiently small \(\tau\), the terminal integration step \(\generator^{\leq\tau}\) behaves as a projection onto the induced manifold \(\manifoldt\). In the general case of anisotropic terminal noise (such as the embedded covariance structure in SDB), this operation acts as an \emph{oblique} projector rather than a standard Euclidean orthogonal projector. 

Consequently, its Jacobian \(\mathbf{J}_{\leq\tau}\) possesses exactly \(\left\lfloor\lidt(\x_0)\right\rfloor\) non-zero singular values. We denote these terminal singular values as \(\sigma_i^{\leq\tau}\) and bound them by the squared spectral norm of the oblique projector: \((\sigma_i^{\leq\tau})^2 \leq \norm{\mathbf{J}_{\leq\tau}}_2^2 \triangleq K\). The remaining \(n - \left\lfloor\lidt(\x_0)\right\rfloor\) singular values are strictly equal to \(0\) (annihilating variance along normal directions).

To decouple the two generative regimes, we apply the singular value product inequality (von Neumann trace inequality) to bound the Frobenius norm of the product. The sum of the squared singular values of the product is bounded by the sum of the products of their individual squared singular values:
\begin{equation}
    \norm{\mathbf{J}_{\leq\tau} \mathbf{J}_{>\tau}}_F^2 \leq \sum_{i=1}^n (\sigma_i^{\leq\tau})^2 (\sigma_i^{>\tau})^2.
\end{equation}

Substituting the bounded spectrum of the terminal oblique projector (\(\sigma_i^{\leq\tau} \leq \sqrt{K}\) for \(i \leq \left\lfloor\lidt(\x_0)\right\rfloor\) and \(0\) otherwise) acts as a strict algebraic cutoff for the summation, scaled by the projection constant \(K\):
\begin{equation}
    \sum_{i=1}^n (\sigma_i^{\leq\tau})^2 (\sigma_i^{>\tau})^2 \leq K \sum_{i=1}^{\left\lfloor\lidt(\x_0)\right\rfloor} (\sigma_i^{>\tau})^2.
\end{equation}

Substituting this strict upper bound back into the \(\lmi\) approximation provides an intermediate theoretical guarantee:
\begin{equation}\label{eq:lmi_intermediate_bound}
    \lmi(\x_1) \lessapprox K \beta^2 \sum_{i=1}^{\left\lfloor\lidt(\x_0)\right\rfloor} (\sigma_i^{>\tau})^2.
\end{equation}

Next, we address the topological equivalence of evaluating the intrinsic dimensionality at the unobserved terminal state \(\x_0\) versus the predicted posterior mean \(\hat{\x}_0^{\btheta}(\x_\tau)\). At any intermediate time \(\tau > 0\), the state \(\x_\tau\) is corrupted by forward process variance. Consequently, \(\x_\tau\) has full support in the \(n\)-dimensional ambient space, meaning \(\x_\tau \notin \manifoldt\). 

To evaluate the topological complexity of the target manifold from \(\x_\tau\), we compute the Tweedie estimate derived from the predicted score function, \(\hat{\x}_0^{\btheta}(\x_\tau)\). By the fundamental theorem of estimation, this conditional expectation acts as the Minimum Mean Square Error (MMSE) estimator. Because \(\manifoldt\) generally possesses curvature, the expectation of points on the manifold strictly resides within the convex hull of \(\manifoldt\), meaning \(\hat{\x}_0^{\btheta}(\x_\tau) \notin \manifoldt\) in the general case.

However, we establish asymptotic equivalence between this estimation and the terminal generative mapping. The terminal state is obtained by integrating the Probability Flow ODE from \(\tau\) to \(0\): \(\x_0 = \generator^{\leq\tau}(\x_\tau)\). As the integration interval \(\tau \to 0\), the distribution concentrates locally, the exact ODE integration converges to the analytic score-based projection, and the Euclidean distance to the manifold vanishes:
\begin{equation}
    \lim_{\tau \to 0} \text{dist}(\hat{\x}_0^{\btheta}(\x_\tau), \manifoldt) = 0.
\end{equation}
This implies that for a sufficiently small \(\tau\), the Tweedie estimate \(\hat{\x}_0^{\btheta}(\x_\tau)\) and the true terminal state \(\x_0\) reside within the exact same infinitesimal local neighborhood.

Under the assumption that \(\manifoldt\) is a smooth manifold, the local intrinsic dimensionality \(\lidt(\x)\) is a locally constant topological invariant. Because \(\hat{\x}_0^{\btheta}(\x_\tau)\) approaches \(\x_0\) sufficiently close, their geometric dimensionalities strictly align:
\begin{equation}
    \lidt(\hat{\x}_0^{\btheta}(\x_\tau)) = \lidt(\x_0).
\end{equation}

Substituting this topological equivalence into the intermediate bound (\cref{eq:lmi_intermediate_bound}) yields the completed operational approximation:
\begin{equation}
    \lmi(\x_1) \lessapprox K \beta^2 \sum_{i=1}^{\left\lfloor\lidt(\hat{\x}_0^{\boldsymbol{\theta}}(\x_\tau))\right\rfloor} (\sigma_i^{>\tau})^2.
\end{equation}
\end{proof}
\subsection{\Cref{th:boltzmann}}
\begin{theorem}\label{th:boltzmann}
Let $p_t^{\btheta}(\mathbf{x}_t)$ be the model-induced marginal density at time $t \leq \tau$. Assume the network parameterizes the exact score of its induced distribution, such that $\mathbf{s}_{\btheta}(\mathbf{x}_t) = \nabla_{\mathbf{x}_t} \log p_t^{\btheta}(\mathbf{x}_t)$. We define the ideal geometrically-guided distribution by re-weighting the marginal over the true terminal states: $p_t^{\btheta, \lambda_t}(\mathbf{x}_t) \propto p_t^{\btheta}(\mathbf{x}_t) \mathbb{E}_{p^{\btheta}(\mathbf{x}_0 | \mathbf{x}_t)}[\exp(-\lambda_t \lidt(\mathbf{x}_0))]$. Assume $\tau$ is sufficiently small such that the posterior variance $\sigma_t^2 \to 0$. Under these conditions, guiding the reverse diffusion with the modified score $\tilde{\mathbf{s}}_{\btheta}(\mathbf{x}_t) = \mathbf{s}_{\btheta}(\mathbf{x}_t) - \lambda_t \nabla_{\mathbf{x}_t} \lidt(\pmptheta)$ is mathematically equivalent to sampling from the ideal distribution $p_t^{\btheta, \lambda_t}(\mathbf{x}_t)$, shifting probability mass toward lower-dimensional target strata.
\end{theorem}

\begin{proof}
We define the ideal guided target density by re-weighting the model-induced marginal $p_t^{\btheta}(\mathbf{x}_t)$ with the expected terminal energy:
\begin{equation}
    p_t^{\btheta, \lambda_t}(\mathbf{x}_t) = \frac{1}{Z_t^{\btheta}} p_t^{\btheta}(\mathbf{x}_t) \mathbb{E}_{p^{\btheta}(\mathbf{x}_0 | \mathbf{x}_t)}\left[\exp(-\lambda_t \lidt(\mathbf{x}_0))\right],
\end{equation}
where $Z_t^{\btheta}$ is the partition function. 

Taking the logarithm and differentiating with respect to $\mathbf{x}_t$ yields the ideal modified score:
\begin{equation}
    \nabla_{\mathbf{x}_t} \log p_t^{\btheta, \lambda_t}(\mathbf{x}_t) = \nabla_{\mathbf{x}_t} \log p_t^{\btheta}(\mathbf{x}_t) + \nabla_{\mathbf{x}_t} \log \mathbb{E}_{p^{\btheta}(\mathbf{x}_0 | \mathbf{x}_t)}\left[\exp(-\lambda_t \lidt(\mathbf{x}_0))\right].
\end{equation}

To render this tractable, we analyze the domain $t \leq \tau$. As the noise scale diminishes ($\sigma_t^2 \to 0$), the model's posterior distribution $p^{\btheta}(\mathbf{x}_0 | \mathbf{x}_t)$ concentrates into a Dirac delta distribution centered at the Tweedie estimate $\pmptheta$. Consequently, the log-expectation of the energy converges exactly to the energy of the expectation:
\begin{equation}
    \lim_{\sigma_t^2 \to 0} \log \mathbb{E}_{p^{\btheta}(\mathbf{x}_0 | \mathbf{x}_t)} \left[ \exp(-\lambda_t \lidt(\mathbf{x}_0)) \right] = -\lambda_t \lidt(\pmptheta).
\end{equation}

While the theoretical intrinsic dimension of a smooth manifold is a locally constant integer (which would yield a trivial zero gradient), our operational definition of \(\lidt\) relies on the generalized score-matching estimator \(\mathcal{L}_{\text{DSM}}(\x, t, \btheta)\). Because this estimator evaluates the model's spatial error, it is a continuous, non-linear, and differentiable function of its inputs, parameterized by the neural network \(\btheta\). Therefore, applying the spatial gradient operator yields a well-defined, non-zero vector field:
\begin{equation}
    \nabla_{\mathbf{x}_t} \mathcal{E}(\mathbf{x}_t) = \nabla_{\mathbf{x}_t} \lidt(\pmptheta) \approx \nabla_{\mathbf{x}_t} \mathcal{L}_{\text{DSM}}(\pmptheta, t, \btheta).
\end{equation}

Substituting this limit and the assumed model score function (\(\mathbf{s}_{\btheta}(\mathbf{x}_t) = \nabla_{\mathbf{x}_t} \log p_t^{\btheta}(\mathbf{x}_t)\)) into the target gradient yields the operational guidance rule:
\begin{equation}
    \tilde{\mathbf{s}}_{\btheta}(\mathbf{x}_t) = \mathbf{s}_{\btheta}(\mathbf{x}_t) - \lambda_t \nabla_{\mathbf{x}_t} \lidt(\pmptheta).
\end{equation}

By substituting \(\tilde{\mathbf{s}}_{\btheta}(\mathbf{x}_t)\) into the reverse probability flow ODE, the generated trajectory simulates sampling from the ideal terminal-weighted distribution \(p_t^{\btheta, \lambda_t}(\mathbf{x}_t)\). Because this guided distribution is multiplicatively bounded by the original prior \(p_t^{\btheta}(\mathbf{x}_t)\), the integration strictly preserves the topological validity dictated by the model-induced distribution while explicitly minimizing the local intrinsic dimension of the expected final sample via the continuous gradient signal.
\end{proof}

\section{Extended experiments}\label{sec:extended_experiments}

\subsection{Experimental setup}\label{sec:experimental_setup}

As the domain of hallucination reduction currently lacks a single unified evaluation scheme, we reimplement all of the methods within our codebase. For DG, we train 6 noisy classifiers from scratch, one for each dataset, using ResNets of different sizes adapted to the complexity of the data for image datasets and a 3-layer MLP for \texttt{GaussianGrid}. For AAM, we adapt the temperature gradient descent scheme for all DMs and train a separate PatchCore anomaly detection model \citep{roth2022towards} for each dataset. To find optimal hyperparameters, we perform a large grid search for AAM, DG and RODS on all datasets, except \texttt{11kHands}, \texttt{FFHQ} and \texttt{AFHQV2} for RODS, where we use the originally proposed ones \citep{tian2025rods}. We report the hyperparameters of all methods in \cref{sec:hyperparameters}.

\subsection{User study}\label{sec:user_study}

For each of the five image datasets, the user study consisted of two subsequent phases: the \emph{calibration} phase and the \emph{labeling} phase. The former involved presenting 128 images from the original dataset in batches, aiming to calibrate the user's perception of structures and correlations that appear in the true data distribution. Although this phase might initially seem redundant, \cref{fig:user_study} presents several examples of hand images from the \texttt{11kHands} dataset that contain atypical hand and finger placements, highlighting the diversity and complex relationships that should not be labeled as hallucinations. Hence, this phase ensures that annotators are aware of what a structural hallucination is in each considered case.

The labeling phase involved presenting the annotator with a grid of images, each generated with a different sampling method from the same starting point \(\x_1\) of the diffusion process. The placement of all methods was fixed across grids, and the study was blinded; \ie, the annotators were not given the names of the methods, which were replaced by \texttt{Image n} captions, with \texttt{n} being the method index (see \cref{fig:user_study}).

The user study relied on two independent annotators, both of whom were machine learning researchers. The task was first briefly described to them to ensure that the term \emph{structural hallucinations} was understood correctly. At each iteration of the labeling phase, the annotators first had to indicate the indices of images that contained a structural hallucination (which could also be none) and then (for natural image datasets) the indices of images that were of the highest quality according to their own perception. 

Because the labeling process was highly time-consuming, the study utilized a different number of images for each dataset, chosen a priori based on our visual assessment and the subjective advantage of \(\iq\). This approach ensured that statistically significant results would be obtained (with a high probability) using a minimal number of images to reduce the burden on the annotators, which could have otherwise negatively impacted the study. The numbers of starting points (\ie, the number of grids presented to each annotator) were as follows: 128 for \texttt{11kHands}, 1024 for \texttt{AFHQV2} and \texttt{FFHQ} each, 512 for \texttt{MNIST} and \texttt{SimpleShapes} each. Each annotator spent approximately 8 hours in total (across two sessions) labeling all of the provided samples.

At last, we note that the initial empirical investigation (\cref{sec:empirical_investigation}) uses the labels provided by one of the annotators on the \texttt{11kHands} dataset.

\begin{figure}
    \centering
    \includegraphics[width=0.98\linewidth]{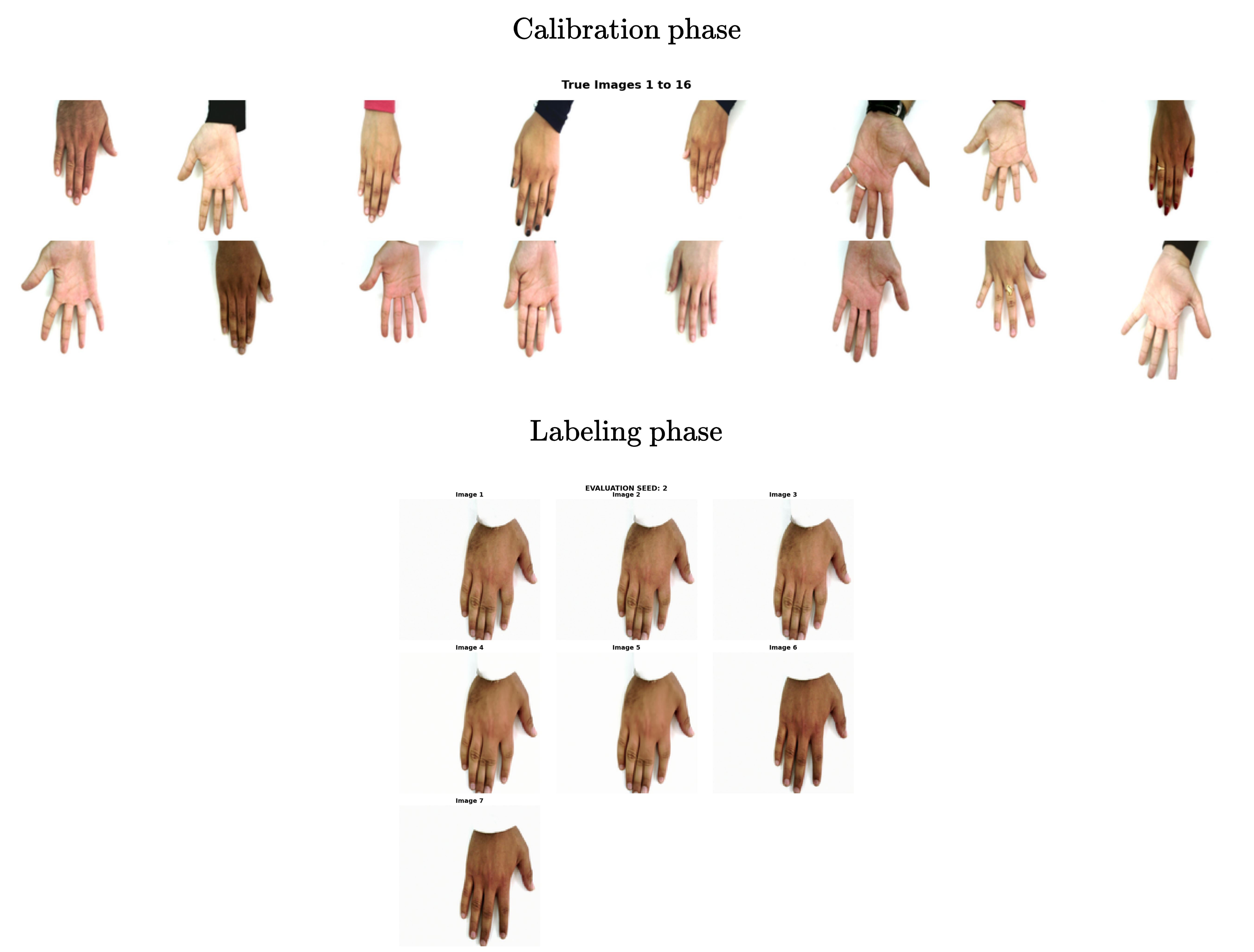}
    \caption{
    Examples of the user study interface from the \texttt{11kHands} dataset, depicting calibration and labeling phases.}
    \label{fig:user_study}
\end{figure}

\subsection{Filter performance}\label{sec:filter_performance}

We compare three filters (TVF, LMI, and \(\lid\)) in terms of separability and classification performance of correct and hallucinated samples in \cref{tab:extended_filter_performance}. For hyperparameters of each filter, see \cref{tab:filter_hyperparameters}. \(\lid\) consistently outperforms both TVF and LMI in almost all cases. LMI outperforms TVF in all cases, except for PR AUC and Cohen's $d$ on \texttt{SimpleShapes}, and outperforms \(\lid\) in some cases on \texttt{AFHQV2} and \texttt{MNIST}.

\begin{table}[h]
  \caption{Extended comparison of three filters (TVF, LMI, and \(\lid\)) in terms of separability and classification performance on natural and synthetic datasets.}
  \label{tab:extended_filter_performance}
  \centering
  
  \setlength{\tabcolsep}{4pt}
  
  % \newsavebox{\tophalftable}
  \sbox{\tophalftable}{%
    \begin{tabular}{l ccc ccc}
      \toprule
      & \multicolumn{3}{c}{\texttt{AFHQV2}} & \multicolumn{3}{c}{\texttt{FFHQ}} \\
      \cmidrule(lr){2-4} \cmidrule(lr){5-7}
      Method & PR AUC $\uparrow$ & ROC AUC $\uparrow$ & Cohen's $d$ $\uparrow$ & PR AUC $\uparrow$ & ROC AUC $\uparrow$ & Cohen's $d$ $\uparrow$ \\
      \midrule
      TVF & 0.12 & 0.55 & 0.18 & 0.16 & 0.53 & 0.17 \\
      LMI & 0.16 & 0.66 & 0.31 & 0.17 & 0.60 & 0.17 \\
      LID & 0.22 & 0.57 & 0.41 & 0.20 & 0.60 & 0.43 \\
      \bottomrule
    \end{tabular}%
  }
  
  \resizebox{\textwidth}{!}{%
    \begin{tabular}{@{}c@{}}
  
    % --- TOP HALF ---
    \usebox{\tophalftable} \\
    
    % --- SPACING BETWEEN HALVES ---
    \\[1ex]
    
    % --- BOTTOM HALF (Expanded to top half's width) ---
    \begin{tabular*}{\wd\tophalftable}{l @{\extracolsep{\fill}} ccc ccc}
      \toprule
      & \multicolumn{3}{c}{\texttt{MNIST}} & \multicolumn{3}{c}{\texttt{SimpleShapes}} \\
      \cmidrule(lr){2-4} \cmidrule(lr){5-7}
      Method & PR AUC $\uparrow$ & ROC AUC $\uparrow$ & Cohen's $d$ $\uparrow$ & PR AUC $\uparrow$ & ROC AUC $\uparrow$ & Cohen's $d$ $\uparrow$ \\
      \midrule
      TVF & 0.58 & 0.60 & 0.40 & 0.67 & 0.58 & 0.34 \\
      LMI & 0.62 & 0.67 & 0.57 & 0.60 & 0.59 & 0.23 \\
      LID & 0.63 & 0.66 & 0.42 & 0.67 & 0.61 & 0.34 \\
      \bottomrule
    \end{tabular*}
    
  \end{tabular}%
  }
\end{table}

\begin{table}[h]
  \caption{Summary of hyperparameters for the three evaluated filters (\(\lid\), LMI, and TVF) across natural and synthetic datasets. Time variables ($t, t_1, t_2$) are normalized to the $[0,1]$ interval as decimal fractions.}
  \label{tab:filter_hyperparameters}
  \centering
  \begin{tabular}{l c c c c}
    \toprule
    \multirow{2}{*}{Dataset} & \(\lid\) & LMI & \multicolumn{2}{c}{TVF} \\
    \cmidrule(lr){2-2} \cmidrule(lr){3-3} \cmidrule(lr){4-5}
    & $t$ & $\beta$ & $t_1$ & $t_2$ \\
    \midrule
    \texttt{11kHands} & 0.05 & 2.0 & 0.05 & 0.125 \\
    \texttt{FFHQ} & 0.00625 & 2.0 & 0.05 & 0.1 \\
    \texttt{AFHQV2} & 0.0125 & 2.0 & 0.9 & 0.95 \\
    \texttt{MNIST} & 0.001 & 0.1 & 0.155 & 0.158 \\
    \texttt{SimpleShapes} & 0.37 & 5.0 & 0.063 & 0.222 \\
    \bottomrule
  \end{tabular}
\end{table}

\subsection{\(\lid\) performance curves}

We visualize how the classification and separability performance of the \(\lid\) filter changes across time on all image datasets in \cref{fig:lid_performance_curves}. Crucially, the optimal timestep is very small in all cases, highlighting the importance of the theoretical assumptions.

\begin{figure}
    \centering
    \includegraphics[width=0.48\linewidth]{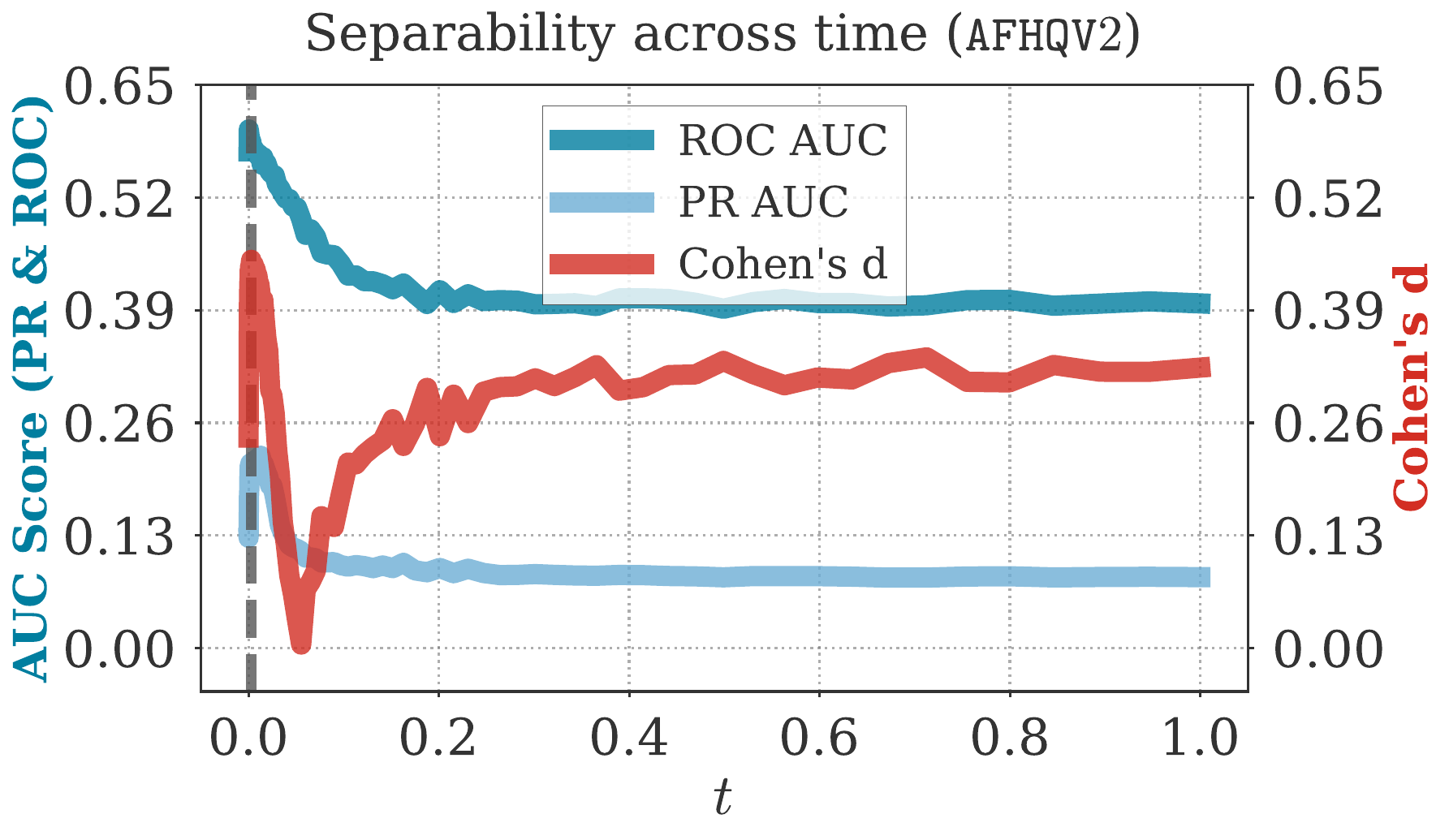}
    \includegraphics[width=0.48\linewidth]{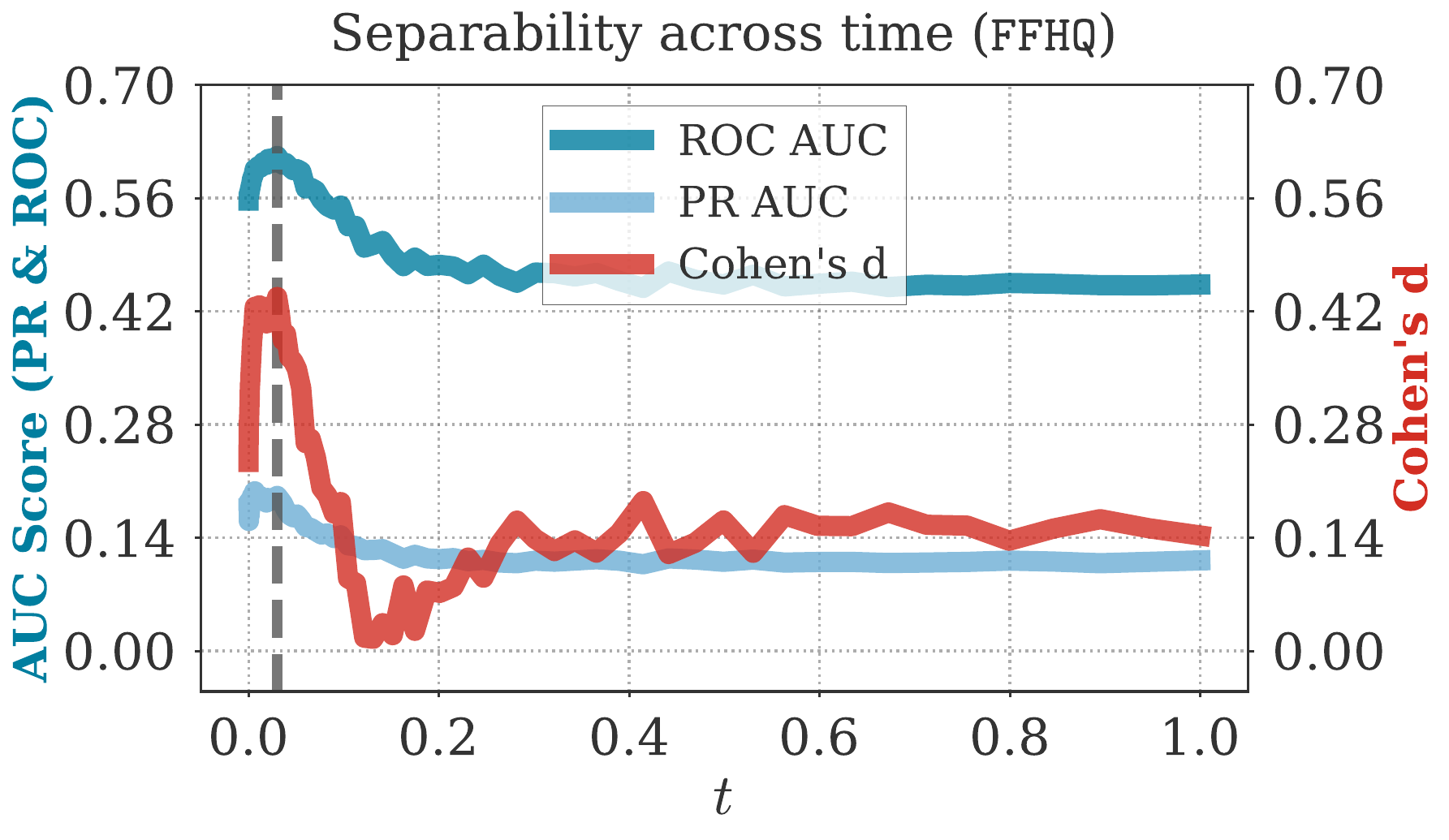}
    \includegraphics[width=0.48\linewidth]{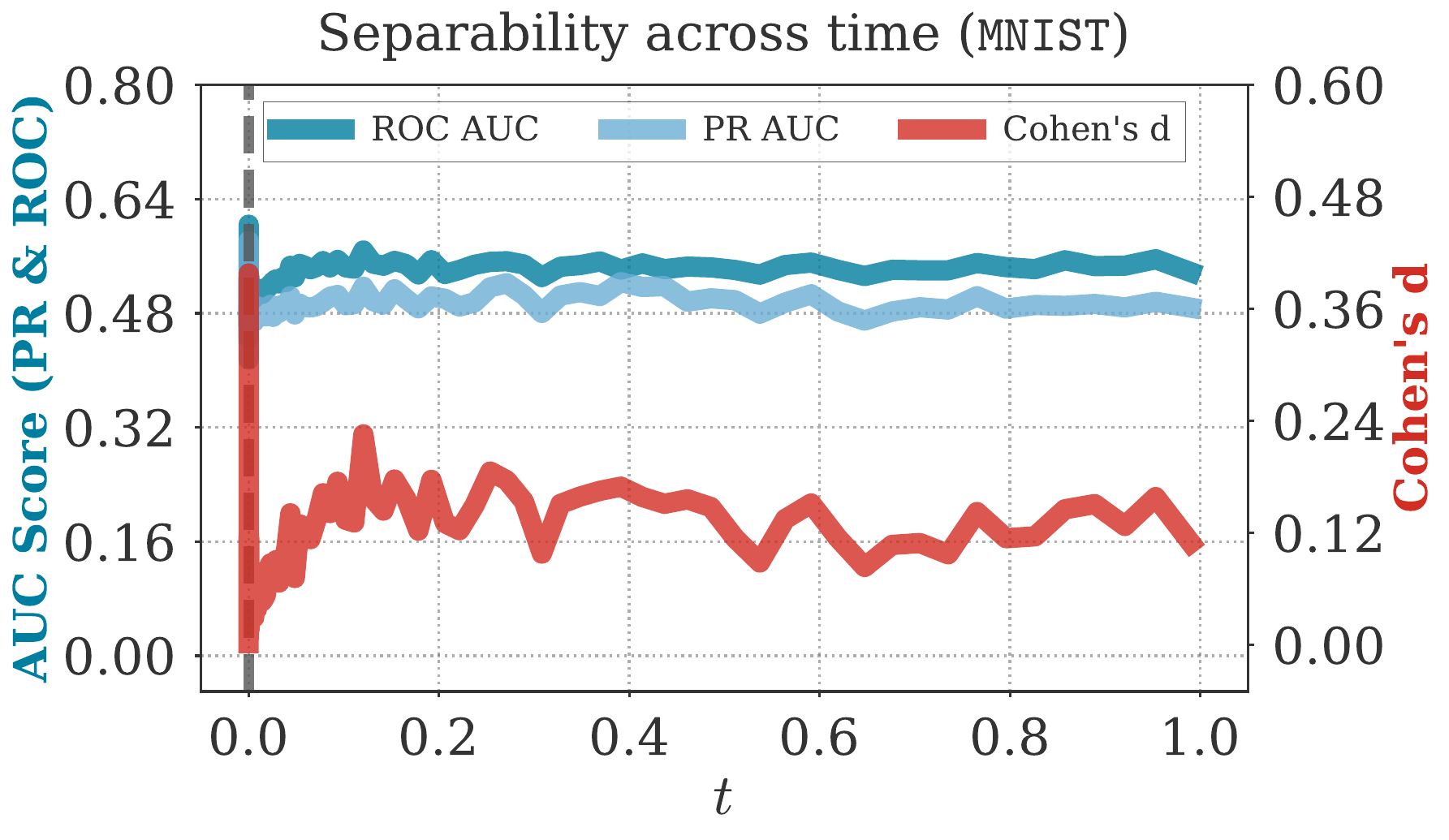}
    \includegraphics[width=0.48\linewidth]{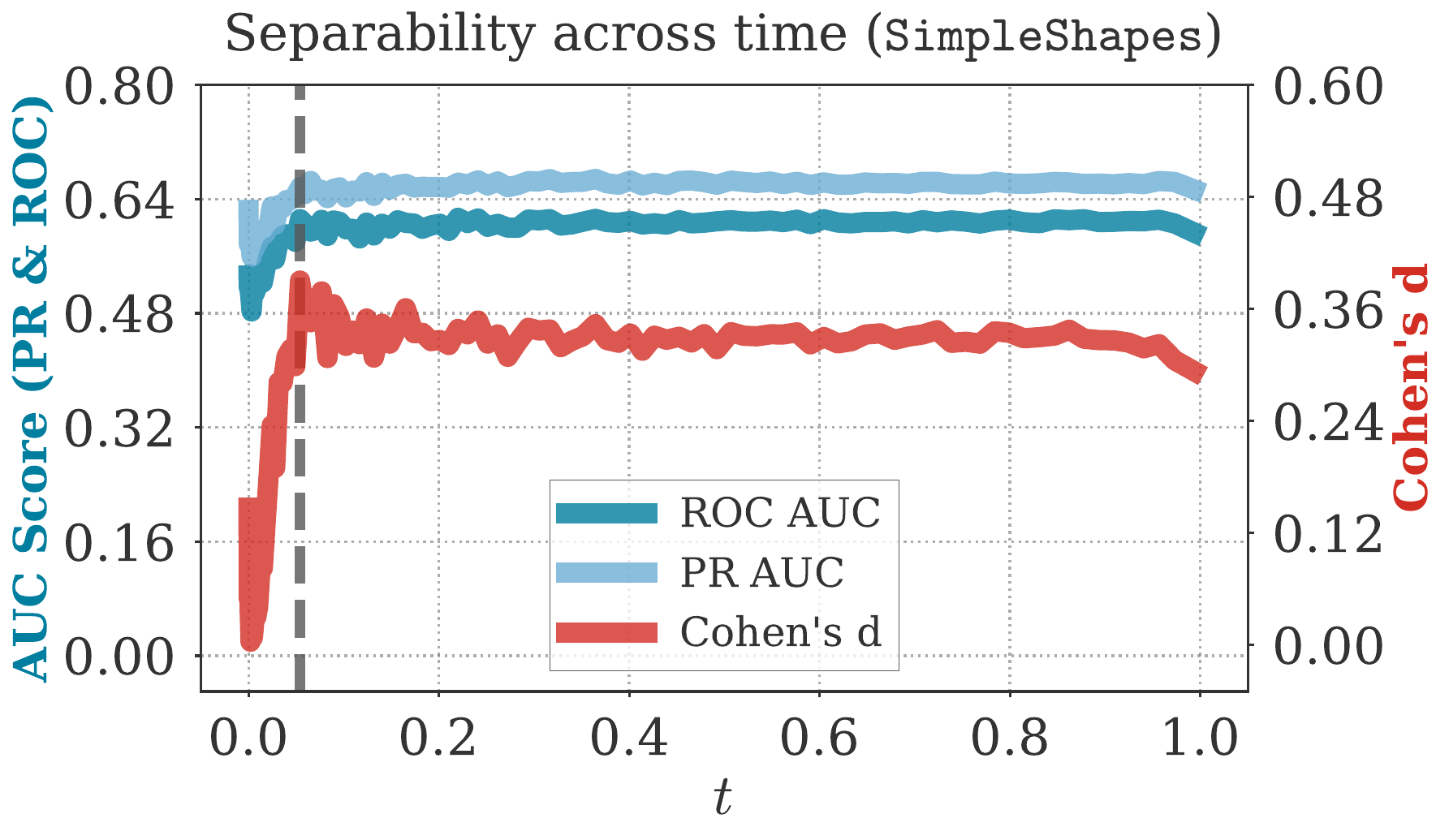}
    \caption{Performance curves of the \(\lid\) filter across diffusion timesteps. The optimal timestep for classification and separability metrics is indicated by a vertical line in each case.}
    \label{fig:lid_performance_curves}
\end{figure}

\subsection{Extended results}

We provide the average HR and UP in all considered cases, together with 95\% CIs, in \cref{tab:appendix_ci_metrics}. These results are based on the annotations of two independent experts (as mentioned in \cref{sec:user_study}) using varying sample sizes for each dataset (for reasons previously described): 128 for \texttt{11kHands}, 1024 for \texttt{AFHQV2} and \texttt{FFHQ} each, and 512 for \texttt{MNIST} and \texttt{SimpleShapes} each. The results for \texttt{GaussianGrid} are based on 16384 samples for each method.

We also note that we generate the \texttt{SimpleShapes} dataset procedurally by using the implementation of \citet{aithal2024understanding}, which divides a black \(64\times64\) plane into three vertical bands of equal width and then independently inserts a triangle (leftmost band), a square (middle band) and a rhomb (rightmost band), each with 50\% probability, into a random position within its band. Due to this procedural approach, we define hallucinations as improper placement of an object (\eg, a triangle in the middle band) or the generation of more than 3 objects.

\begin{table}[h]
  \caption{Detailed 95\% confidence intervals for User Preference (UP) and Hallucination Ratio (HR) evaluated on high-resolution and synthetic datasets. Values are formatted as Mean {\scriptsize [Lower Bound, Upper Bound]}.}
  \label{tab:appendix_ci_metrics}
  \centering
  
  \setlength{\tabcolsep}{4pt}
  
  % \newsavebox{\tophalftable}
  \sbox{\tophalftable}{%
    \begin{tabular}{l cccccc}
      \toprule
      & \multicolumn{2}{c}{\texttt{11kHands}} & \multicolumn{2}{c}{\texttt{FFHQ}} & \multicolumn{2}{c}{\texttt{AFHQV2}} \\
      \cmidrule(lr){2-3} \cmidrule(lr){4-5} \cmidrule(lr){6-7}
      Method & UP (\%) $\uparrow$ & HR (\%) $\downarrow$ & UP (\%) $\uparrow$ & HR (\%) $\downarrow$ & UP (\%) $\uparrow$ & HR (\%) $\downarrow$ \\
      \midrule
      Baseline & 39.8 {\scriptsize [34.8, 44.9]} & 29.3 {\scriptsize [24.1, 35.1]} & 45.3 {\scriptsize [42.8, 47.9]} & 8.2 {\scriptsize [7.0, 9.4]} & 41.8 {\scriptsize [38.5, 45.0]} & 6.9 {\scriptsize [5.9, 8.1]} \\
      DG       & 39.5 {\scriptsize [34.6, 44.3]} & 29.7 {\scriptsize [24.4, 35.6]} & 44.1 {\scriptsize [41.3, 47.0]} & 9.8 {\scriptsize [8.6, 11.1]} & 40.2 {\scriptsize [36.4, 44.0]} & 8.6 {\scriptsize [7.5, 9.9]} \\
      AAM      & 40.6 {\scriptsize [35.7, 45.5]} & 29.3 {\scriptsize [24.1, 35.1]} & 45.7 {\scriptsize [43.2, 48.2]} & 10.2 {\scriptsize [9.0, 11.6]} & 42.2 {\scriptsize [39.0, 45.4]} & 14.3 {\scriptsize [12.9, 15.9]} \\
      RODS-CAS & 40.2 {\scriptsize [35.3, 45.1]} & 25.8 {\scriptsize [20.8, 31.5]} & 45.3 {\scriptsize [42.8, 47.9]} & 7.7 {\scriptsize [6.6, 8.9]} & 42.2 {\scriptsize [39.0, 45.4]} & 6.9 {\scriptsize [5.9, 8.1]} \\
      RODS-SAS & 41.0 {\scriptsize [35.8, 46.2]} & 29.7 {\scriptsize [24.4, 35.6]} & 45.3 {\scriptsize [42.8, 47.9]} & 8.2 {\scriptsize [7.0, 9.4]} & 42.2 {\scriptsize [39.0, 45.4]} & 6.6 {\scriptsize [5.6, 7.7]} \\
      \(\iq\)      & 68.0 {\scriptsize [63.1, 72.8]} & 9.0 {\scriptsize [6.1, 13.1]} & 46.1 {\scriptsize [43.7, 48.4]} & 4.2 {\scriptsize [3.4, 5.1]} & 42.6 {\scriptsize [39.5, 45.7]} & 5.9 {\scriptsize [4.9, 7.0]} \\
      \bottomrule
    \end{tabular}%
  }
  
  \resizebox{\textwidth}{!}{%
    \begin{tabular}{@{}c@{}}
  
    % --- TOP HALF ---
    \usebox{\tophalftable} \\
    
    % --- SPACING BETWEEN HALVES ---
    \\[1ex]
    
    % --- BOTTOM HALF (Expanded to top half's width) ---
    \begin{tabular*}{\wd\tophalftable}{l @{\extracolsep{\fill}} ccc}
      \toprule
      & \multicolumn{1}{c}{\texttt{MNIST}} & \multicolumn{1}{c}{\texttt{SimpleShapes}} & \multicolumn{1}{c}{\texttt{GaussianGrid}} \\
      \cmidrule(lr){2-2} \cmidrule(lr){3-3} \cmidrule(lr){4-4}
      Method & HR (\%) $\downarrow$ & HR (\%) $\downarrow$ & HR (\%) $\downarrow$ \\
      \midrule
      Baseline & 37.3 {\scriptsize [33.2, 41.6]} & 25.8 {\scriptsize [19.0, 34.0]} & 20.2 {\scriptsize [19.5, 20.8]} \\
      DG       & 23.6 {\scriptsize [20.2, 27.5]} & 27.3 {\scriptsize [20.4, 35.6]} & 11.1 {\scriptsize [10.6, 11.6]} \\
      AAM      & 34.8 {\scriptsize [30.8, 39.0]} & 55.5 {\scriptsize [46.8, 63.8]} & - \\
      RODS-CAS & 41.8 {\scriptsize [37.6, 46.1]} & 30.5 {\scriptsize [23.2, 38.9]} & 10.4 {\scriptsize [9.9, 10.9]} \\
      RODS-SAS & 55.3 {\scriptsize [50.9, 59.5]} & 28.9 {\scriptsize [21.8, 37.3]} & 19.4 {\scriptsize [18.8, 20.0]} \\
      \(\iq\)      & 10.2 {\scriptsize [7.8, 13.1]}  & 9.4 {\scriptsize [5.4, 15.7]}   & 8.9 {\scriptsize [8.4, 9.3]} \\
      \bottomrule
    \end{tabular*}
    
  \end{tabular}%
  }
\end{table}

\subsection{Disentangling diversity and hallucinations}\label{sec:diversity_analysis}
Because \(\iq\) exhibits drops in proxy diversity metrics on several datasets (\eg, \texttt{11kHands}), a natural question is whether these drops stem solely from the removal of invalid hallucinated samples, or whether they also limit the diversity of valid outputs. To investigate this, we recompute both diversity metrics on a subset of samples limited purely to the correct non-hallucinated cases, based on the annotators' labels from \cref{sec:user_study}. To eliminate the effect of sample size, we limit the number of samples of each method to the minimum number of non-hallucinated samples across all methods. \Cref{tab:diversity_correct} presents the obtained results on \texttt{11kHands} up to 3 significant digits, with UP and HR included for reference. The proxy metrics for diversity are essentially matched across all approaches on correct samples, differing only at the 3rd significant digit. This equivalent performance is obtained with \(\iq\) simultaneously achieving much greater hallucination mitigation in comparison to other methods. These results strongly support the claim that hallucinations are indeed responsible for artificially inflating proxy diversity metrics, while \(\iq\) preserves the diversity of valid outputs.

\begin{table}[h]
  \caption{Comparison of proxy diversity metrics computed on the subset of correct, non-hallucinated \texttt{11kHands} samples, with UP and HR reported for reference. To eliminate the effect of sample size, each method is limited to the minimum number of non-hallucinated samples across all methods. Diversity values are reported up to 3 significant digits.}
  \label{tab:diversity_correct}
  \centering
  \begin{tabular}{l cccc}
    \toprule
    Method & IV $\uparrow$ & DSV $\uparrow$ & UP $\uparrow$ & HR $\downarrow$ \\
    \midrule
    True data & 0.0353 & 0.156 & - & - \\
    Baseline & 0.0309 & 0.111 & 39.8 & 29.3 \\
    DG & 0.0309 & 0.115 & 39.5 & 29.7 \\
    AAM & 0.0308 & 0.112 & 40.6 & 29.3 \\
    RODS\textsuperscript{CAS} & 0.0305 & 0.118 & 40.2 & 25.8 \\
    RODS\textsuperscript{SAS} & 0.0305 & 0.115 & 41.0 & 29.7 \\
    \(\iq\) & 0.0304 & 0.114 & 68.0 & 9.0 \\
    \bottomrule
  \end{tabular}
\end{table}

\subsection{Hallucination reduction in flow-based models}\label{sec:flow_matching}
As flow-based models gain prevalence, a natural question is whether the proposed \(\lid\)-based analysis and the \(\iq\) correction extend beyond diffusion models. From a theoretical standpoint, diffusion- and flow-based models can be considered equivalent thanks to the PF-ODE formulation and a bijective mapping between the score and the velocity originally used in Flow Matching \citep[FM;][]{lipmanflow}. Therefore, to apply \(\iq\) with the exact formulation from our work, one has to simply translate the velocity obtained from an FM model to the score, apply the steps outlined in \cref{sec:theoretical_foundation}, and translate back to the velocity domain. To demonstrate the feasibility of this combination in practice, we train an Optimal Transport Conditional Flow Matching \citep[OT-CFM;][]{e7257ba672d8434ba0a01b8a10e3e0b8} model with a standard 2D UNet architecture \citep{ronneberger2015u} on \texttt{11kHands} for 24 A100 GPU hours, extend its sampling scheme to allow for the application of both \(\iq\) and RODS in both variants, and conduct the evaluation using the exact setup from \cref{sec:experiments} (\eg, 2048 samples for proxy metrics, the blinded human study for UP and HR) with a 100-step Euler sampling scheme.

\Cref{tab:flow_results} presents the results. We observe that the considered OT-CFM model hallucinates slightly more often by default than the EDM model considered in our work, possibly due to the shorter training time. Interestingly, despite extensive hyperparameter tuning, RODS fails to mitigate hallucinations almost entirely, repeating the pattern observed for EDM, while only partially improving the proxy metrics, highlighting their inability to properly align with human perception and underscoring the difficulty of mitigating hallucinations in this setting. On the other hand, \(\iq\) significantly reduces hallucinations and achieves the highest UP, confirming that humans prefer its samples. Moreover, \(\iq\) exhibits the same minor shifts in diversity metrics as for EDM, which we explain in \cref{sec:diversity_analysis} by hallucinations artificially amplifying these scores. Overall, the results obtained for OT-CFM mirror the pattern observed for EDM, highlighting the strong performance and robustness of \(\iq\) across different frameworks.

\begin{table}[h]
  \caption{Quantitative comparison of all sampling methods on \texttt{11kHands} using an OT-CFM model with a 100-step Euler sampling scheme.}
  \label{tab:flow_results}
  \centering
  \begin{tabular}{l ccccc}
    \toprule
    Method & FID $\downarrow$ & IV $\uparrow$ & DSV $\uparrow$ & UP $\uparrow$ & HR $\downarrow$ \\
    \midrule
    Baseline & 18.8 & 0.026 & 0.13 & 58.4 & 43.8 \\
    RODS\textsuperscript{SAS} & 17.9 & 0.026 & 0.13 & 57.9 & 41.4 \\
    RODS\textsuperscript{CAS} & 18.1 & 0.026 & 0.13 & 58.4 & 43.0 \\
    \(\iq\) & 18.9 & 0.023 & 0.11 & 71.2 & 7.8 \\
    \bottomrule
  \end{tabular}
\end{table}

\subsection{Ablation studies}\label{sec:ablation_studies}

In the following subsections, we provide an ablation study of each hyperparameter of \(\iq\) on the \texttt{11kHands} dataset, where hallucinations are easiest to judge subjectively. For each hyperparameter combination, we manually labeled 128 samples, resulting in a total of 4096 labeled cases, and reported feature-based metrics using 2048 samples. Because the original annotators did not take part in these studies, the HR differs from the main results in \cref{tab:main_results}. In each study, we freeze all of the hyperparameters and vary only the one of interest. We use the settings from \cref{tab:params_iq} as the optimal configuration.

We note that while feature-based metrics are not the optimal choice for evaluating image quality when hallucinations are present, they serve as a useful proxy; large collapses or expansions in these metrics indicate an inevitable change in quality. Crucially, in the following experiments, we also observe that diversity metrics (IV and DSV) consistently inflate (suggesting \emph{improvements}) even when FID grows exponentially and HR achieves almost 100\%. This observation further highlights the misalignment between feature-based metrics (for diversity in this case) and human perception. 

\subsubsection{Ablating \(\lambda\)}

\begin{figure}
    \centering
    \includegraphics[width=0.99\linewidth]{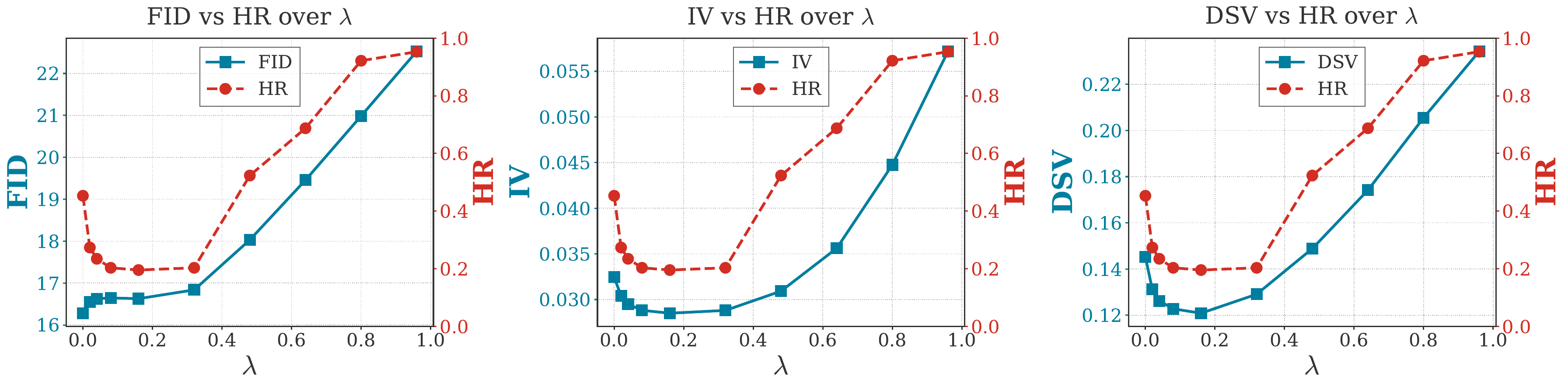}
    \caption{Ablation of the adaptive step size magnitude \(\lambda\) on the \texttt{11kHands} dataset.}
    \label{fig:guidance_ratio}
\end{figure}

\Cref{fig:guidance_ratio} provides the results of varying the adaptive step size magnitude \(\lambda\) across the \([0,1]\) range. Across all metrics, we observe that \([0.05, 0.2]\) is its optimal range, where HR plateaus around 0.3, while other metrics start to grow exponentially.

\subsubsection{Ablating time interval}

To ablate the choice of the time interval for \(\iq\), we pick \(t_1\) from an approximately log-uniform range of the EDM \(\sigma\) scale and choose the corresponding \(t_2\) so that guidance is applied at exactly 4 consecutive timesteps. \Cref{fig:guidance_intervals} reveals a sharp transition from high HR values into a clear, optimal point, after which the hallucinations once again increase, confirming the theoretical results. Moreover, a clear trade-off between feature-based metrics and HR is observed, once again showing that FID, as well as other metrics, improve when hallucinations occur frequently, and start to worsen once they decrease.

\begin{figure}
    \centering
    \includegraphics[width=0.99\linewidth]{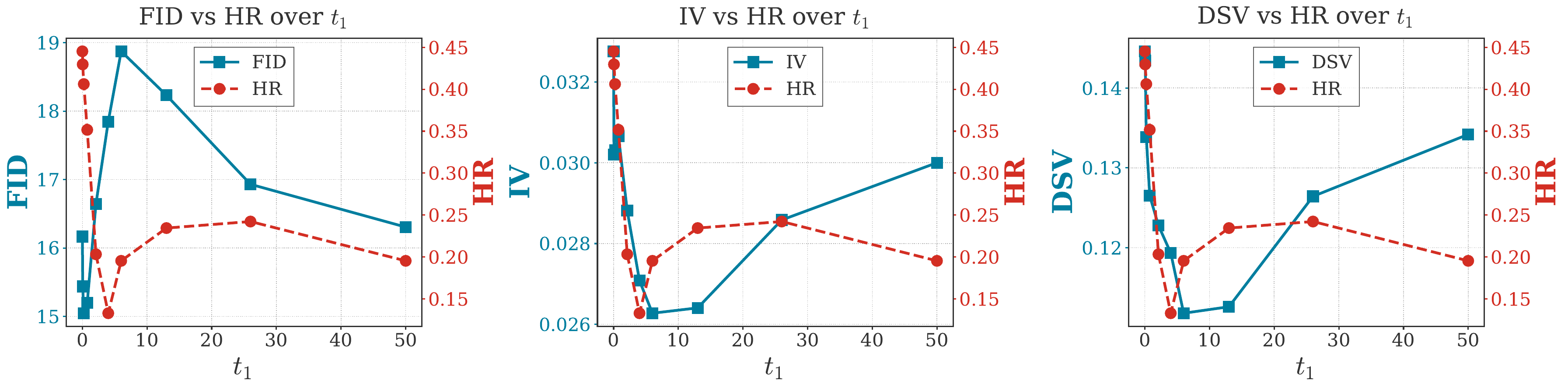}
    \caption{Ablation of the time interval bounds for \(\iq\) on the \texttt{11kHands} dataset. For each \(t_1\), its corresponding \(t_2\) was chosen so that \(\iq\) is applied at exactly 4 consecutive timesteps.}
    \label{fig:guidance_intervals}
\end{figure}

\subsubsection{Ablating filtering}

To ablate the filtering parameter \(q\), we sweep it across the \([0,1]\) range, where \(q=0\) implies that \(\iq\) is applied unconditionally across all samples, while \(q=1\) means that it is never applied. \Cref{fig:discard_percentile} visualizes the resulting dependency between \(q\), feature-based metrics, and HR, leading to two conclusions. First, there is an evident correlation between \(q\) and the metrics (negative for FID, positive for IV and DSV). Second, there appears to be an optimal choice of \(q\) at around 0.4, where the HR achieves satisfactory values while not significantly altering feature-based metrics. This once again provides evidence that \(\iq\) might be trading the inevitable cost in feature-based metrics to decrease hallucinations.

The effect of filtering is also clear when analyzing the \texttt{GaussianGrid} results (\cref{fig:hq_gauss_vis}). With \(q=0\) (depicted as \(\iq^*\)), all of the generated samples squeeze at the Gaussian modes, almost entirely zeroing out the variance of each Gaussian component. By applying filtering (\(\iq\)), this effect is eliminated and the original variance is almost matched.

\begin{figure}
    \centering
    \includegraphics[width=0.99\linewidth]{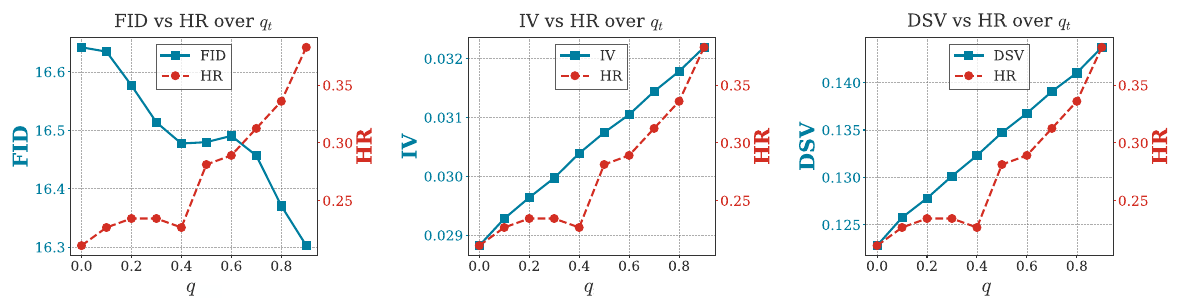}
    \caption{Ablation of the filtering parameter \(q\) on the \texttt{11kHands} dataset.}
    \label{fig:discard_percentile}
\end{figure}

\subsubsection{Ablating constant scaling}

To assess the validity of adaptive step size \(\lambda_t\), we replace it with constant scaling and sweep its values between 0.0 and 10.0 (based on the range of values we observed from the adaptive approach). \Cref{fig:guidance_scale} visualizes the results with an additional dimension for the presence of artifacts in the generated samples, \ie, oversaturated colors and severe structural distortions in more extreme cases. The results confirm the instability of constant scaling, showing that within the optimal range of the trade-off between HR and other metrics, artifacts start to appear. This behavior is not observed when using the adaptive size (\cref{fig:guidance_ratio}), confirming its validity.

\begin{figure}
    \centering
    \includegraphics[width=0.99\linewidth]{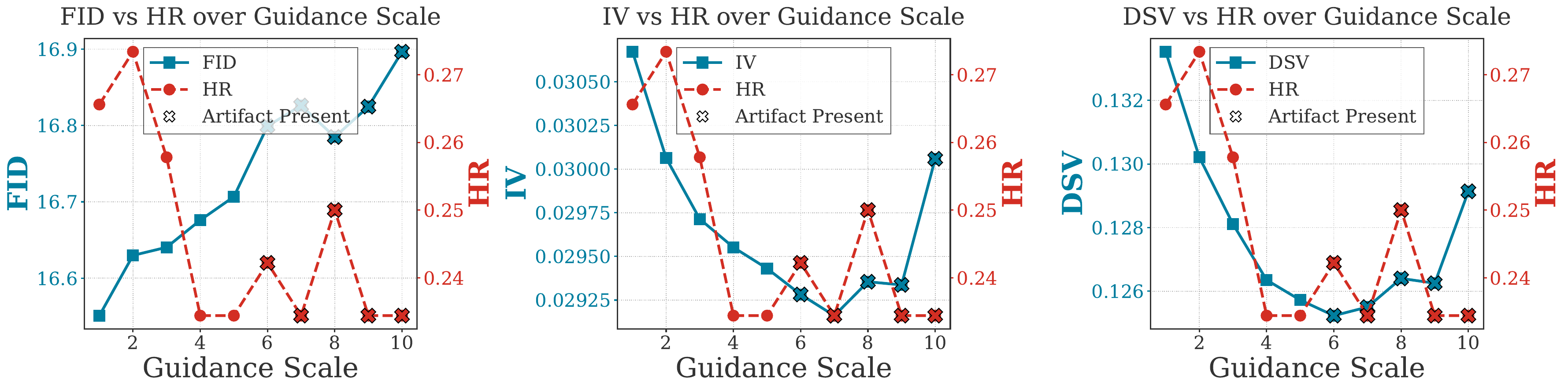}
    \caption{Ablation of constant scaling on the \texttt{11kHands} dataset.}
    \label{fig:guidance_scale}
\end{figure}

\subsubsection{Ablating the number of noise samples}
The last hyperparameter of \(\iq\) is \(k\), the number of i.i.d noise samples used to estimate the DSM expectation and, in turn, \(\lidt\) (see \cref{sec:hyperparameters}). In contrast to the previous studies, here we reuse the exact evaluation from \cref{sec:experiments} (2048 samples for proxy metrics, the blinded human study for UP and HR), hence the results for \(k=32\) match \cref{tab:main_results}. \Cref{tab:ablation_k} presents the results of sweeping \(k\in\{2, 8, 32\}\) on \texttt{11kHands}, together with the mean runtime. Decreasing \(k\) 4-fold has almost no effect on the performance of \(\iq\), while lowering the runtime by almost \(3\times\). Furthermore, \(k=2\) still preserves most of \(\iq\)'s effectiveness, mitigating over 50\% of hallucinations with a runtime almost \(5\times\) lower than for the default setting. These results suggest that \(k\) does not have to account for highly precise gradient estimation, but rather functions as a dial for the efficiency-budget trade-off. We leverage this observation when benchmarking large-scale models in \cref{sec:runtime_comparison}.

\begin{table}[h]
  \caption{Ablation of the number \(k\) of i.i.d noise samples used to estimate \(\lidt\) on the \texttt{11kHands} dataset. Runtime is reported as the mean runtime (in seconds), alongside the 95\% confidence interval.}
  \label{tab:ablation_k}
  \centering
  \begin{tabular}{l cccccc}
    \toprule
    Method & FID $\downarrow$ & IV $\uparrow$ & DSV $\uparrow$ & UP $\uparrow$ & HR $\downarrow$ & Runtime (s) \\
    \midrule
    Baseline & 16.3 & 0.032 & 0.15 & 39.8 & 29.3 & $2.84 \pm 0.003$ \\
    \(\iq\) (\(k=2\)) & 16.4 & 0.031 & 0.14 & 61.3 & 13.5 & $3.90 \pm 0.004$ \\
    \(\iq\) (\(k=8\)) & 16.5 & 0.031 & 0.13 & 66.5 & 9.8 & $6.93 \pm 0.010$ \\
    \(\iq\) (\(k=32\)) & 16.6 & 0.030 & 0.13 & 68.0 & 9.0 & $19.01 \pm 0.010$ \\
    \bottomrule
  \end{tabular}
\end{table}

\subsection{Runtime comparison}\label{sec:runtime_comparison}

We compare the average sampling runtime of \(\iq\) with default unconditional sampling and both RODS variants as the only unconstrained methods for hallucination reduction. The comparison is performed on the most practical and high-resolution \texttt{RSNA} dataset, with results presented in \cref{tab:runtime_rsna}. While all methods add a significant overhead to baseline sampling, \(\iq\) remains around 20\% faster than both RODS\textsuperscript{CAS} and RODS\textsuperscript{SAS}. These results confirm the superiority of \(\iq\) but also point toward an important research direction of reducing the computational cost while preserving the method's performance.

\begin{table}[h]
  \caption{Runtime comparison of unconstrained methods on the \texttt{RSNA} dataset. Values are reported as the mean sampling loop time (in seconds) over 16 batches of size 8, alongside the 95\% confidence interval.}
  \label{tab:runtime_rsna}
  \centering
  \begin{tabular}{lc}
    \toprule
    Method & Runtime (s) \\
    \midrule
    Baseline & $11.49 \pm 0.63$ \\
    IQ & $49.83 \pm 0.40$ \\
    RODS\textsuperscript{CAS} & $63.86 \pm 2.27$ \\
    RODS\textsuperscript{SAS} & $60.33 \pm 2.37$ \\
    \bottomrule
  \end{tabular}
\end{table}

To further assess the practical applicability of \(\iq\), we implement the inference pipelines of Stable Diffusion XL (SDXL) \citep{podell2024sdxl} and PixelDiT \citep{yu2026pixeldit} (ImageNet, \(512\times512\) resolution) and adapt both \(\iq\) and RODS to their sampling procedures, using the same hyperparameters as for \texttt{11kHands} (for \(\iq\), this also means that the time window contains the same number of timesteps) and varying \(k\), following the analysis in \cref{sec:ablation_studies}. \Cref{tab:runtime_large_scale} presents the mean runtime with 95\% CIs for all methods on both models. Crucially, in terms of runtime, \(\iq\) scales much better to both text-to-image SDXL and high-resolution PixelDiT than RODS. Even in the default setting (\(k=32\)), \(\iq\) remains faster than both RODS variants by a visible margin (\eg, a \(6.7\times\) overhead on SDXL compared to \(7.8\times\) and \(9.8\times\) for RODS\textsuperscript{SAS} and RODS\textsuperscript{CAS}), and decreasing \(k\) to 2 reduces the overhead to \(1.5\times\) on SDXL and \(1.3\times\) on PixelDiT.

\begin{table}[h]
  \caption{Runtime comparison of unconstrained methods on large-scale models. All methods use the hyperparameters from \texttt{11kHands} (for \(\iq\), the time window contains the same number of timesteps). Values are reported as the mean runtime (in seconds) alongside the 95\% confidence interval.}
  \label{tab:runtime_large_scale}
  \centering
  \begin{tabular}{l cc}
    \toprule
    & \multicolumn{2}{c}{Runtime (s)} \\
    \cmidrule(lr){2-3}
    Method & SDXL & PixelDiT \\
    \midrule
    Baseline & $31.931 \pm 0.014$ & $7.88 \pm 0.003$ \\
    \(\iq\) (\(k=2\)) & $48.633 \pm 0.045$ & $10.25 \pm 0.010$ \\
    \(\iq\) (\(k=8\)) & $81.274 \pm 0.017$ & $14.45 \pm 0.003$ \\
    \(\iq\) (\(k=32\)) & $215.031 \pm 0.053$ & $31.42 \pm 0.005$ \\
    RODS\textsuperscript{SAS} & $249.459 \pm 1.501$ & $41.78 \pm 0.015$ \\
    RODS\textsuperscript{CAS} & $311.964 \pm 2.464$ & $41.82 \pm 0.015$ \\
    \bottomrule
  \end{tabular}
\end{table}

\subsection{Hyperparameters}\label{sec:hyperparameters}

We include the hyperparameters used for DG in \cref{tab:params_dg}, AAM in \cref{tab:params_aam}, RODS\textsuperscript{CAS} in \cref{tab:params_rods_cas}, RODS\textsuperscript{SAS} in \cref{tab:params_rods_sas}, and \(\iq\) in \cref{tab:params_iq}. These were obtained by running a grid search over approximately 500 configurations across all datasets for each baseline method.

For \(\iq\), we emphasize that \(k\), the number of i.i.d noise samples used to estimate \(\lidt\), is kept constant at \(k=32\) across all datasets in the main evaluation. We observed that this value universally ensures the stability of the estimator; \cref{sec:ablation_studies} additionally shows that lowering \(k\) preserves most of the mitigation strength while substantially reducing the runtime. To compute \(q_t\), we utilize 2048 reference samples for each dataset.

For \texttt{RSNA}, \(\iq\) is applied between \(t=0.7\) and \(t=0.75\), which corresponds to low noise levels within the SDB framework that utilizes an inverted U-shape noise schedule. For \texttt{SimpleShapes}, we followed \citet{aithal2024understanding}, which noted that the typical behavior observed in DMs shifts to larger timesteps and hence apply \(\iq\) at relatively higher times. However, for timesteps smaller than the indicated interval, the image is already fully formed and hence guidance has no practical effect.

\begin{table}[h]
  \caption{Hyperparameters for Dynamic Guidance (DG). The names of hyperparameters follow their respective names in our implementation. For EDM datasets, \texttt{dg\_start} and \texttt{dg\_end} are normalized to the $[0, 1]$ interval by dividing the $\sigma$ boundaries by $\sigma_{max}=80$. For DDPM datasets, they are normalized by dividing the timesteps by $T_{max}=1000$. For SDB, timesteps natively reside in the $[0, 1]$ interval.}
  \label{tab:params_dg}
  \small
  \centering
  \begin{tabular}{l ccccccc}
    \toprule
    Parameter & \texttt{11kHands} & \texttt{FFHQ} & \texttt{AFHQV2} & \texttt{MNIST} & \texttt{SimpleShapes} & \texttt{GaussianGrid} &  \texttt{RSNA}  \\
    \midrule
    \texttt{dg\_scale} & 15.0    & 10.0   & 1.0    & 10.0 & 10.0 & 10.0 & $10^5$ \\
    \texttt{dg\_start} & 0.125   & 1.0    & 0.125  & 0.5  & 0.25 & 0.25 & 0.7 \\
    \texttt{dg\_end}   & 0.00125 & 0.0125 & 0.0125 & 0.25 & 0.0  & 0.0  & 0.4 \\
    \bottomrule
  \end{tabular}
\end{table}

\begin{table}[h]
  \caption{Hyperparameters for Adaptive Attention Modulation (AAM). The names of hyperparameters follow their respective names in our implementation.}
  \label{tab:params_aam}
  \centering
  \begin{tabular}{l cccccc}
    \toprule
    Parameter & \texttt{11kHands} & \texttt{FFHQ} & \texttt{AFHQV2} & \texttt{MNIST} & \texttt{SimpleShapes} &  \texttt{RSNA}  \\
    \midrule
    $\gamma$ & 2.5  & 1.5  & 2.5  & 2.5 & 2.5  & 2.0  \\
    $\beta$  & 1.25 & 1.3  & 1.3  & 0.1 & 0.3  & 0.01 \\
    $N$      & 15   & 15   & 15   & 15  & 5    & 15   \\
    $\eta$   & 0.1  & 0.05 & 0.05 & 0.1 & 0.05 & 0.01 \\
    \bottomrule
  \end{tabular}
\end{table}

\begin{table}[h]
\small
  \caption{Hyperparameters for RODS\textsuperscript{CAS}. The names of hyperparameters follow their respective names in our implementation.}
  \label{tab:params_rods_cas}
  \centering
  \begin{tabular}{l ccccccc}
    \toprule
    Parameter & \texttt{11kHands} & \texttt{FFHQ} & \texttt{AFHQV2} & \texttt{MNIST} & \texttt{SimpleShapes} & \texttt{GaussianGrid} &  \texttt{RSNA}  \\
    \midrule
    \texttt{perturb\_scale}  & 30.0  & 8.0  & 1.0 & 5.0   & 1.0   & 0.1    & 0.1   \\
    \texttt{cor\_threshold}  & 0.014 & 0.09 & 0.1 & 0.005 & 0.005 & 0.0001 & 0.002 \\
    \texttt{num\_adv\_steps} & 1     & 1    & 1   & 1     & 1     & 1      & 1     \\
    \bottomrule
  \end{tabular}
\end{table}

\begin{table}[h]
    \small
  \caption{Hyperparameters for RODS\textsuperscript{SAS}. The names of hyperparameters follow their respective names in our implementation.}
  \label{tab:params_rods_sas}
  \centering
  \begin{tabular}{l ccccccc}
    \toprule
    Parameter & \texttt{11kHands} & \texttt{FFHQ} & \texttt{AFHQV2} & \texttt{MNIST} & \texttt{SimpleShapes} & \texttt{GaussianGrid} &  \texttt{RSNA}  \\
    \midrule
    \texttt{perturb\_scale}  & 30.0  & 8.0  & 1.0 & 5.0  & 20.0 & 0.1    & 0.1   \\
    \texttt{cor\_threshold}  & 0.014 & 0.09 & 0.1 & 0.01 & 0.05 & 0.0001 & 0.002 \\
    \texttt{num\_adv\_steps} & 1     & 1    & 1   & 1    & 1    & 1      & 1     \\
    \bottomrule
  \end{tabular}
\end{table}

\begin{table}[h]
  \caption{Hyperparameters for \(\iq\). The names of hyperparameters follow their respective names in our implementation. For EDM datasets, $t_1$ and $t_2$ are normalized to the $[0, 1]$ interval by dividing the guidance $\sigma$ boundaries by $\sigma_{max}=80$. For DDPM datasets, they are normalized by dividing the timesteps by $T_{max}=1000$. For SDB, timesteps natively reside in the $[0, 1]$ interval.}
  \label{tab:params_iq}
  \small
  \centering
  \begin{tabular}{l ccccccc}
    \toprule
    Parameter & \texttt{11kHands} & \texttt{FFHQ} & \texttt{AFHQV2} & \texttt{MNIST} & \texttt{SimpleShapes} & \texttt{GaussianGrid} &  \texttt{RSNA}  \\
    \midrule
    $t_1$     & 0.025  & 0.025  & 0.0088 & 0.24 & 0.48 & 0.0  & 0.70 \\
    $t_2$     & 0.0625 & 0.0375 & 0.025  & 0.40 & 0.96 & 0.72 & 0.75 \\
    $\lambda$ & 0.08   & 0.035  & 0.022  & 0.5  & 0.08 & 0.9  & 0.8  \\
    $q$       & 0.4    & 0.05   & 0.2    & 0.0  & 0.05 & 0.65  & 0.0  \\
    $k$       & 32     & 32     & 32     & 32   & 32   & 32   & 32   \\
    \bottomrule
  \end{tabular}
\end{table}

\subsection{SDXL evaluation details}\label{sec:sdxl_details}
We provide the complete details of the small-scale SDXL evaluation from \cref{sec:experiments}, including the structured prompt, the statistical analysis, and the hyperparameters. The prompt enforces predictable hand positioning, enabling straightforward visual inspection of structural hallucinations (\eg, missing or additional fingers, deformed joints). While these errors mirror cases from \texttt{11kHands}, their scope is naturally broader given SDXL's generative capacity, making mitigation significantly more challenging:
\begin{quote}
\small\texttt{A hyper-realistic macro photograph of a single human left hand, palm facing the camera, held flat against a plain white background, with all five fingers spread wide apart and clearly separated: one thumb, one index finger, one middle finger, one ring finger, and one little finger, five fingers in total. A plain silver band ring is worn on the ring finger. Soft even studio lighting, sharp focus, high contrast, 8k resolution, highly detailed skin texture.}
\end{quote}
Following an initial hyperparameter grid search, the evaluation is conducted on 64 generated samples labeled case-by-case by 3 independent annotators. For \(\iq\), the hyperparameters comprise the quenching magnitude \(\lambda=1\), the time window \([0.76, 0.81]\), expressed in SDXL's normalized scale, and no filtering (\(q=0\)), with \(k=16\) set to maintain a reasonable time budget. The same grid search procedure, including the tuning of both variants, is applied to RODS. To evaluate statistical significance, we fit a Generalized Estimating Equation (GEE) logistic regression model (binomial family, logit link, exchangeable working correlation grouped by the model's generation seeds). Based on Wald's test on model coefficients, we obtain a Holm-Bonferroni-adjusted \(p\)-value of 0.00059, controlling for multiple testing across the evaluated hyperparameter grid, confirming the high statistical robustness of the reported mitigation. Annotators reported no visible degradation in image quality, suggesting that in models like SDXL, \(\iq\) primarily acts by shifting generation toward regions with deflated \(\lid\) to stabilize structural output. Naturally, this initial result does not claim that \(\iq\) is the ultimate solution for hallucination mitigation in such models, but offers a promising base for potential future improvements.

\subsection{Pseudocode}\label{sec:pseudocode}

We provide detailed pseudocode of \(\iq\) in \cref{alg:iq}.

\begin{algorithm}[htbp]
\caption{\textbf{Intrinsic Quenching (IQ): sampling step.}\\
\small{Note: \texttt{autograd} returns the gradient of the energy function w.r.t the input.}}
\label{alg:iq}
\begin{lstlisting}[language=PseudoPython]
# net: EDM denoiser predicting clean data x_0
# x: noisy input batch at time t
# lambda: target magnitude ratio for guidance
# q: target quantile for filtration (e.g., 0.4)

# Apply Intrinsic Quenching only within the window [t_1, t_2]
if t < t_1 or t > t_2:
    return stopgrad(net(x, t))

x.requires_grad_(True)

# Predict posterior mean x_0_hat and evaluate DSM energy (LID)
x_0_hat = net(x, t)
LID = dsm_loss(x_0_hat, x, t)

# Compute energy gradient w.r.t the input state x
grad = autograd(LID.sum(), x)

# Project guidance into data space (Sigma_t = t^2 * I for EDM)
raw_guidance = (t ** 2) * grad

# Calculate adaptive scale to maintain constant magnitude ratio
nat_update = x_0_hat - x
scale = (lambda * norm(nat_update)) / (norm(raw_guidance) + 1e-8)

# Filtration strategy
# q_t is the q-th quantile of the baseline LID distribution at time t
q_t = quantile(baseline_LID(t), q)
mask = where(LID >= q_t, 1.0, 0.0)

# Apply masked adaptive EDM update
guided_x_0_hat = x_0_hat - (mask * scale * raw_guidance)

return stopgrad(guided_x_0_hat)
\end{lstlisting}
\end{algorithm}

\subsection{Resources}\label{sec:resources}

All of the experiments were performed on NVIDIA A100 40GB and H100 80GB GPUs, with runtime (\cref{tab:runtime_rsna}) reported on the former. The entire project required around 1000 GPU hours in total (this includes any preliminary or grid search experiment). The GPUs were part of NVIDIA DGX nodes.

\subsection{Qualitative examples}\label{sec:qualitative_examples}

We include more qualitative comparisons on \texttt{11kHands} (\cref{fig:11k_hands_vis}), \texttt{AFHQV2} (\cref{fig:afhqv2_vis}), \texttt{FFHQ} (\cref{fig:ffhq_vis}), \texttt{MNIST} (\cref{fig:mnist_vis}), \texttt{SimpleShapes} (\cref{fig:shapes_vis}), high-resolution version of \texttt{GaussianGrid} results (\cref{fig:hq_gauss_vis}), ground truth scans and their reconstructions from  \texttt{RSNA} (\cref{fig:rsna_recon}) and the corresponding qualitative comparison (\cref{fig:rsna_vis}).

\begin{figure}
    \centering
    \includegraphics[width=0.99\linewidth]{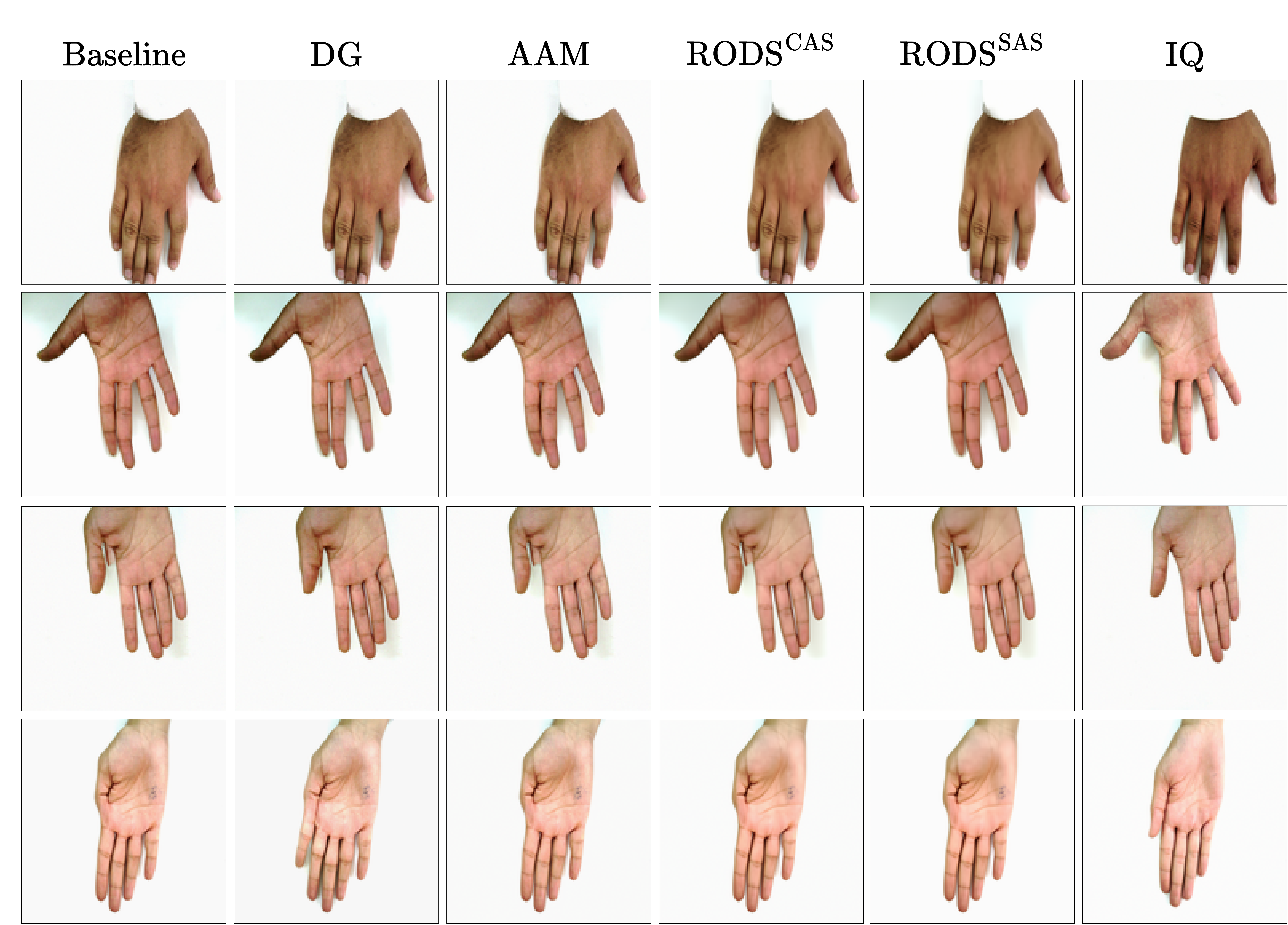}
    \caption{Qualitative comparison of all sampling methods on \texttt{11kHands}.}
    \label{fig:11k_hands_vis}
\end{figure}

\begin{figure}
    \centering
    \includegraphics[width=0.99\linewidth]{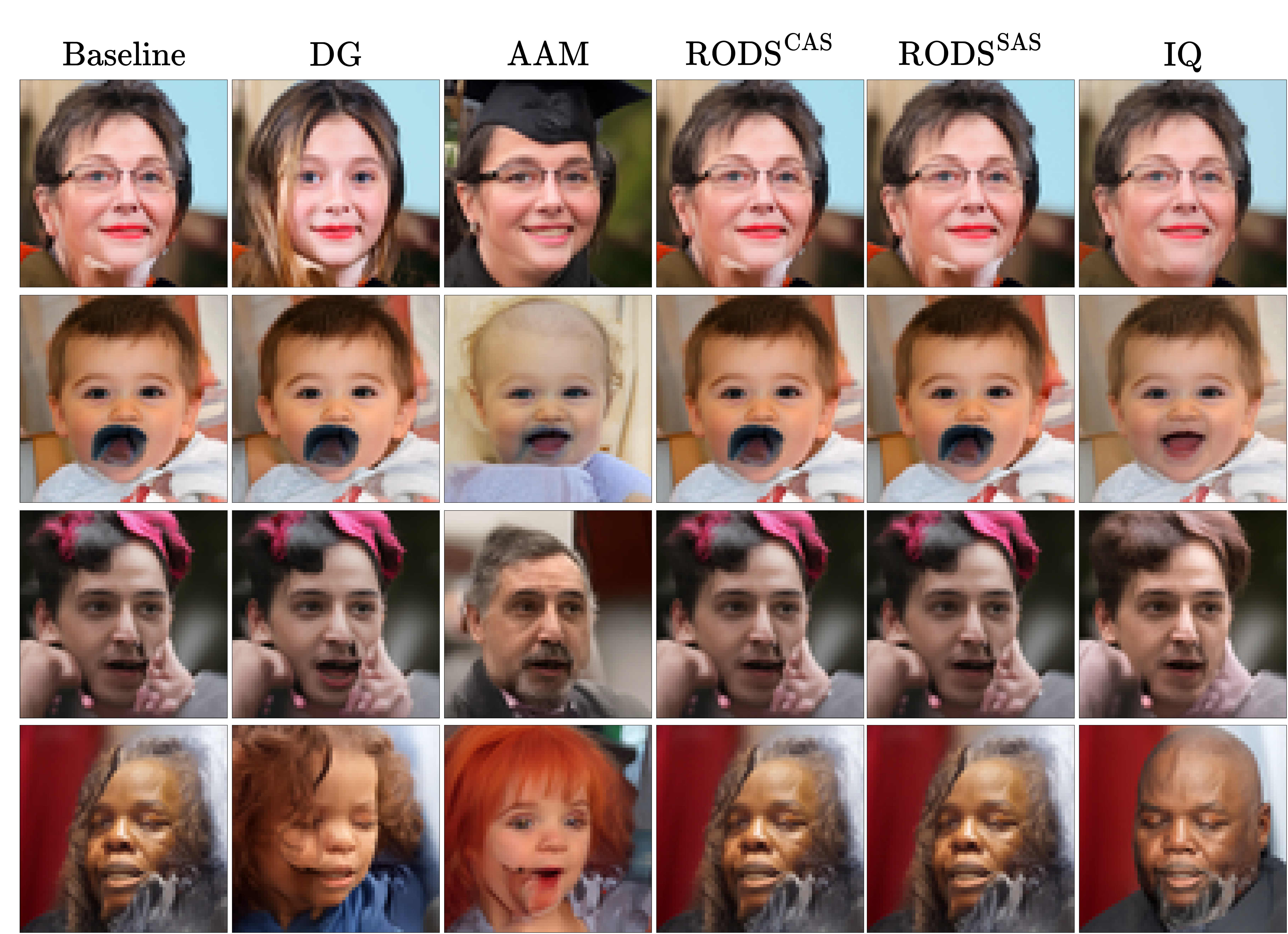}
    \caption{Qualitative comparison of all sampling methods on \texttt{FFHQ}.}
    \label{fig:ffhq_vis}
\end{figure}

\begin{figure}
    \centering
    \includegraphics[width=0.99\linewidth]{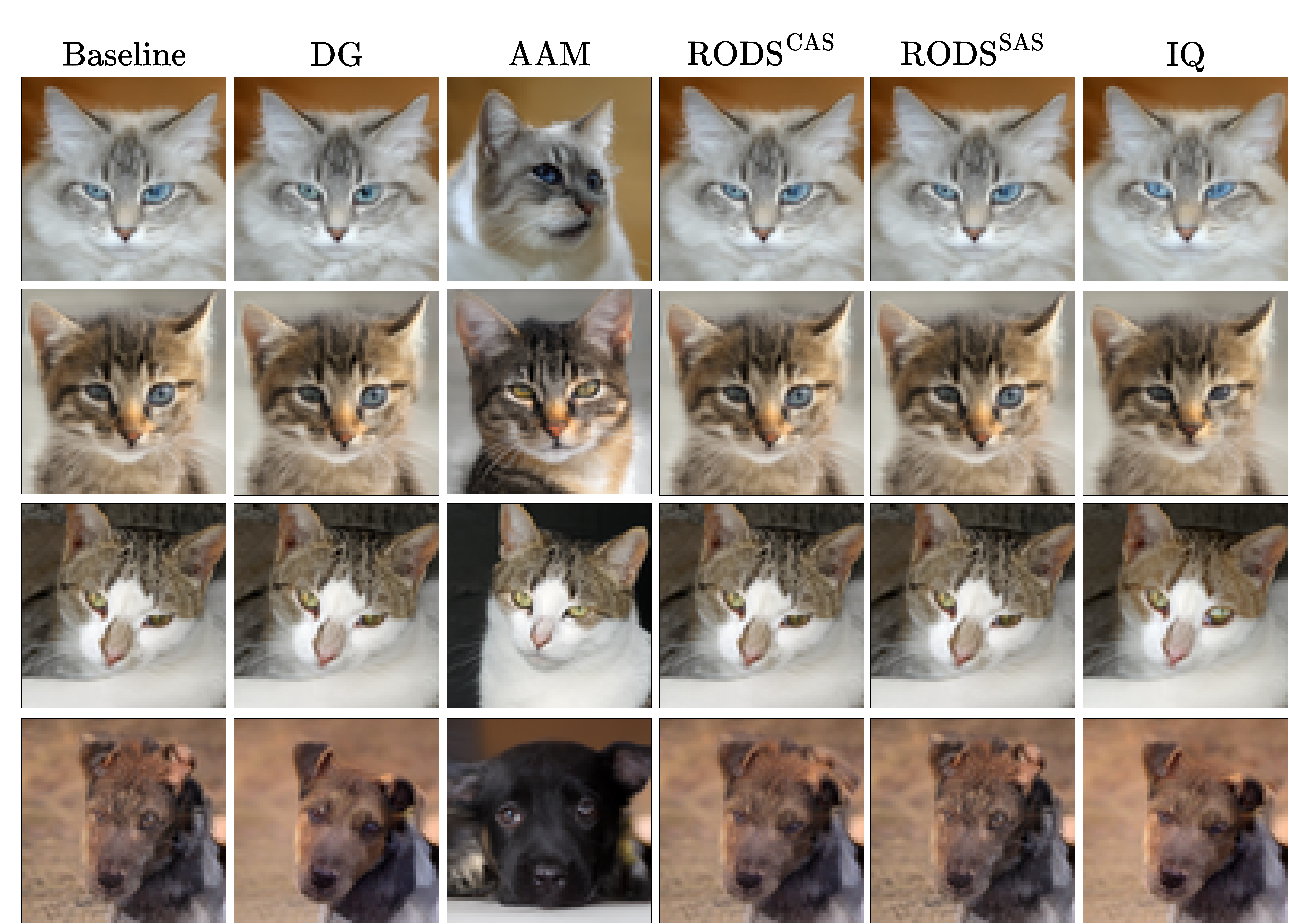}
    \caption{Qualitative comparison of all sampling methods on \texttt{AFHQV2}.}
    \label{fig:afhqv2_vis}
\end{figure}

\begin{figure}
    \centering
    \includegraphics[width=0.99\linewidth]{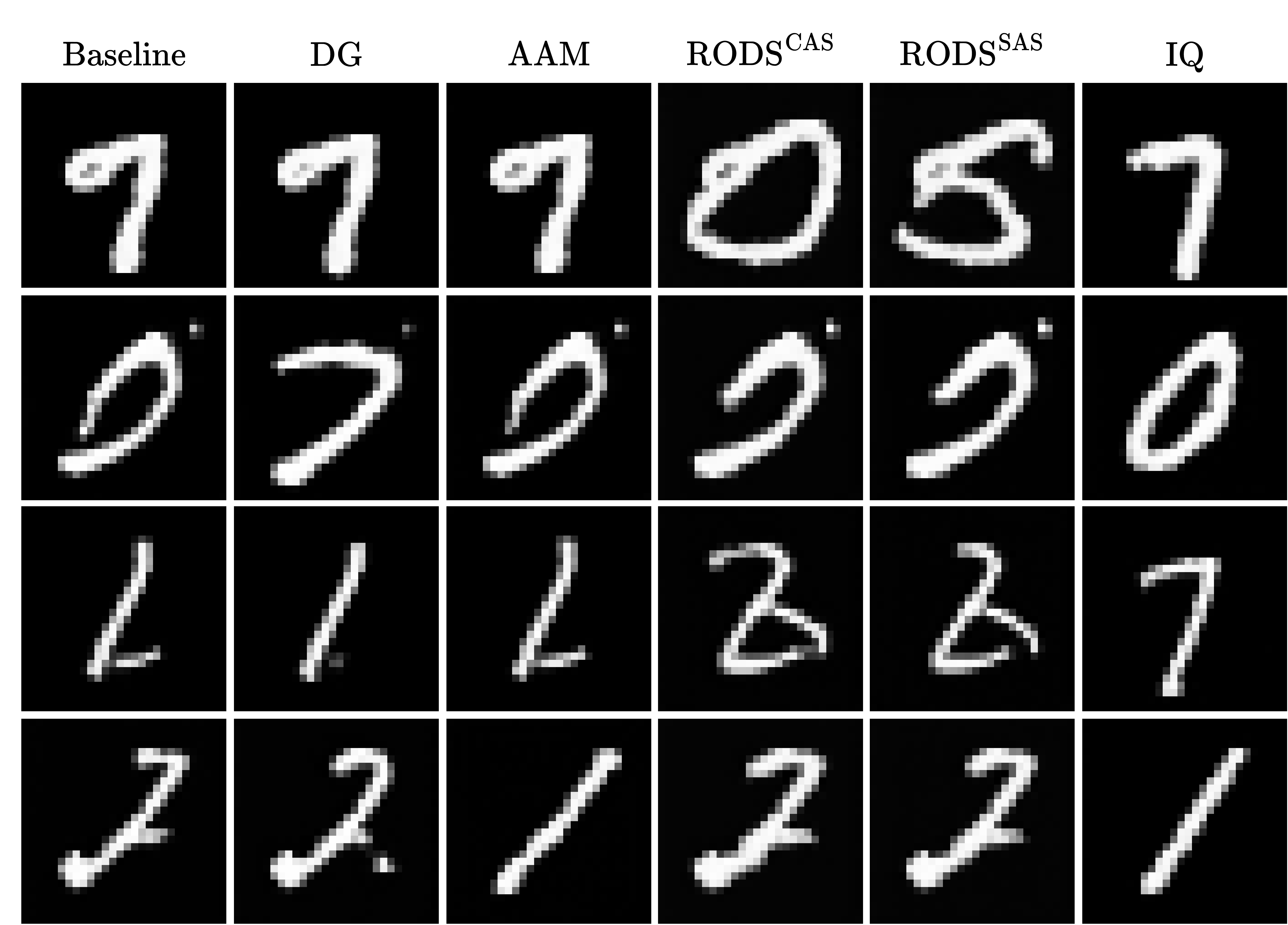}
    \caption{Qualitative comparison of all sampling methods on \texttt{MNIST}.}
    \label{fig:mnist_vis}
\end{figure}

\begin{figure}
    \centering
    \includegraphics[width=0.99\linewidth]{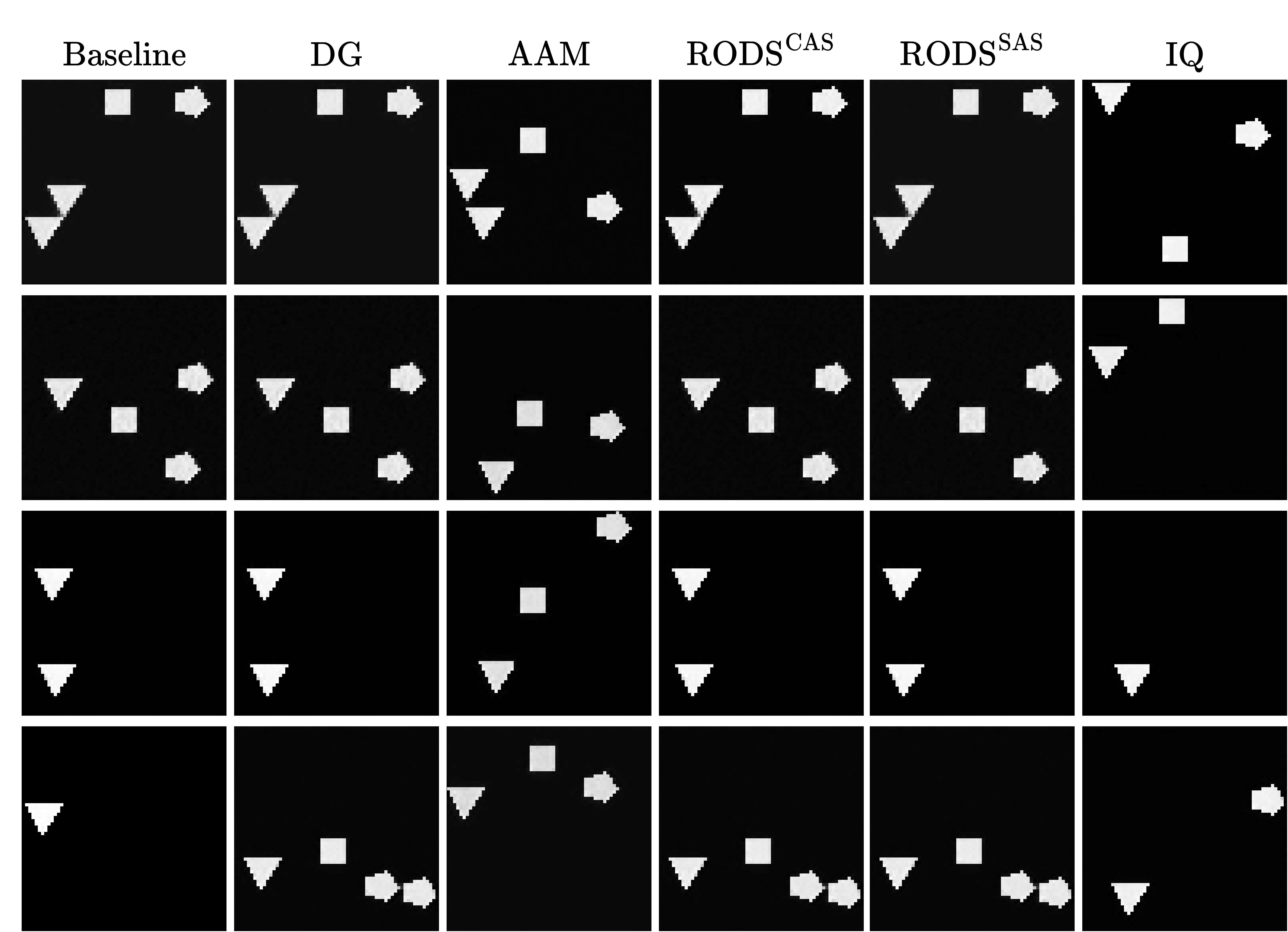}
    \caption{Qualitative comparison of all sampling methods on \texttt{SimpleShapes}.}
    \label{fig:shapes_vis}
\end{figure}

\begin{figure}
    \centering    \includegraphics[width=0.98\linewidth]{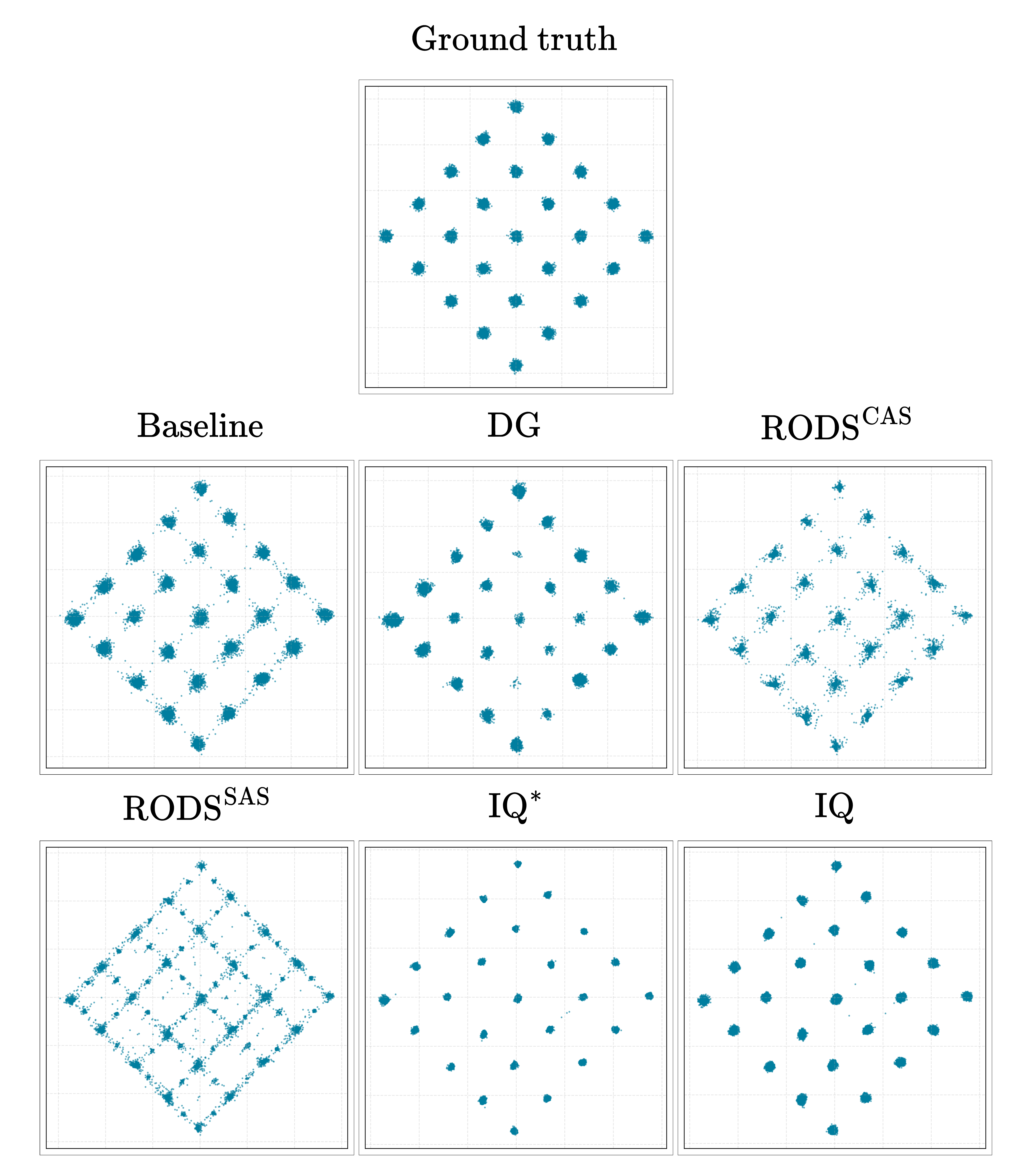}
    \caption{High-resolution visualization of samples generated by each method on the \texttt{GaussianGrid} dataset. Each plot depicts 16384 samples.}
    \label{fig:hq_gauss_vis}
\end{figure}

\begin{figure}
    \centering
    \includegraphics[width=0.99\linewidth]{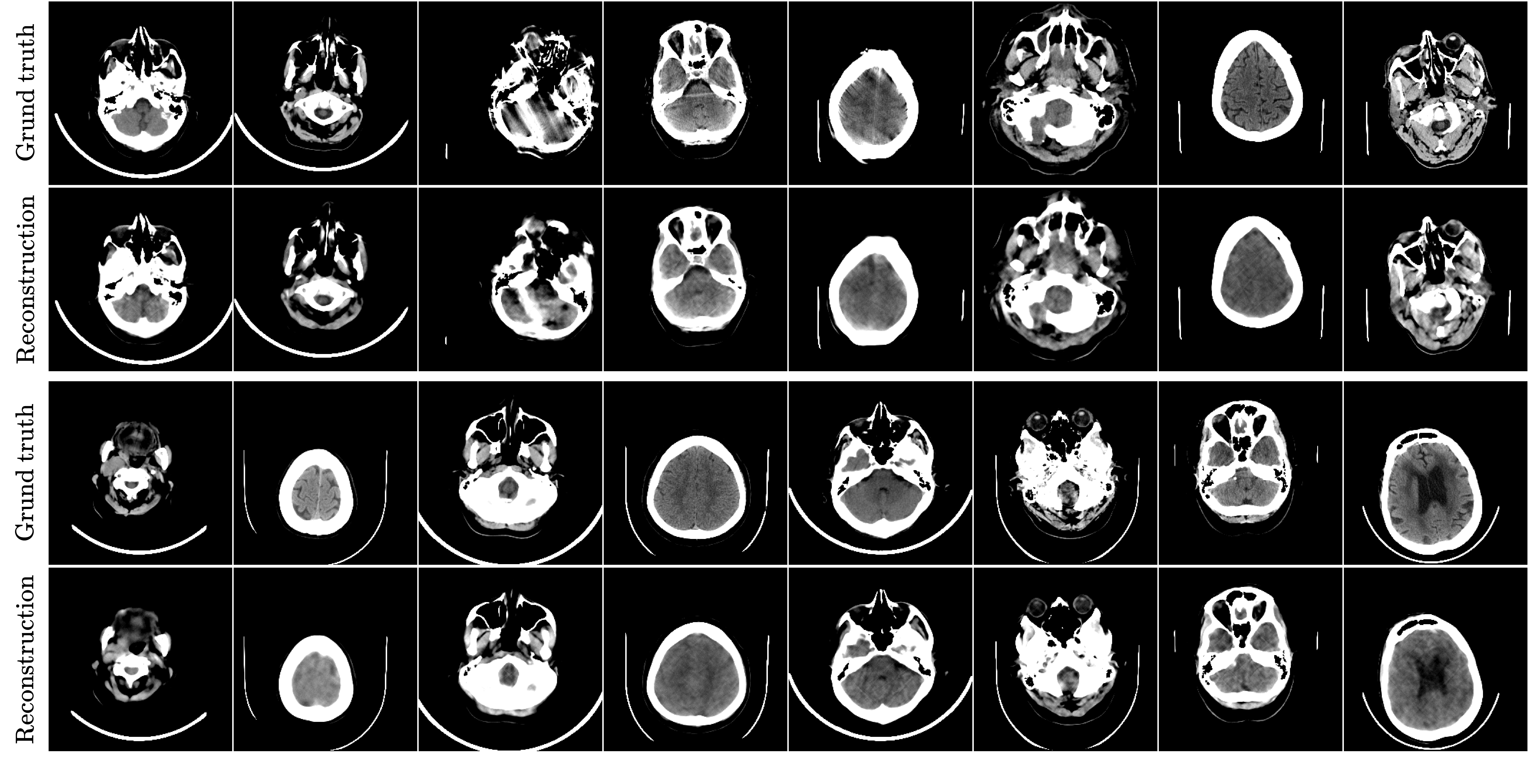}
    \caption{Ground truth brain CT scans from the \texttt{RSNA} dataset and their corresponding reconstructions obtained with SDB using 100-step sampling.}
    \label{fig:rsna_recon}
\end{figure}

\begin{figure}
    \centering
    \includegraphics[width=0.99\linewidth]{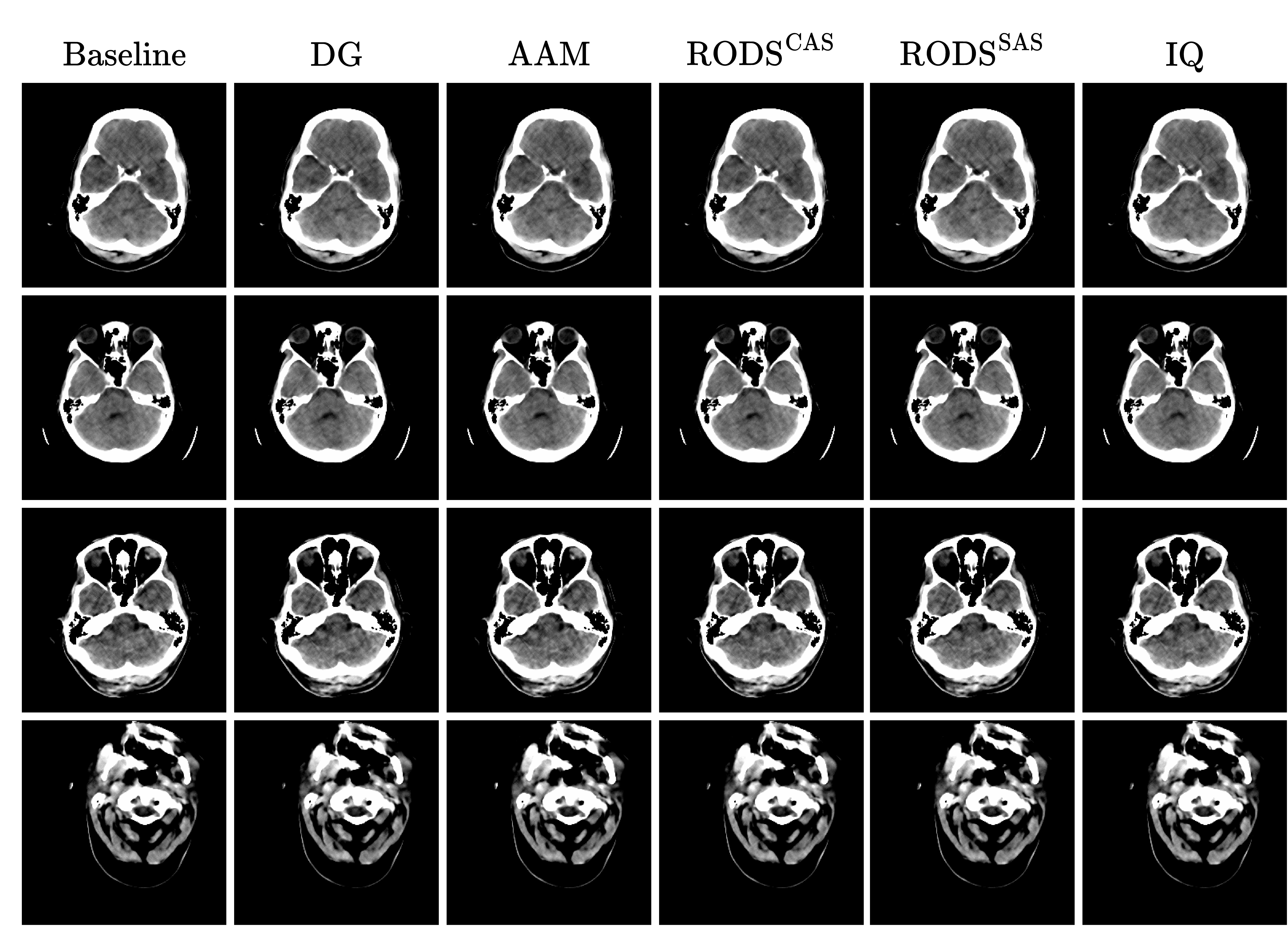}
    \caption{Qualitative comparison of all sampling methods on \texttt{RSNA}.}
    \label{fig:rsna_vis}
\end{figure}
%%%%%%%%%%%%%%%%%%%%%%%%%%%%%%%%%%%%%%%%%%%%%%%%%%%%%%%%%%%%

\clearpage
\section*{NeurIPS Paper Checklist}

\begin{enumerate}

\item {\bf Claims}
    \item[] Question: Do the main claims made in the abstract and introduction accurately reflect the paper's contributions and scope?
    \item[] Answer: \answerYes{} % Replace by \answerYes{}, \answerNo{}, or \answerNA{}.
    \item[] Justification: Our contributions include the initial empirical investigation (\cref{sec:empirical_investigation}), the theoretical results with proofs (informal statements in \cref{sec:theoretical_foundation}, full formal statements with proofs in \cref{sec:theoretical_results}) and the empirical results in \cref{sec:experiments} and \cref{sec:extended_experiments}.

\item {\bf Limitations}
    \item[] Question: Does the paper discuss the limitations of the work performed by the authors?
    \item[] Answer: \answerYes{} % Replace by \answerYes{}, \answerNo{}, or \answerNA{}.
    \item[] Justification: The final section discusses the limitations regarding the limited scope of proxy metrics for image quality and diversity, pointing to an unknown theoretical relationship, and the computational overhead of the proposed method.

\item {\bf Theory assumptions and proofs}
    \item[] Question: For each theoretical result, does the paper provide the full set of assumptions and a complete (and correct) proof?
    \item[] Answer: \answerYes{} % Replace by \answerYes{}, \answerNo{}, or \answerNA{}.
    \item[] Justification: These are provided in \cref{sec:theoretical_foundation} and \cref{sec:theoretical_results}.

    \item {\bf Experimental result reproducibility}
    \item[] Question: Does the paper fully disclose all the information needed to reproduce the main experimental results of the paper to the extent that it affects the main claims and/or conclusions of the paper (regardless of whether the code and data are provided or not)?
    \item[] Answer: \answerYes{} % Replace by \answerYes{}, \answerNo{}, or \answerNA{}.
    \item[] Justification: The paper provides a detailed explanation of the proposed method, high-level overview of each baseline, hyperparameters of all methods and a fully reproducible codebase.

\item {\bf Open access to data and code}
    \item[] Question: Does the paper provide open access to the data and code, with sufficient instructions to faithfully reproduce the main experimental results, as described in supplemental material?
    \item[] Answer: \answerYes{} % Replace by \answerYes{}, \answerNo{}, or \answerNA{}.
    \item[] Justification: The code is included as supplementary material.

\item {\bf Experimental setting/details}
    \item[] Question: Does the paper specify all the training and test details (e.g., data splits, hyperparameters, how they were chosen, type of optimizer) necessary to understand the results?
    \item[] Answer: \answerYes{} % Replace by \answerYes{}, \answerNo{}, or \answerNA{}.
    \item[] Justification: The necessary details are either provided in the main text, the Appendix or in the attached codebase.

\item {\bf Experiment statistical significance}
    \item[] Question: Does the paper report error bars suitably and correctly defined or other appropriate information about the statistical significance of the experiments?
    \item[] Answer: \answerYes{} % Replace by \answerYes{}, \answerNo{}, or \answerNA{}.
    \item[] Justification: Error bars are reported for the human-based evaluation. Proxy metrics (such as FID) are reported based on a single run, following standard practice and due to high computational demands.

\item {\bf Experiments compute resources}
    \item[] Question: For each experiment, does the paper provide sufficient information on the computer resources (type of compute workers, memory, time of execution) needed to reproduce the experiments?
    \item[] Answer: \answerYes{} % Replace by \answerYes{}, \answerNo{}, or \answerNA{}.
    \item[] Justification: Resources are reported in \cref{sec:resources}.
    
\item {\bf Code of ethics}
    \item[] Question: Does the research conducted in the paper conform, in every respect, with the NeurIPS Code of Ethics \url{https://neurips.cc/public/EthicsGuidelines}?
    \item[] Answer: \answerYes{} % Replace by \answerYes{}, \answerNo{}, or \answerNA{}.
    \item[] Justification: The conducted research fully conforms with the Code of Ethics. 

\item {\bf Broader impacts}
    \item[] Question: Does the paper discuss both potential positive societal impacts and negative societal impacts of the work performed?
    \item[] Answer: \answerYes{} % Replace by \answerYes{}, \answerNo{}, or \answerNA{}.
    \item[] Justification: Removing structural hallucinations in diffusion models has positive societal impacts; it allows for a smoother, more reliable interaction with image generation models, increasing trustworthiness, and may positively affect scenarios where diffusion models influence a decision-making process, \eg, in medical image reconstruction. Regarding negative impacts, failure of \(\iq\) or similar methods in reducing hallucinations leads to increased computational overhead, which inflates the carbon footprint.
    
\item {\bf Safeguards}
    \item[] Question: Does the paper describe safeguards that have been put in place for responsible release of data or models that have a high risk for misuse (e.g., pre-trained language models, image generators, or scraped datasets)?
    \item[] Answer: \answerNA{} % Replace by \answerYes{}, \answerNo{}, or \answerNA{}.
    \item[] Justification: The paper does not pose such risks.

\item {\bf Licenses for existing assets}
    \item[] Question: Are the creators or original owners of assets (e.g., code, data, models), used in the paper, properly credited and are the license and terms of use explicitly mentioned and properly respected?
    \item[] Answer: \answerYes{} % Replace by \answerYes{}, \answerNo{}, or \answerNA{}.
    \item[] Justification: The paper properly cites each source of existing data that was used throughout the research process and is required for reproduction. For ported chunks of code, the original repository is mentioned. The supplementary material does not provide any of the cited resources directly; it redirects the reader to the proper source.

\item {\bf New assets}
    \item[] Question: Are new assets introduced in the paper well documented and is the documentation provided alongside the assets?
    \item[] Answer: \answerNA{} % Replace by \answerYes{}, \answerNo{}, or \answerNA{}.
    \item[] Justification: The paper's contributions are not focused on providing new assets. The details required for reproducing the results are provided within the attached codebase.

\item {\bf Crowdsourcing and research with human subjects}
    \item[] Question: For crowdsourcing experiments and research with human subjects, does the paper include the full text of instructions given to participants and screenshots, if applicable, as well as details about compensation (if any)? 
    \item[] Answer: \answerYes{} % Replace by \answerYes{}, \answerNo{}, or \answerNA{}.
    \item[] Justification: The user study is described in detail (\cref{sec:user_study}). The human subjects were compensated in a way that fully aligns with their country's law.

\item {\bf Institutional review board (IRB) approvals or equivalent for research with human subjects}
    \item[] Question: Does the paper describe potential risks incurred by study participants, whether such risks were disclosed to the subjects, and whether Institutional Review Board (IRB) approvals (or an equivalent approval/review based on the requirements of your country or institution) were obtained?
    \item[] Answer: \answerYes{} % Replace by \answerYes{}, \answerNo{}, or \answerNA{}.
    \item[] Justification: The IRB approval for the conducted user study was not required, as it was based on risk-free protocol and anonymous, publicly available datasets. The human subjects were fully informed about the structure of the user study and gave informed consent regarding the participation.

\item {\bf Declaration of LLM usage}
    \item[] Question: Does the paper describe the usage of LLMs if it is an important, original, or non-standard component of the core methods in this research? Note that if the LLM is used only for writing, editing, or formatting purposes and does \emph{not} impact the core methodology, scientific rigor, or originality of the research, declaration is not required.
    %this research? 
    \item[] Answer: \answerNA{} % Replace by \answerYes{}, \answerNo{}, or \answerNA{}.
    \item[] Justification: LLMs were used solely for editing and formatting purposes.

\end{enumerate}

\end{document}